\documentclass[11pt,onecolumn]{article}
\usepackage[utf8]{inputenc}
\usepackage{multirow}
\usepackage{multicol}
\usepackage{enumitem}

\usepackage{sidecap}
\usepackage{slashbox}
\usepackage{times}
\usepackage{graphicx,subfigure}
\usepackage{amsmath,amssymb, amsthm}
\usepackage{cite}
\usepackage{hyperref}
\usepackage{url}
\usepackage{dsfont}

\usepackage{bbm}
\usepackage{color}
\usepackage{algorithm}
\usepackage{algpseudocode}
\usepackage[leftFloats,CaptionAfterwards]{fltpage}
\usepackage{mwe}
\usepackage{etoolbox}
\makeatletter
\patchcmd{\FP@floatEnd}{%
	\begin{\@captype}[b!]%
}{%
	\begin{\@captype}[t!]%
}%
{}{}
\makeatother
\usepackage{mathtools}
\newtheorem{theorem}{Theorem}
\newtheorem{lemma}{Lemma}

\newtheorem{definition}{Definition}
\newtheorem{remark}{Remark}
\usepackage[toc,page, title, titletoc]{appendix}
\usepackage[table]{xcolor}
\newcommand\blfootnote[1]{%
  \begingroup
  \renewcommand\thefootnote{}\footnote{#1}%
  \addtocounter{footnote}{-1}%
  \endgroup
}
\newcommand{\dff}{\stackrel{\scriptscriptstyle\triangle}{=}}
\def \bX {{\mathbf{X}}}
\def \bx {{\mathbf{x}}}
\def \bs {{\mathbf{s}}}
\def \bz {{\mathbf{z}}}

\def \bLambda {{\mathbf{\Lambda}}}
\def \bC {{\mathbf{C}}}

\def \cU {{\mathcal{U}}}
\def \bV {{\mathbf{V}}}

\def \bs {{\mathbf{s}}}
\def \bD {{\mathbf{D}}}
\def \bW {{\mathbf{W}}}
\def \bJ {{\mathbf{J}}}

\def \bL {{\mathbf{L}}}

\def \bH {{\mathbf{H}}}

\def \by {{\mathbf{y}}}
\def \br {{\mathbf{r}}}
\def \bv{{\mathbf{v}}}
\def \bu {{\mathbf{u}}}
\def \bp {{\mathbf{p}}}
\def \bI {{\mathbf{I}}}
\def \mR {{\mathbb{R}}}
\def \mE {{\mathbb{E}}}
\def \mP {{\mathbb{P}}}
\def \mZ {{\mathbb{Z}}}
\def \bE {{\mathbf{E}}}
\def \cH {{\mathcal{H}}}
\def \cC {{\mathcal{C}}}

\usepackage{tcolorbox}

\usepackage{authblk}

\date{}
\title{Transferability of coVariance Neural Networks and Application to Interpretable Brain Age Prediction using Anatomical Features}
\begin{document}

\author[1]{ Saurabh Sihag}
\author[2]{ Gonzalo Mateos}
\author[1]{ Corey T. McMillan}
\author[1]{ Alejandro Ribeiro}

\affil[1]{\small University of Pennsylvania, Philadelphia, PA.}
\affil[2]{\small University of Rochester, Rochester, NY.}

\maketitle
\begin{abstract}
Graph convolutional networks (GCN) leverage topology-driven graph convolutional operations to combine information across the graph for inference tasks. In our recent work, we have studied GCNs with covariance matrices as graphs in the form of coVariance neural networks (VNNs) that draw similarities with traditional PCA-driven data analysis approaches while offering significant advantages over them. In this paper, we first focus on theoretically characterizing the transferability of VNNs. The notion of transferability is motivated from the intuitive expectation that learning models could generalize to ``compatible" datasets (possibly of different dimensionalities) with minimal effort. VNNs inherit the scale-free data processing architecture from GCNs and here, we show that VNNs exhibit transferability of performance (without re-training) over datasets whose covariance matrices converge to a limit object. Multi-scale neuroimaging datasets enable the study of the brain at multiple scales and hence, provide an ideal scenario to validate the theoretical results on the transferability of VNNs. To gauge the advantages offered by VNNs in neuroimaging data analysis, we focus on the task of ``brain age" prediction using cortical thickness features. In clinical neuroscience, there has been an increased interest in machine learning algorithms derived from MRI cortical thickness features which provide estimates of “brain age” that deviate from chronological age. Importantly, discordance between brain age and chronological age (“brain age gap”) can reflect increased vulnerability or resilience toward neurological disease or cognitive impairments. We leverage the architecture of VNNs to extend beyond the coarse metric of brain age gap in Alzheimer’s disease (AD) and make two important observations: (i) VNNs can assign anatomical interpretability to elevated brain age gap in AD by identifying contributing brain regions, and (ii) the interpretability offered by VNNs is contingent on their ability to exploit specific principal components of the anatomical covariance matrix. We further leverage the transferability of VNNs to cross validate the aforementioned observations across datasets of different dimensionalities.

    
\end{abstract}
\blfootnote{Under review. Portions of this manuscript have appeared in~\cite{sihag2022covariance} and~\cite{sihag2022predicting}. See Section~\ref{data_code} for information regarding data and code availability.}

\section{Introduction}


In various modern applications, the number of features (denoted by $m$) in a dataset is a fundamental component of data acquisition that is typically a characteristic of the desired application and logistics involved~\cite{liu2015survey, brinkmann2009large}. Most machine learning algorithms and statistical inference approaches are designed over a pre-defined feature set and hence, their computational and sample complexities inherently depend on the dimensionality~$m$~\cite{cai2018feature,ng2004feature}.  In this paper, we study a deep learning framework called coVariance neural networks (VNN)~\cite{sihag2022covariance} that is based on graph neural networks operating on sample covariance matrix from a given dataset and is scale-free, i.e., the number of learnable parameters in VNN is independent of the dimensionality $m$ of the dataset (Fig.~\ref{vnn_archit}). The scale-free aspect of VNNs makes it feasible for them to be \emph{transferable} between datasets of different dimensionalities without any changes to their architecture, i.e., VNNs trained on a dataset with dimensionality $m = m_1$ can process another dataset with dimensionality $m=m_2$ with the same set of learned parameters.

While a larger number of features $m$ in a dataset may imply higher resolution or quality of data collected, too many features can lead to challenges related to storage, computational complexity, and interpretability of statistical models for effective data analysis. On the other hand, a dataset with too few features may be devoid of enough relevant information for accuracy and inference. Nevertheless, one can intuitively expect some correspondence between a dataset consisting of $m = m_1$ features and another dataset consisting of $m = m_2$ features if both sets of features describe a similar phenomenon at different scales or resolutions. Conventional data analysis approaches (e.g., principal component analysis) and machine learning algorithms are unable to exploit or accommodate this aspect of similarity between datasets when the number of learnable parameters for inference is determined by the dimensionality $m$ of the dataset. Hence, such statistical models need to be re-designed from scratch if the data dimensionality changes. In this context, we provide the theoretical conditions on the covariance matrices under which the performance achieved by a VNN model for an inference task over one dataset can be transferred to that over another dataset with different dimensionality without re-training or changes to the VNN model (Fig.~\ref{vnn_transfer_overview}). 
Besides the methodological gains over traditional data analysis approaches, VNN frameworks also offer advantages in managing computational complexity. Indeed, under appropriate conditions, a model trained on a lower feature count dataset can be directly applied for inference from a higher feature count dataset.



Neuroimaging is an example of a modern application in which the number of features is highly variable across datasets, but different datasets contain similar information~\cite{yang2022data,markello2022neuromaps}. Specifically, MRI data can be represented in many scales ranging from single voxels (typically~$\sim1\text{ mm}^3$) to regions-of-interest (ROIs) derived from multi-scale brain atlases that range from dozens to thousands of parcellations (e.g., from 100 to 1000 number of parcellations in a multi-scale brain atlas~\cite{schaefer2018local, hagmann2008mapping}). Multi-scale brain atlases provide mappings that divide the brain cortex into different number of parcellations and therefore, associated neuroimaging datasets describe similar information over the brain cortex at different scales. Hence, we empirically validate the theoretical results for transferability of VNNs on a set of neuroimaging datasets curated according to different scales of a commonly used brain atlas.  




The simplicity of VNN framework allows us to analyze the contributions of each feature to the final statistical outcome of a VNN model. This aspect is significant for applications that use neuroimaging datasets, as the features therein are associated with distinct brain regions. Hence, VNNs make it feasible to assign anatomical or regional \emph{interpretability} to the learning outcomes when trained on datasets consisting of anatomical features. Motivated by these observations, we pursue the task of ``brain age" prediction from cortical thickness features derived from magnetic resonance imaging (MRI) as an application of VNNs (Fig.~\ref{brain_age_overview}). Critically, individuals may experience age-related effects at different rates, captured by so-called ``biological aging". Hence, accelerated aging (e.g., when biological age is elevated as compared to chronological age) may predict age-related vulnerabilities like risk for cognitive decline or neurological conditions like Alzheimer's disease and related dementias (ADRD) ~\cite{habes2016advanced}.   In this domain, the metric of interest is \emph{brain age gap}, i.e., the difference between the biological age and the chronological age. We use the notation $\Delta$-Age to refer to the brain age gap. Since brain age has no ground truth, $\Delta$-Age is essentially a qualitative metric that is expected to be elevated in individuals with underlying neurodegenerative condition as compared to the healthy population. Numerous existing studies based on a large spectrum of machine learning approaches report elevated $\Delta$-Age in neurodegenerative conditions, including Alzheimer's disease~\cite{franke2012longitudinal} and schizophrenia~\cite{hajek2019brain}.  However, several criticisms for brain age evaluation approaches using machine learning have also been identified. Major criticisms include the coarseness of $\Delta$-Age that results in lack of specificity of brain regions contributing to elevated brain age; and unexplained reliance on the prediction accuracy for chronological age in the design of these machine learning models~\cite{cole2017predicting}.  The architecture of VNN enables us to propose a principled framework for brain age prediction that accommodates interpretability by isolating the contributing brain regions to elevated $\Delta$-Age in neurodegeneration. These contributing brain regions are identified by analyzing the group differences between the group with the neurodegenerative condition and healthy controls with respect to the contributions of different features to the corresponding predictions made by VNNs (Fig.~\ref{brain_age_overview}). A layman overview of the advantage offered by VNNs in terms of adding interpretability to brain age prediction is included in Appendix~\ref{layman_vnn}.  

Together, the interpretability and transferability aspects of VNNs provide novel insights into the role of training the model to predict chronological age in the brain age prediction framework. A significant portion of existing literature that studies brain age using deep learning models considers the ability of their models to accurately predict chronological age (time since birth) for healthy controls~\cite{yin2023anatomically,bashyam2020mri,couvy2020ensemble} as a relevant metric for assessing quality of their methodological approach. Simultaneously, deep learning models that have a relatively moderate fit on the chronological age of healthy controls can also provide better insights into brain age than the ones with a tighter fit~\cite{bashyam2020mri}.  
Thus, there is a lack of conceptual clarity in the role of the quality of fit against chronological age of healthy controls in predicting a meaningful brain age~\cite{butler2021pitfalls}, and this issue is unlikely to be addressed by any statistical approach for evaluating brain age if it lacks transparency (Appendix~\ref{layman_vnn}).

It is reasonable to postulate that an accurate or a near-perfect prediction of chronological age in healthy controls does not necessarily equip the deep learning models to accommodate neurodegeneration-driven changes in the brain. 
The study of interpretability and transferability aspects of VNNs demonstrates that the contributing brain regions behind elevated $\Delta$-Age in neurodegeneration may be qualitatively recovered after transferring VNNs from one dataset to another. Moreover, we also find that the ability of VNNs to exploit specific eigenvectors of the anatomical covariance matrix is the underlying factor behind the anatomical interpretability offered by VNNs in the context of $\Delta$-Age.  Hence, VNNs can facilitate a principled decoupling of the objective of inferring brain age from the aim of achieving a near-perfect performance on the task of predicting chronological age for healthy controls. 

\noindent
{\bf Contributions:}
Our contributions in this paper are summarized as follows:
\begin{itemize}
    \item[-] {\bf Transferability of VNNs:} We theoretically characterize the transferability of VNNs between datasets of different dimensionalities. For a dataset with $m_1$ features and covariance matrix $\bC_{m_1}$ and another dataset with $m_2$ features and covariance matrix $\bC_{m_2}$,  we demonstrate that the outputs of a VNN when initialized on  $\bC_{m_1}$ and $\bC_{m_2}$ are close in some sense under appropriate conditions on covariance matrices $\bC_{m_1}$ and $\bC_{m_2}$ (see Theorem~\ref{transferthm}). The theoretical results on transferability of VNNs were validated on a regression task based on a set of cortical thickness datasets curated according to different scales of a commonly used multi-scale brain atlas (Fig.~\ref{vnn_tranfer_fig} and Tables~\ref{transfer_tbl1} and~\ref{transfer_tbl2}). 
    \item[-] {\bf Brain age prediction using VNNs:}
    We deployed VNNs for the task of brain age prediction and compared the $\Delta$-Age between healthy controls and individuals with AD diagnosis. The insights gained in this set of experiments are summarized below:
    \begin{itemize}
        \item[a)] {\bf VNNs provide anatomically interpretable $\Delta$-Age in neurodegeneration:} $\Delta$-Age in individuals with AD diagnosis was elevated as compared to healthy controls. The simplicity of the VNN architecture allowed us to characterize the regional contributors to the elevated $\Delta$-Age, thus, adding anatomical interpretability to $\Delta$-Age (Fig.~\ref{roi_AD}). On a multi-scale dataset, the transferability of VNNs also helped cross-validate the spatial robustness of the observed regional profiles for $\Delta$-Age in neurodegeneration (Fig.~\ref{bvftd_vnn} in Appendix~\ref{explor}). 
        \item[b)] {\bf Anatomical interpretability was correlated with eigenvectors of the anatomical covariance matrix:} Our experiments demonstrated that there was a correlation between specific eigenvectors of the anatomical covariance matrix and the features that facilitated anatomical interpretability for $\Delta$-Age  (Fig.~\ref{corr_eig_plots_oasis}). Thus, $\Delta$-Age was linked to the ability of VNNs to exploit specific eigenvectors of the anatomical covariance matrix. 
        \item[c)] {\bf Clarity in the role of training VNNs to predict chronological age:} Our experiments conclusively showed that training the VNNs to predict chronological age enabled them to exploit the eigenvectors of the anatomical covariance matrix associated with elevated $\Delta$-Age in neurodegeneration (Fig.~\ref{corr_eig_plots_ftdc_oasis}, Fig.~\ref{oasis_ftdc_transfer}, and Fig.~\ref{inner_prod_reg_profile}). The observations made in this context facilitated decoupling of the brain age prediction task from the objective of achieving a near perfect performance on the prediction of chronological age in healthy controls.
    \end{itemize}
\end{itemize}

\noindent
Next, we provide a literature review pertinent to our contributions in this paper. 




\noindent
{\bf Related Literature.}
Graph neural networks (GNNs) are a widely popular adaptation of convolutional neural networks to graph-structured data~\cite{zhou2020graph, scarselli2008graph, wu2020comprehensive}. Graphs are natural descriptors of complex, spatially-distributed phenomena and therefore, graph-structured datasets are prevalent in a variety of application domains~\cite{sahu2017ubiquity}, including physical infrastructure~\cite{quintana1982power}, social network analysis~\cite{myers2014information}, biology~\cite{fout2017protein}, network neuroscience~\cite{sporns2022graph} and natural sciences~\cite{sanchez2018graph}. Processing of graph-structured data faces various practical challenges (like generalizability, reproducibility, scalability)~\cite{sahu2017ubiquity} and therefore, a number of variants of GNNs have been proposed to address them. We refer the reader to recent survey articles in~\cite{wu2020comprehensive} and~\cite{abadal2021computing} that categorize GNN architectures according to diverse criteria, including mathematical formulations, algorithms, and hardware/software implementations. Convolutional GNNs~\cite{zhang2019graph}, graph autoencoders~\cite{kipf2016variational}, recurrent GNNs, and gated GNNs~\cite{li2016gated} are among a few prominently studied and applied categories of GNNs. 

The taxonomy pertinent to this paper is that of graph convolutional networks~\cite{zhang2019graph}. GCNs typically rely on an information aggregation procedure (referred to as graph convolutions) over a graph structure for data processing. Several implementation strategies for graph convolution operations have been proposed in the literature, including spectral convolutions~\cite{bruna2013spectral}, Chebyshev polynomials~\cite{hammond2011wavelets}, ordinary polynomials~\cite{gama2020graphs}, and diffusion based representations~\cite{atwood2016diffusion}.
GCNs admit the properties of stability to topological perturbations and transferability across graphs of different sizes in various settings~\cite{verma2019stability,keriven2020convergence,gama2020stability, ruiz2020graphon}, which makes them a well-motivated data analysis tool for graph-structured data.  


In our recent work in~\cite{sihag2022covariance}, we studied coVariance neural networks (VNN), which are GCNs with sample covariance matrices as graph and polynomial graph filters as convolution operation. Covariance matrices and principal component analysis (PCA) form the two cornerstones of non-parametric analyses in real world applications that have spatially distributed, multi-variate data acquisition protocols, including neuroimaging~\cite{evans2013networks}, computer vision~\cite{murase1994illumination,de2001robust}, weather modeling~\cite{stephenson1997correlation}, traffic flow analysis~\cite{shao2014estimation}, and cloud computing~\cite{ismail2013detecting}. Our results in~\cite{sihag2022covariance}  established the following significant observations i) there exist similarities between the spectral analysis of graph convolution on covariance matrix and the standard PCA transformation;  and ii) VNNs are robust to the number of samples used to estimate the sample covariance matrix, thus, overcoming a potential source of instability and irreproducibility of PCA based statistical inference~\cite{joliffe1992principal,elhaik2022principal}. 


The transferability of GNNs from training graphs to some compatible family of test graphs has been previously studied from different theoretical perspectives~\cite{levie2021transferability,ruiz2020graphon,zhu2021transfer,maskey2023transferability}. The notion of transferability of GNNs broadly encapsulates the intuition that GNNs may be able to retain their performance for some inference task when applied over test graphs (irrespective of the size) that describe the same phenomenon as the training graphs. In this context, the study in~\cite{levie2021transferability} considers transferability of GNNs over graphs that represent the discretization of underlying topological spaces. Several studies also consider GNN transferability over graphs that belong to a converging sequence that approaches a limiting object in the asymptote of a large number of nodes~\cite{ruiz2020graphon,maskey2023transferability}. In~\cite{zhu2021transfer}, the similarity between the ego graph distributions (derived from graph topology) formed the workhorse for assessing transferability of GNNs. Transferability of GNNs also provides advantages in terms of computational complexity, which for GNNs scales as ${\cal O}(m^2)$ for dense graphs with $m$ nodes. In this paper, we extend the notion of transferability to VNNs and establish the transference over covariance matrices of different sizes that converge in some sense. In this context, transferability is not feasible for traditional PCA-driven statistical models, as the principal components are restricted within the feature space of the original dataset and need to be re-evaluated if the number of features change. Specifically, PCA does not provide any notion of similarity between the principal components extracted from a covariance matrix of size $m_1\times m_1$ and that from another covariance matrix of size $m_2\times m_2$. Thus, transferability of VNNs is broadly relevant to the domain of multivariate statistics. 


There has been a growing interest in multi-scale datasets in neuroscience~\cite{zeighami2021association,luppi2021combining, royer2022open,markello2022neuromaps,khundrakpam2015prediction}. These datasets rely on brain atlases or templates that allow a multi-scale parcellation of the brain surface (for instance,  Schaefer's atlas~\cite{schaefer2018local} and Lausanne atlas~\cite{hagmann2008mapping}). A multi-scale brain atlas partitions the brain cortex into a variable number of regions at different scales. However, statistically sound approaches that optimally leverage the redundancy of information in datasets consisting of features at multiple scales are currently lacking. In this paper, we leverage the cortical thickness datasets curated according to multi-scale brain atlases to validate the transferability of VNNs. 
In scenarios where the theoretical guarantees for VNN transferability do not hold between datasets curated according to different brain atlases, appropriate mappings that account for the differences between different brain atlases~\cite{yang2022data,markello2022neuromaps} may be necessary to scale VNN outputs for good quantitative performance on the chronological age prediction task when VNNs are transferred from one dataset to another. However, these mappings are unlikely to possess any information regarding neurodegeneration or brain age and hence, are not studied here.

The human aging process is characterized by progressive anatomical and functional changes in the brain~\cite{lopez2013hallmarks}. The $\Delta$-Age for a pathology can be seen as a scalar representation of longitudinal, pathology-driven atypical changes in the brain~\cite{franke2012longitudinal}.  Data from neuroimaging modalities, including structural magnetic resonance imaging (MRI), functional MRI, and positron emission tomography, capture the changes of the brain due to neurodegeneration and healthy aging~\cite{frisoni2010clinical,chen2019functional}. Thus, inferring brain age from different neuroimaging modalities has been an active area of research~\cite{lee2022deep, yin2022deep,millar2022multimodal,dafflon2020automated,baecker2021brain,aycheh2018biological}. In this paper, we leverage the datasets of cortical thickness measures derived from structural MRI images in OASIS-3 dataset to study brain age~\cite{lamontagne2019oasis}. Cortical thickness evolves with normal aging~\cite{mcginnis2011age} and is impacted due to neurodegeneration in AD~\cite{lehmann2011cortical,sabuncu2011dynamics}. Thus, the age-related and disease severity related variations also appear in anatomical covariance matrices evaluated from the correlation among the cortical thickness measures across a population~\cite{evans2013networks}.

Due to inherent interpretability offered by VNNs to their statistical outcomes, VNNs provide an explainable regional profile to the elevated $\Delta$-Age in neurodegeneration. Limited focus has been on comparable studies in this regard that associate brain age gaps with regional profiles~\cite{yin2023anatomically,popescu2021local}. The study in~\cite{yin2023anatomically} adopts a convolutional neural network approach to infer brain age from MRI images directly and assigns importance to brain regions in evaluating the brain age. In principle, the interpretability offered by VNNs in the context of brain age is similar, as we infer a regional profile for $\Delta$-Age by isolating the brain regions that are contributors to the elevated $\Delta$-Age in neurodegeneration. In addition, the regional profile identified by VNNs is correlated with specific eigenvectors or the principal components of the anatomical covariance matrix. Hence, the $\Delta$-Age inferred by our framework is driven by the ability of a VNN to manipulate the input data according to certain principal components of the anatomical covariance matrix. Also, VNNs are significantly less complex deep learning models as compared to those studied in~\cite{yin2023anatomically}. Our results demonstrate that the VNNs trained with less than $2000$ learnable parameters exhibit (spatially robust) regional interpretability in the context of brain age in AD. In general, the regional expressivity offered by VNNs is in stark contrast to a multitude of existing relevant studies that rely on less transparent statistical approaches and further use post-hoc analyses (such as ablation analysis~\cite{feng2020estimating, feng2018deep,essemlali2020understanding} or exploring correlations with region-specific markers~\cite{lee2022deep} and psychiatric symptoms~\cite{karim2021aging,liem2017predicting}) to assign interpretability to a scalar, elevated $\Delta$-Age effect. 

\noindent

\section{coVariance Neural Networks}\label{vnn_intro}
VNNs operate on covariance matrices and have similar architecture as GNNs \footnote{GCNs and GNNs are used interchangeably in the rest of the paper.}. We start by providing preliminary definitions pertaining to the architecture and discuss the theoretical properties associated with VNNs later.

Consider an $m-$dimensional random vector $\bx \in \mR^{m \times 1}$ whose ensemble covariance matrix is defined as 
\begin{align}
    \bC \dff \mE[(\bx-\mE[\bx])(\bx-\mE[\bx])^{\sf T}]\;,
\end{align}  
where $\cdot^{\sf T}$ is the transpose operator and $\mE[\cdot]$ is the expectation with respect to the probability distribution of~$\bx$. In practice, we usually have access to a dataset that provides us with the statistical information about~$\bx$. Therefore, we also consider a dataset consisting of $n$ random, independent and identically distributed (i.i.d) samples of $\bx$, given by~$\bx_i \in \mR^{m \times 1},\forall i\in\{1,\dots,n\}$, where the dataset can be represented in the matrix form as $\bX_n = [\bx_1,\dots,\bx_n]$. Using $\bX_n$, we estimate the ensemble covariance matrix, conventionally referred to as the sample covariance matrix as follows
\begin{align}\label{sample_cov}
    \hat \bC \triangleq \frac{1}{n-1} \sum\limits_{i=1}^n(\bx_i - \bar\bx) (\bx_i-\bar\bx)^{\sf T}\;,
\end{align}
where $\bar\bx$ is the sample mean of samples in $\bX_n$. Next, we discuss the motivation behind studying VNNs separately from GNNs. 
\subsection{Motivation}\label{mtvn}
Covariance matrices are ubiquitous in real world applications that have spatially distributed, multi-variate data acquisition protocols~\cite{evans2013networks,stephenson1997correlation,shao2014estimation,ismail2013detecting}. The eigenvectors of covariance matrices are termed as principal components of the dataset and constitute the well-known PCA transformation~\cite{shlens2014tutorial}. Our discussion in this section shows that the spectral domain representation of the graph convolution operator instantiated on the covariance matrix as graph yields the PCA transformation. This observation suggests that learning with a GNN with covariance matrix as a graph is achieved (at least in part) by manipulation of the data according to the eigenbasis of the covariance matrix. Therefore, this paper provides a conceptual contribution towards the study of GNNs instantiated on covariance matrices in the form of VNNs. We also provide theoretical guarantees that result in significant advantages over models that perform statistical inference using PCA. 

The covariance matrix $\bC$ can be viewed as the adjacency matrix of a graph representing the stochastic structure of the vector $\bx$, where the $m$ dimensions of $\bx$ can be thought of as the nodes of an $m$-node, undirected graph and its edges represent the pairwise covariance between elements in $\bx$. Furthermore, the eigenvalues of $\bC$ encode the variability of the dataset along different directions in an orthogonal space determined by the associated eigenvectors or principal components.

In graph signal processing, a vector defined on the nodes of the graph is viewed as the graph signal and the projection of a graph signal on the eigenbasis of the graph yields the graph Fourier transform~\cite{ortega2018graph}. The graph Fourier transform provides a systematic mathematical tool to analyze convolutional filters over graphs~\cite{zhang2019graph, shuman2013emerging}. Interestingly, the classical Fourier transform and graph Fourier transform converge over a discrete, periodic time series represented on a directed, cyclic graph~\cite{sandryhaila2013discrete}. Similarly to the graph Fourier transform, we can define the coVariance Fourier transform as the  projection of a random instance $\bx$\footnote{For ease of notation, we will subsequently use the notation $\bx$ to refer to a random instance of the random vector whose covariance matrix is $\bC$.} on the eigenvectors of the covariance matrix $\bC$~\cite[Definition 1]{sihag2022covariance}. The definition of coVariance Fourier transform from~\cite{sihag2022covariance} is stated next. For this purpose, we leverage the eigendecomposition of $\bC$ given by
\begin{align}\label{sample_eig}
   \bC = \bV \bLambda \bV^{\sf T}\;,
\end{align}
where $\bV = [\bv_1,\dots,\bv_m]$ is a matrix of size $m\times m$ with its columns as the eigenvectors and $\bLambda= {\sf diag}(\lambda_1,\dots,\lambda_m)$ is a diagonal matrix with its diagonal elements representing the eigenvalues of $\bC$.  
\begin{definition} [coVariance Fourier Transform]
The coVariance Fourier transform (VFT) of a random sample $\bx$ is defined as its projection on the eigenspace of $\bC$ and is given by
\begin{align}\label{vft}
    \tilde\bx \dff \bV^{\sf T} \bx\;.
\end{align}
\end{definition}
\noindent 
The $i$-th entry of $\tilde\bx$, i.e., $[\tilde\bx]_i$ represents the projection of $\bx$ on eigenvector $\bv_i$ and hence, it is associated with the eigenvalue~$\lambda_i$. Thus, the similarity between PCA transformation and VFT in~\eqref{vft} implies that eigenvalue~$\lambda_i$ encodes the variability of the dataset $\bX_n$ in the direction of the principal component $\bv_i$. In this context, the eigenvalues of the covariance matrix are the mathematical equivalent of the notion of graph frequencies in graph signal processing~\cite{ortega2018graph}.

GNNs with convolutional filters that rely on a \emph{linear shift-and-sum} operation fundamentally exhibit the stability to changes in graph topology~\cite{gama2020stability}. Since the sample covariance matrix $\hat\bC$ is likely to be perturbed with respect to $\bC$~\cite{loukas2017close}, stability is desirable to mitigate the impact of number of samples on statistical inference. Motivated by this observation, we define the notion of coVariance filters (VF) that are polynomials in the covariance matrix and characterize the convolution operation in VNNs. 



\begin{definition}[coVariance Filters]\label{def1}
Given a set of real valued, scalar parameters $\{h_k\}_{k=0}^K$, the coVariance filter on a covariance matrix $\bC$ is defined as
\begin{align}
    \bH(\bC)\dff \sum\limits_{k=0}^K h_k\bC^k\;.
\end{align}
Furthermore, the output of the covariance filter $\bH(\bC)$ for an input $\bx$ is given by 
\begin{align}
\bz = \bH(\bC)\bx\;.
\end{align}
\end{definition}
\noindent
The application of coVariance filter $\bH(\bC)$ on an input $\bx$ translates to combining information across different sized neighborhoods. To elucidate this observation, consider a coVariance filter with $K=1$ and $h_0 = 0$. In this scenario, the $i$-th element of $\bz$ is evaluated as 
\begin{align}\label{exconv}
    [\bz]_i = h_1 \sum\limits_{j=1}^m[\bC]_{ij} [\bx]_j\;.
\end{align}
Thus, $[\bz]_i$ represents the linear combination of all elements according to the $i$-th row of $\bC$. This observation implies that the neighborhood of $i$-th dimension in $\bx$ (derived from the graph representation $\bC$) determines the outcome $[\bz]_i$ of the convolution operation in~\eqref{exconv}. For $K>1$, the convolution operation combines information across multi-hop neighborhoods (up to $K$-hop) according to the weights $h_k$.  

The spectral analysis of the covariance filtering operation in Definition~\ref{def1} via VFT of the filter output~$\bz$ yields the frequency response of the covariance filter and reveals the similarities between covariance filtering and PCA. After taking the VFT of~$\bz$, we have
\begin{align}
    \tilde \bz &= \bV^{\sf T} \bH(\bC)\bx\;,\label{eq1} \\
    &= \sum\limits_{k=0}^K h_k\bLambda^k \bV^{\sf T}\bx = \sum\limits_{k=0}^K h_k\bLambda^k \tilde\bx\;,\label{eq2}
\end{align}
where $\tilde\bx = \bV^{\sf T}\bx$ is the covariance Fourier transform of $\bx$
and~\eqref{eq2} follows from~\eqref{eq1} from the orthonormality of eigenvectors of $\bC$. The frequency response of the coVariance filter depends on the filter taps $\{h_k\}$ and the eigenvalues of $\bC$ and is given by
\begin{align}\label{vvf}
    h(\lambda) = \sum\limits_{k=0}^K h_k \lambda^k\;.
\end{align}
Furthermore, since $\tilde\bx$ is a projection of $\bx$ on the eigenvector space $\bV$ and $[\tilde \bz]_i = h(\lambda_i) [\bV^{\sf T}\bx]_i$, the $i$-th element of $\tilde\bz$ yields similarities with the standard PCA transformation. This observation is formalized in Lemma~\ref{lm1}. 
\begin{lemma}[Spectrum of coVariance Filter and PCA]\label{lm1}
Given a covariance matrix $\bC$ with eigendecomposition in~\eqref{sample_eig}, if the PCA transformation of input $\bx$ is given by $\by = \bV^{\sf T} \bx$, there exists a filter bank of coVariance filters $\{\bH_i(\bC): i\in \{1,\dots,m\}\}$, such that, the score of the projection of input $\bx$ on eigenvector~$\bv_i$ can be recovered by the application of a coVariance filter $\bH_i(\bC)$ as:
\begin{align}
    [\by]_i = \bv_i^{\sf T}  \bH_i(\bC) \bx \;, 
\end{align}
where the frequency response $h_i(\lambda)$ of the filter $\bH_i(\bC)$ is given by
\begin{equation}
    h_i(\lambda) = \begin{cases}
     \omega_i ,\quad \text{if} \quad \lambda =\lambda_i  \;,\\
     0, \quad \text{otherwise}
    \end{cases}\;.
\end{equation}
\end{lemma}
\noindent
Lemma~\ref{lm1} establishes equivalence between processing data samples with PCA and processing data samples with a specific polynomial on the covariance matrix.

Our previous work in~\cite{sihag2022covariance} showed that in contrast to PCA involving eigenvectors of the sample covariance matrix,  information processing with polynomials of the sample covariance matrix can be stable to finite sample induced perturbations. Indeed, in practice, we may only have access to 
the sample covariance matrix $\hat\bC$ which is an estimate of $\bC$. Since $\hat\bC$ is a consistent estimator of $\bC$, the eigenvalues and eigenvectors of $\hat\bC$ approach those of $\bC$ in the limit of infinite number of samples, i.e., $n\rightarrow\infty$. However, for a finite number of samples $n$, the eigenvectors and eigenvalues of $\hat\bC$ are perturbed with respect to those of $\bC$~\cite{loukas2017close}. In principle, statistical inference using PCA can be prone to instability due to eigenvectors corresponding to eigenvalues of the ensemble covariance matrix that are close~\cite{joliffe1992principal} and, thus, lead to irreproducible statistical conclusions~\cite{elhaik2022principal}. In this context, we informally state Theorem~2 from~\cite{sihag2022covariance}.
\begin{theorem}[Stability of coVariance filter]\label{vfstability}
    Consider a dataset with  sample covariance matrix $\hat\bC$ formed by $n$ samples and the counterpart ensemble covariance matrix $\bC$. Under Assumption 1 in~\cite{sihag2022covariance},  the following holds with high probability:
\begin{align}\label{filterstab_rslt2}
    \left\lVert \bH(\hat\bC) - \bH(\bC)\right\rVert = {\cal O}\left(\frac{\nu}{n^{\frac{1}{2} - \varepsilon}}\right)\;,
\end{align}
for some $\nu>0$ and $\varepsilon\in (0,1/2)$. 
\end{theorem}

\noindent
Theorem~\ref{vfstability} establishes that information processing using a polynomial of the covariance matrix offers stability with respect to the perturbations between the sample covariance matrix $\hat\bC$ and $\bC$~\cite{sihag2022covariance}. Also, as a corollary to Theorem~\ref{vfstability}, we can state that the difference between outputs of covariance filters instantiated on distinct sample covariance matrices are bounded. These observations imply that statistical inference based on covariance filters are characterized by robustness to the effects of finite sample size and, thus, result in consistent statistical outcomes with high confidence. No such guarantees are offered by PCA. Next, we discuss the architecture of VNNs that is based on covariance filters,  which results in VNNs inheriting the stability offered by coVariance filters. 
\subsection{Architecture}
We begin with the description of a coVariance perceptron that forms one layer of the VNN architecture. For this purpose, we leverage the definition of a pointwise, nonlinear activation function $\sigma(\cdot)$, such that, for $\bx = [x_1,\dots,x_m]$, we have $\sigma(\bx) = [\sigma(x_1),\dots,\sigma(x_m)]$. Examples of point-wise, nonlinear activation functions are ${\sf ReLU}$ and $\sf tanh$.


\begin{definition}[coVariance Perceptron]\label{def_cov}
For a given pointwise nonlinear activation function $\sigma(\cdot)$, input $\bx$, a coVariance filter~$\bH(\bC)$ and its corresponding set of filter taps $\cH$, the coVariance perceptron is defined as
\begin{align}\label{1l}
 \Phi(\bx; \bC, \cH) \dff \sigma(\bH(\bC)\bx)\;.
\end{align}
\end{definition}
\noindent
A VNN can be constructed by stacking perceptrons to form multi-layer information processing architecture. This observation is formalized next.

\begin{remark}[Multi-layer VNN]
 Consider an $L$-layer architecture formed by stacking $L$ coVariance perceptrons defined in Definition~\ref{def_cov}. In this scenario, we denote the coVariance filter in layer $\ell$ of a VNN by~$\bH_{\ell}(\bC)$ and its corresponding set of filter taps are given by $\cH_{\ell}$. For a given pointwise nonlinear activation function $\sigma(\cdot)$, the relationship between the input $\bx_{\ell-1}$ and the output $\bx_{\ell}$ for the coVariance perceptron in the $\ell$-th layer is given by
\begin{align}
    \bx_{\ell} = \sigma(\bH_{\ell}(\bC)\bx_{\ell-1})\; \quad\text{for }\quad  \ell\in \{1,\dots,L\},
\end{align}
where $\bx_0$ is the input $\bx$. We refer to this $L$-layer architecture as an $L$-layer VNN. 
 \end{remark}

 \begin{figure}[!htbp]
  \centering
  \includegraphics[scale=0.4]{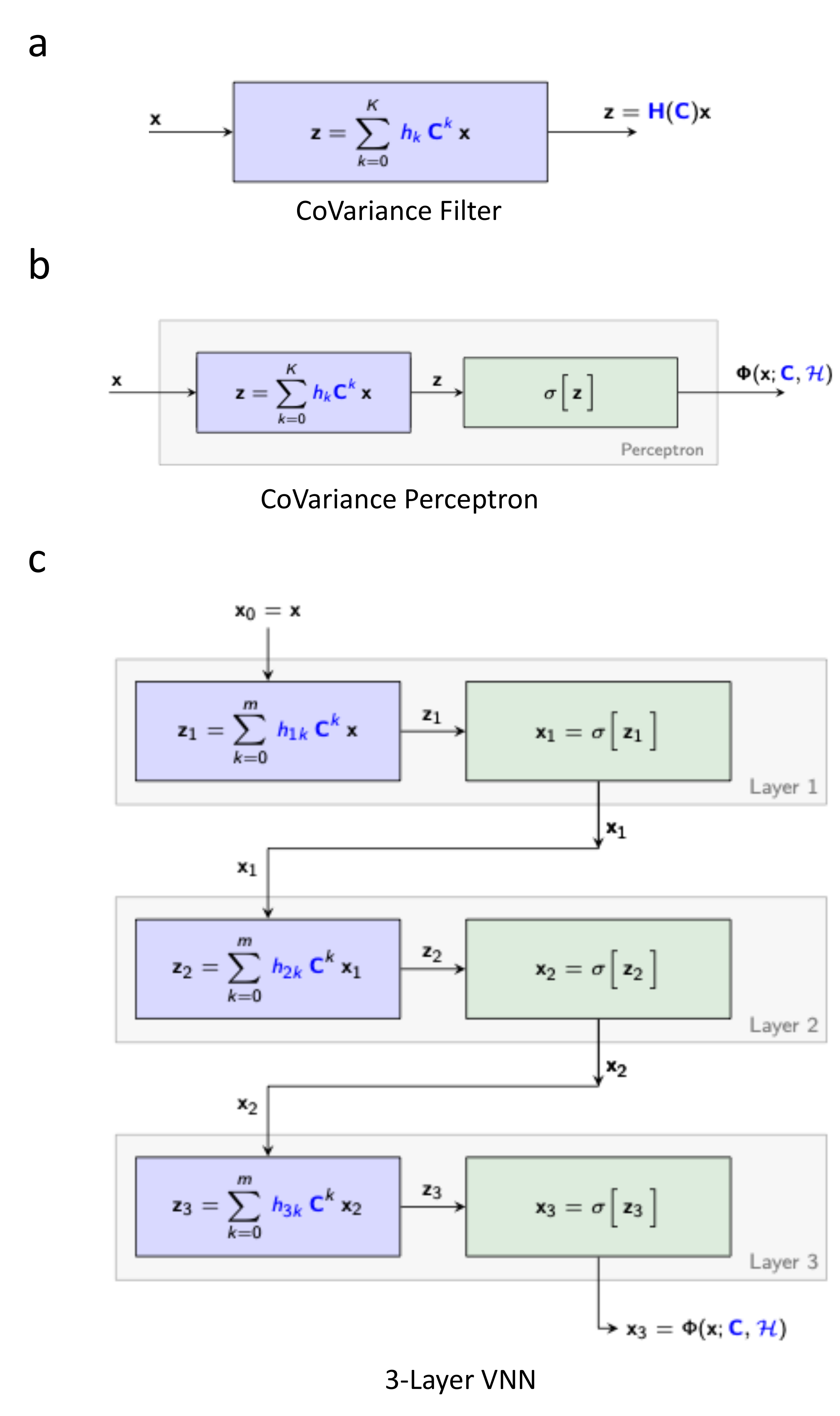}
   \caption{{\bf Basics of VNN architecture.} Panel~{\bf a} illustrates that the coVariance filter $\bH(\bC)$ is a polynomial in $\bC$ and its application on input $\bx$. Panel~{\bf b} shows the construction of a coVariance perceptron based on coVariance filter $\bH(\bC)$ and pointwise nonlinearity $\sigma$. coVariance perceptron specifies one layer of VNN. Panel~{\bf c} shows a basic multi-layer VNN architecture formed by stacking three coVariance perceptrons. }
   \label{vnn_archit}
\end{figure}
\noindent
A pictorial illustration of a multi-layer VNN is included in Fig.~\ref{vnn_archit}. Furthermore, similar to other deep learning models, sufficient expressive power can be facilitated in the VNN architecture by incorporating multiple input multiple output (MIMO) processing at every layer. Formally, consider a VNN layer $\ell$ that can process $F_{\ell-1}$ number of $m$-dimensional inputs and outputs $F_{\ell}$ number of $m$-dimensional outputs via $F_{\ell-1} \times F_{\ell}$ number of filter banks~\cite{gama2020stability}. In this scenario, the input is specified as $\bX_{\sf in} = [\bx_{\sf in}[1],\dots,\bx_{\sf in}[F_{\sf in}]]$, and the output is specified as $\bX_{\sf out} = [\bx_{\sf out}[1],\dots,\bx_{\sf out}[F_{\sf out}]]$. The relationship between the $f$-th output $\bx_{\sf out}[f]$ and the input $\bx_{\sf in}$ is given by
\begin{align}
   \bx_{\sf out}[f] &=  \sigma\left(\sum\limits_{g = 1}^{F_{\sf in} } \bH_{fg}(\bC)\bx_{\sf in} [g] \right)\label{vfbnk}\;,
\end{align}
where $\bH_{fg}(\bC)$ is the coVariance filter that processes $\bx_{\sf in}[g]$. Without loss of generality, we focus the subsequent discussion on the scenario when we have $F_{\ell} = F,\forall \ell \in \{1,\dots,L\}$. In this case, the set of all filter taps is given by  ${\cal H} = \{\cH_{fg}^{\ell}\}, \forall f,g \in \{1,\dots, F\}, \ell \in \{1,\dots,L\}$, where $\cH_{fg} = \{h_{fg}^{\ell}[k]\}_{k=0}^K$ and $h_{fg}^{\ell}[k]$ is the $k$-th filter tap for filter $\bH_{fg}(\bC)$. Thus, we can compactly represent a multi-layer VNN architecture capable of MIMO processing via the notation $\Phi(\bx;\bC,{\cal H})$ as the set of filter taps $\cH$ captures the full span of its architecture. We also use the notation $\Phi(\bx;\bC,{\cal H})$ to denote the output of the VNN. For a VNN with $F$ number of $m$-dimensional outputs in the final layer, the size of the VNN output $\Phi(\bx;\bC,{\cal H})$ is $m\times F$.

The output $\Phi(\bx;\bC,{\cal H})$ is succeeded by a readout function that maps $\Phi(\bx;\bC,{\cal H})$ to the desired output.  In this paper, we assume non-adaptive or non-learnable readout function (e.g., mean, max or min functions) which preserves the property of permutation invariance for VNN model. Furthermore, a non-adaptive readout function is essential for the transferability property of VNNs (discussed in Section~\ref{trnsfrr}).

It is imperative to study the robustness of VNN outputs to the number of samples $n$ in order to guarantee reproducibility of VNN statistical outcomes. Specifically, it is desirable that the change in VNN outputs is controlled or bounded when the architecture is trained using sample covariance matrices estimated from $n_1$ or $n_2$ samples when $n_1\neq n_2$. In Theorem~\ref{thm_stability}, we informally state the result on the stability of VNNs by analyzing $\|\Phi(\bx;\hat\bC,\cH) - \Phi(\bx;\bC,\cH)\|$, i.e., the difference between the VNN outputs for the sample covariance matrix $\hat\bC$ and the ensemble covariance matrix $\bC$. This Theorem was also previously established in~\cite{sihag2022covariance}.
\begin{theorem}[VNN Stability]\label{thm_stability}
Consider an ensemble covariance matrix $\bC$ and its estimate $\hat\bC$ formed from $n$ samples. Given a bank of coVariance filters with filter taps $\cH = \{\cH_{fg}^{\ell}: f\in \{1,\dots, F\}, \ell\in \{1,\dots,L\}\}$, the coVariance filters are stable and satisfy
\begin{align}\label{alpha_n}
    \|\bH_{fg}^{\ell}(\hat\bC)- \bH_{fg}^{\ell}(\bC)\| \leq \alpha_n \;,
\end{align}
for some $\alpha_n > 0$ with high probability (Theorem~\ref{vfstability}). Also, for a pointwise nonlinearity function $\sigma(\cdot)$, such that, $|\sigma(a) - \sigma(b)|\leq |a-b|$, we have
\begin{align}\label{vnn_stab}
    \|\Phi(\bx;\hat\bC, \cH)-\Phi(\bx;\bC,\cH)\| \leq LF^{L} \alpha_n\;.
\end{align}
\end{theorem}
\begin{proof}
    See Appendix~\ref{pf_thm1} in the Supplementary Material. 
\end{proof}
\noindent 
The parameter $\alpha_n$ in~\eqref{alpha_n} represents the finite sample effect on the perturbations in $\hat\bC$ with respect to $\bC$ and is borrowed from Theorem~\ref{vfstability}. By leveraging the perturbation theory of covariance matrices to analyze the stability of coVariance filters, we also show in the proof of Theorem~\ref{thm_stability} that $\alpha_n$ scales as $1/n^{\frac{1}{2} - \varepsilon}$ for some $\varepsilon\in (0,1/2)$ with respect to the number of samples $n$. We note that the bound in~\eqref{vnn_stab} becomes looser with increase in $F$ or $L$ which is consistent with the parallel result for GNNs~\cite{gama2020stability}. However, without the analysis of the lower bounds on $ \|\Phi(\bx;\hat\bC, \cH)-\Phi(\bx;\bC,\cH)\| $, we cannot claim that the robustness of VNNs indeed worsens with increase in $F$ or $L$. Moreover, we remark that VNNs sacrifice discriminability between eigenvectors associated with close eigenvalues to achieve stability~\cite{sihag2022covariance}. As a corollary, we also state that Theorem~\ref{thm_stability} can readily be extended to characterize the difference between VNN outputs corresponding to sample covariance matrices estimated from $n_1$ and $n_2$ samples via~\eqref{vnn_stab} and application of triangle inequality.


\subsection{Transferability of VNNs}\label{trnsfrr}

The notion of transferability of VNNs is made feasible by the properties of coVariance filters (Definition~\ref{def1}) that can be instantiated on covariance matrices of any dimension and therefore, the VNNs can readily be transferred to process a dataset of a different dimensionality. From the perspective of implementation, transferability of VNNs to a dataset of different number of features amounts to replacing the covariance matrix $\bC$ in a VNN model $\Phi(\cdot;\bC,\cH)$ with a covariance matrix of another size, while keeping the parameters $\cH$ fixed. Since we consider covariance matrices of different dimensionalities, we denote a covariance matrix $\bC$ of size $m\times m$ by $\bC_m$. Informally, we can state our objective for assessing transferability as follows.

\noindent
{\bf Informal Problem Statement for VNN Transferability.} Given a data point $\bx_{m_1}$ from a dataset with $m_1$ features and associated covariance matrix $\bC_{m_1}$, and another data point $\bx_{m_2}$ from a dataset with $m_2$ features and associated covariance matrix $\bC_{m_2}$, we aim to characterize the conditions under which the VNN outputs $\Phi(\bx_{m_1};\bC_{m_1},\cH)$ and $\Phi(\bx_{m_2};\bC_{m_2},\cH)$ converge. When $\Phi(\bx_{m_1};\bC_{m_1},\cH)$ and $\Phi(\bx_{m_2};\bC_{m_2},\cH)$ converge, we can conclude that the parameters $\cH$ are transferable between two datasets consisting of $m_1$ and $m_2$ features.

Note that the VNN  outputs  $\Phi(\cdot;\bC_{m_1},\cH)$ and $\Phi(\cdot;\bC_{m_2},\cH)$ have distinct dimensionalities if $m_1\neq m_2$ and therefore, a direct comparison between them is not natural. Fundamentally, it is imperative to provide a mathematical framework to compare vectors and covariance matrices of different sizes in order to be able to analyze the transferability of VNNs. To facilitate such a comparison between vectors of different sizes, we consider a simple mapping that represents the vector on a continuous interval $[0,1]$. Specifically, given an $m$-dimensional vector $\bx = [x_1,\dots,x_m]$, we can define a continuous representation of $\bx$ as a function $y_{\bx}: [0,1] \mapsto \mR$, such that, $y_{\bx}(u) = x_i$ for $u\in \cU_i$, where $\cU_i$ is a pre-defined interval associated with the $i$-th element of $\bx$. Similarly, we can map a covariance matrix $\bC_m$ to a compact space defined by $[0,1]^2$ using the mapping $\bW_{\bC_m}: [0,1]^2 \mapsto \mR$ where we have $\bW_{\bC_m}(u,v) = [\bC_m]_{ij}$ for $u\in \cU_i$ and $v\in \cU_j$. A pictorial illustration of $y_{\bx}$ and $\bW_{\bC}$ for covariance matrix $\bC$ is included in Fig~\ref{cont_represent}. Therein, the intervals $\cU_i$ are parameterized by variables $\rho_i$, which will be discussed subsequently in~\eqref{interval1}. Note that we can recover $\bx$ from $y_{\bx}$ and vice-versa (similarly for $\bC_m$ and $\bW_{\bC_m}$). Hence, for data points $\bx_{m_1}$ and $\bx_{m_2}$ consisting of $m_1$ and $m_2$ elements, respectively, the closeness of continuous representations $y_{\bx_{m_1}}$ and $y_{\bx_{m_2}}$  can be used as a metric to assess the similarity between data points in multi-scale datasets. This observation also extends to the comparison between covariance matrices $\bC_{m_1}$ and $\bC_{m_2}$. 

\begin{figure}[!htbp]
  \centering
  \includegraphics[scale=0.35]{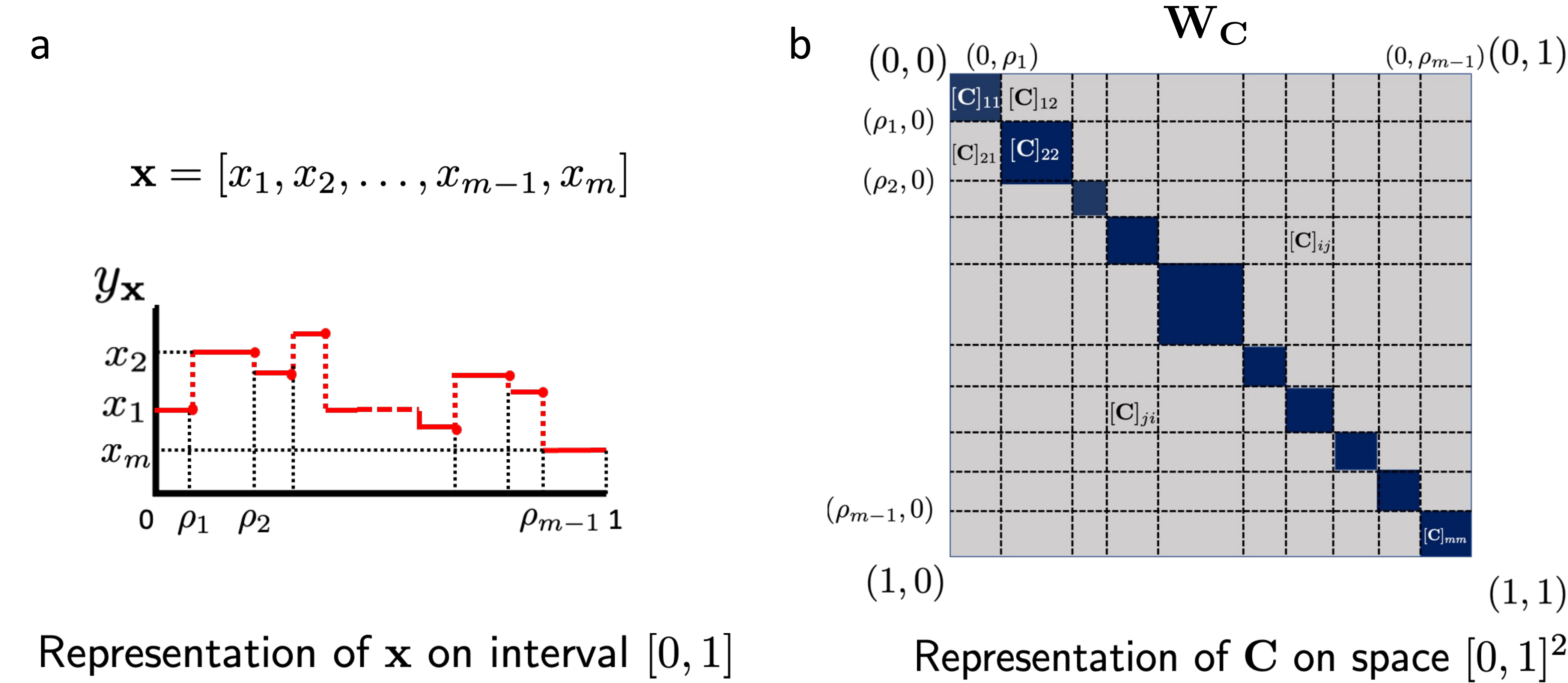}
   \caption{{\bf Representations of $m$-dimensional vector $\bx$ and associated covariance matrix $\bC$ in the continuous domain.} {\bf a.} Representation for $\bx$ is obtained by discretizing the space $[0,1]$. {\bf b.} Representation $\bW_{\bC}$ for $\bC$ is obtained by discretizing the space $[0,1]^2$ according to~\eqref{interval1}. Thus, $\bW_{\bC}$ retains the symmetry of $\bC$. The area spanned by the diagonal elements of $\bC$ is marked in blue. The size of the square area allotted to a diagonal element is proportional to its value. Other parts of the grid accommodate the off-diagonal elements of $\bC$.}
   \label{cont_represent}
\end{figure}

For a VNN architecture with $F$ number of outputs in the final layer, the dimensionality of  VNN outputs $\Phi(\bx_{m_1};\bC_{m_1},\cH)$ is $m_1\times F$ and that for $\Phi(\bx_{m_2};\bC_{m_2},\cH)$ is $m_2\times F$. Thus, we can compare $\Phi(\bx_{m_1};\bC_{m_1},\cH)$ and $\Phi(\bx_{m_2};\bC_{m_2},\cH)$ via the comparison between the continuous representations of every column in outputs $\Phi(\bx_{m_1};\bC_{m_1},\cH)$ and $\Phi(\bx_{m_2};\bC_{m_2},\cH)$, where the continuous representations are defined in the same fashion as $y_{\bx}$ above. For VNN $\Phi(\bx_{m_1};\bC_{m_1},\cH)$, we use the notation $y_{m_1}[f]$ to denote the continuous representation of $f$-th output in $\Phi(\bx_{m_1};\bC_{m_1},\cH)$, i.e.,
 $y_{[\Phi(\bx_{m_1};\bC_{m_1},\cH)]_f}$. Similar to the relationship between $y_{\bx}$ and $\bx$, the $f$-th VNN output $[\Phi(\bx_{m_1};\bC_{m_1},\cH)]_f$ and its continuous representation $y_{m_1}[f]$ are operationally interchangeable (see Appendix~\ref{gip} for details). Using the continuous representations above, we can now describe the assessment of transferability of VNNs more concretely. 

\noindent
{\bf Formal Problem Statement for VNN Transferability:} Consider two VNNs $\Phi(\bx_{m_1};\bC_{m_1},\cH)$ and $\Phi(\bx_{m_2};\bC_{m_2},\cH)$ instantiated on data with $m_1$ and $m_2$ features, respectively. If we have the following conditions: (a) the continuous approximations of inputs $\bx_{m_1}$ and $\bx_{m_2}$ are close, i.e., $\|y_{\bx_{m_1}} - y_{\bx_{m_2}}\|_2$ is bounded; and (b) the continuous approximations of covariance matrices $\bC_{m_1}$ and $\bC_{m_2}$ are close, i.e., $\|\bW_{\bC_{m_1}} - \bW_{\bC_{m_2}}\|_2$ is bounded; we aim to show that the continuous representations of VNN outputs $\Phi(\bx_{m_1};\bC_{m_1},\cH)$ and $\Phi(\bx_{m_2};\bC_{m_2},\cH)$ are close, i.e., $\|y_{m_1}[f] - y_{m_2}[f]\|_2$ is bounded for all $f\in \{1,\dots,F\}$.

 Next, we informally state the main result of this section that establishes the transferability between VNNs processing datasets consisting of $m_1$ and $m_2$ features. 

 \begin{theorem}[Transferability of VNNs]\label{transferthm}
Consider two VNNs $\Phi(\bx_{m_1};\bC_{m_1},\cH)$ and $\Phi(\bx_{m_2};\bC_{m_2},\cH)$ consisting of $L$ layers and $F$ outputs per layer. If the continuous representations of inputs and covariance matrices are close, i.e., $\|y_{\bx_{m_1}} - y_{\bx_{m_2}}\|_2$ and $\|\bW_{\bC_{m_1}} - \bW_{\bC_{m_2}}\|_2$ are bounded, and assumptions A1 and A2 are satisfied (described in~\eqref{interval1} and~\eqref{lips}), we have 
\begin{align}\label{transfereq}
\|y_{m_1}[f] - y_{m_2}[f] \|_2 \leq LF^{L} \beta\Big(\frac{1}{m_1^{3\zeta/2-1}} + \frac{1}{m_2^{3\zeta/2-1}}\Big) \;, \forall f \in \{1,\dots,F\}\;,
\end{align}
for some $\beta>0$ and $\zeta\in (2/3,1]$.
\end{theorem}
\begin{proof}
    See Appendix~\ref{pf_thm2}. 
\end{proof}
\noindent 
Theorem~\ref{transferthm} implies that continuous representations of all $F$ outputs of the respective final layers of VNNs $\Phi(\bx_{m_1};\bC_{m_1},\cH)$ and $\Phi(\bx_{m_2};\bC_{m_2},\cH)$ converge with increase in $m_1$ and $m_2$. Since the continuous representation $y_{m_1}[f]$ and VNN output $[\Phi(\bx_{m_1};\bC_{m_1},\cH)]_f$ are operationally interchangeable, we expect the measures of central tendency (e.g., mean, median) of outputs~$[\Phi(\bx_{m_1};\bC_{m_1},\cH)]_f$ and~$[\Phi(\bx_{m_2};\bC_{m_2},\cH)]_f$ to converge as well. By extension, we also expect the measures of central tendency  for $\Phi(\bx_{m_1};\bC_{m_1},\cH)$ and $\Phi(\bx_{m_2};\bC_{m_2},\cH)$ to converge if Theorem~\ref{transferthm} holds.  In this context, if the readout function for VNN is unweighted mean, we expect the statistical outcomes of VNNs  $\Phi(\bx_{m_1};\bC_{m_1},\cH)$ and $\Phi(\bx_{m_2};\bC_{m_2},\cH)$  to be close and this convergence to be stronger for large $m_1$ and $m_2$. The impact of Theorem~\ref{transferthm} is broad, as we have shown that the parameters $\cH$ can be ``scale-free" while preserving the performance over an inference task. Specifically, a VNN can be instantiated on a dataset of different dimensionality than the training dataset and the VNN recovers close statistical outcomes for the same parameters $\cH$ for both datasets, provided the data samples and covariance matrices of the training dataset and the new dataset are close in terms of their continuous representations. Thus, VNNs also offer a significant advantage over PCA-based analysis approaches as the principal components are restricted within the feature space of a dataset and do not provide any mathematical insight into the structure of another dataset of different dimensionality even when the datasets may be related. An overview of the transferability of VNNs is illustrated in Fig.~\ref{vnn_transfer_overview}. While the rigorous details behind the proof of Theorem~\ref{transferthm} are relegated to Appendix~\ref{pf_thm2}, we briefly discuss the major mathematical concepts and assumptions that have enabled us to establish~\eqref{transfereq} in Theorem~\ref{transferthm} in the highlighted text on page 19. 

Thus far, we have discussed the transferability of VNNs in terms of achieving statistical outcomes (e.g., outcome of a regression model) that are close using datasets of different dimensionalities, given that the datasets are aligned in some way. As discussed previously, a multi-scale neuroimaging dataset provides an ideal setting for validating the theoretical guarantees of Theorem~\ref{transferthm} in this context. However, note that the VNNs also provide expressivity at the feature-level at the final layer. For instance, if VNN is deployed for a regression task and the readout layer is a simple mean function, the VNN final layer outputs could be leveraged to characterize the contributions of each feature in the dataset to the final regression outcome. In applications based on neuroimaging, this observation can be of great interest, as each feature in a neuroimaging dataset is typically associated with a distinct brain region. Thus, VNNs naturally provide a feasible way to \emph{interpret} the final statistical outcomes.  Moreover, if such an interpretability is achieved using VNNs, we can further intertwine our analysis with the transferability of VNNs to test its spatial robustness across datasets of different dimensionalities. The observations made here motivated us to pursue investigating the utility of VNNs in characterizing the contributing brain regions to elevated $\Delta$-Age due to age-related neurodegeneration. 

\clearpage
\begin{figure}[t]
  \centering
  \includegraphics[scale=0.4]{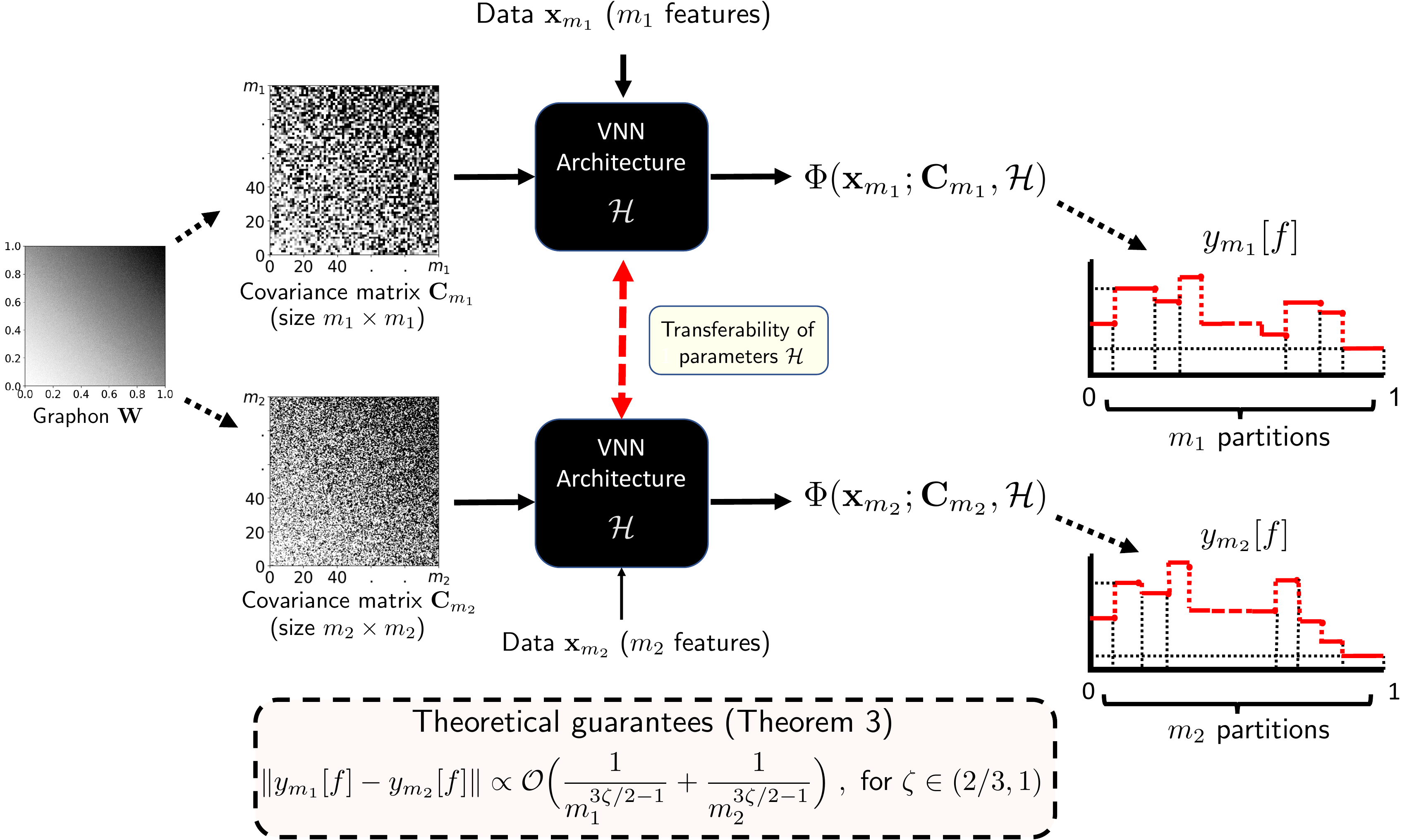}
   \caption{{\bf Overview of transferability of VNNs.} $y_{m_1}[f]$ is the continuous representation of $f$-th output of VNN $\Phi(\bx_{m_1};\bC_{m_1},\cH)$ that is instantiated on data and covariance matrix with $m_1$ features. Similarly,  $y_{m_2}[f]$ represents the $f$-th output of VNN $\Phi(\bx_{m_2};\bC_{m_2},\cH)$ that is instantiated on dataset with $m_2$ features. If the continuous counterparts of covariance matrices $\bC_{m_1}$ and $\bC_{m_2}$, i.e., $\bW_{\bC_{m_1}}$ and $\bW_{\bC_{m_2}}$, belong to a sequence that converges to a graphon $\bW$ (Definition~\ref{grphon}) and the continuous representations of inputs $\bx_{m_1}$ and $\bx_{m_1}$ are close, the convergence between $y_{m_1}[f]$ and $y_{m_2}[f]$ is characterized in terms of $m_1$ and $m_2$ in Theorem~\ref{transferthm}. }
   \label{vnn_transfer_overview}
\end{figure}

\begin{tcolorbox}[colback=yellow!5!white,colframe=red!75!black,title={Mathematical Foundations of Transferability},]
\footnotesize
\begin{multicols}{2}
The continuous representations of graph signals and graphs have previously been leveraged to study transferability of GNNs under the domain of graphon information processing~\cite{ruiz2021graphon}. Specifically, GNNs can be transferable between graphs belonging to a converging sequence if the graphs in this sequence converge to a limit object called \emph{graphon} as the number of nodes approaches infinity~\cite{borgs2008convergent}. In a similar fashion, we leverage the theory of graphons~\cite{borgs2008convergent} and graphon signal processing~\cite{ruiz2021graphon} to establish Theorem~\ref{transferthm}. Graphons are the limits of \emph{dense} graphs (i.e., graphs with number of edges of the order $\Theta(m^2)$)~\cite{lovasz2012large} and hence, appropriate to study limits of covariance matrices that are typically dense. The definition of a graphon is provided in Definition~\ref{grphon}.
\begin{definition}[Graphon]\label{grphon}
A graphon is a bounded, symmetric, measurable function ${\bW: [0,1]^2 \mapsto [-1,1]}$.   
\end{definition} 
\noindent
Under the setting where a covariance matrix is viewed as a weighted graph and normalized such that its largest eigenvalue is $1$, we claim that a sequence of covariance matrices $\{\bC_m\}$ being convergent implies that the sequence of their counterpart continuous representations, i.e., $\{\bW_{\bC_m}\}$, converges to some graphon $\bW$ if $\bW_{\bC_m}$ is appropriately constructed from $\bC_m$. This claim is based on generalizing~\cite[Corollary 3.9]{borgs2008convergent} to our setting, and the formal statement in this regard is included in Remark~\ref{grphnlimit} in Appendix~\ref{gip}. Theorem~\ref{transferthm} holds for any pair of covariance matrices in the converging sequence $\{\bC_m\}$ (under certain assumptions that will be discussed shortly) and thus, parameters $\cH$ can be transferred between any two VNNs instantiated on distinct covariance matrices in this sequence. The construction of $\bW_{\bC_m}$ relies on appropriately defining the intervals $\cU_i$ and is described in the following steps. 
\begin{itemize}[noitemsep,topsep=0pt,leftmargin=4mm]
\item[a.] Partition the interval $[0,1]$ into $m$ disjoint intervals $[\cU_1,\dots,\cU_m]$, such that, 
    \begin{align}\label{interval1}
        \cU_i = \begin{cases}
        [0,\rho_1] \quad \text{if} \quad i=1\\
        (\rho_{i-1}, \rho_i] \quad\text{if}\quad i\in\{2,\dots,m\} 
        \end{cases}\;,
    \end{align}
    where 
    \begin{align}
        \rho_i \dff \frac{1}{{\sf tr} (\bC_m)}\sum\limits_{j=1}^{i} [\bC_m]_{jj}\;,
    \end{align}
    and ${\sf tr}(\bC_m)$ is the trace of $\bC_m$. Clearly, $\rho_m = 1$.  
\item[b.] The relationship between feature $i$ and feature $j$ is given by $\bW_{{\bf C}_{m}}(u_i,u_j) = [\bC_m]_{ij}$ for $u_i \in \cU_i, u_j\in \cU_j$. 
\end{itemize}
If the continuous representation $\bW_{\bC_m}$ of $\bC_m$ is constructed according to the above steps, we refer to $\bW_{\bC_m}$ as the graphon approximation of $\bC_m$. Thus, the graphon limit $\bW$ forms the schema for which the covariance matrix $\bC_m$ represents the covariance realization at resolution $m$. Next, we note that the result~\eqref{transfereq} is contingent upon two main assumptions related to the covariance matrix $\bC_m$ and the graphon limit $\bW$. These assumptions are mentioned below. 
\begin{itemize}[noitemsep,topsep=0pt,leftmargin=7mm]
    \item[{\bf A1.}] ($(\Omega,\zeta)$-dominant property of covariance matrices) For the sequence $\{\bC_m\}$, there exist positive constants $\Omega$ and $\zeta$, such that, we have
    \begin{align}\label{dom_prop}
        \frac{1}{{\sf tr}(\bC_m)}\max_{j\in \{1,\dots,m\}}[\bC_m]_{jj} \leq \frac{\Omega}{m^{\zeta}}\;,
    \end{align}
    for all finite $m$. Our analysis in Appendix~\ref{pf_thm2} shows that Theorem~\ref{transferthm} holds for any $\Omega > 0$ and $\zeta \in (2/3,1]$. The property in~\ref{dom_prop} implies that 
\begin{align}\label{cnvrg}
\frac{1}{{\sf tr}(\bC_m)}\max_{j\in \{1,\dots,m\}}[\bC_m]_{jj} \rightarrow 0
\end{align}
as $m \rightarrow \infty$. We refer to the covariance matrix $\bC_m$ satisfying the property in~\eqref{dom_prop} as being $(\Omega,\zeta)$-dominant. 
We also note that~\eqref{cnvrg} suggests that the variance profile of individual features in the dataset, characterized by their corresponding diagonal elements in the covariance matrix, must not be concentrated within a small subset of features. 

\item[{\bf A2.}] (Lipschitz continuity of Graphon.) If $\bW$ is the limit of the sequence $\{\bW_{\bC_m}\}$ as $m\rightarrow \infty$, then for Theorem~\ref{transferthm} to hold,  $\bW$ must satisfy
    \begin{align}\label{lips}
        |\bW(u_1,v_1) - \bW(u_2,v_2)| \leq \alpha(|u_1-u_2| + |v_1 - v_2|)\;,
    \end{align}
    for any $u_1,v_1,u_2,v_2 \in [0,1]$. Any graphon satisfying~\eqref{lips} is termed as an $\alpha$-Lipschitz  graphon.
    The Lipschitz continuity of graphon $\bW$ determines the smoothness of the information present between any two coordinates $(u_1,v_1)$ and $(u_2,v_2)$. Therefore, it is intuitively expected that for a given $m$, a graphon $\bW$ with smaller Lipschitz constant $\alpha$ will be better approximated by $\bW_{\bC_m}$ as compared to that with higher Lipschitz constant. 
\end{itemize}
\end{multicols}
\end{tcolorbox}


\begin{figure}[!htbp]
  \centering
  \includegraphics[scale=0.4]{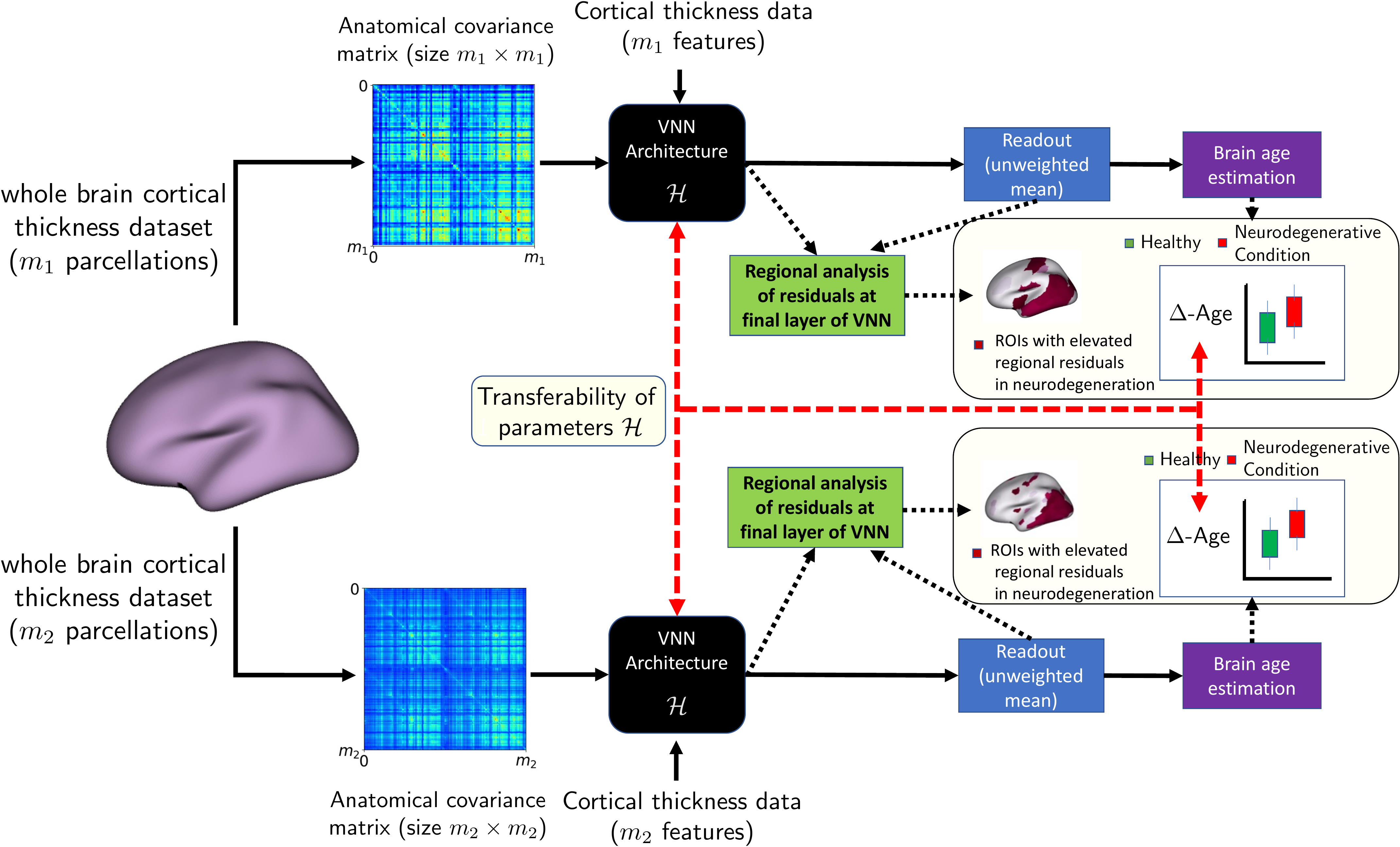}
   \caption{{\bf Overview of brain age prediction framework using VNNs.} The VNN is trained as a regression model by optimizing filter taps $\cH$ to predict chronological age using cortical thickness features for HC group. The VNN prediction is formed via a readout function that averages the outputs at the final layer of VNN. For brain age prediction, we replace the covariance matrix in the trained VNN model with the anatomical covariance matrix from the combined population of HC group and individuals with AD. For each individual, VNNs first form the scalar output by combining outputs at the final layer of VNN via the readout function, and then bias-correction is applied to finally evaluate the brain age. Statistical analysis of the final layer outputs of the VNN is performed  across the combined population to identify locally elevated residual effects in neurodegeneration. These locally elevated residual effects eventually contribute to the elevated $\Delta$-Age effect and hence, provide a regional profile to elevated $\Delta$-Age due to age-related neurodegeneration. If the transferability of performance on the chronological age prediction task holds for different datasets (according to Theorem~\ref{transferthm}), we expect to see similar $\Delta$-Age distributions across them without re-training. However, the regional profiles do not hinge on the performance of VNN in predicting chronological age of HC group and are derived before age-bias correction is applied to calculate $\Delta$-Age for an individual. Hence, if the regional profiles are characteristic of accelerated aging, we expect to observe similar regional profiles after transferring VNNs to different datasets even when Theorem~\ref{transferthm} may not hold (for instance when transferring the VNN model from FTDC datasets to OASIS-3 dataset).}
   \label{brain_age_overview}
\end{figure}
\section{Application: Brain Age Prediction}\label{age_intro}
$\Delta$-Age is a known biomarker of cognitive decline and neurodegeneration~\cite{franke2019ten,cole2017predicting}. However, in the absence of a ground truth, the notion of $\Delta$-Age is abstract and has a limited clinical utility without identification of the main contributors to the elevated brain age due to neurodegeneration. In this paper, we focus on Alzheimer's disease (AD) as an example of pathological age-related neurodegeneration. Age is a major risk factor for AD and hence, AD is characterized by biological traits that signify accelerated aging~\cite{jove2014metabolomics}. We leverage the architecture of VNNs to provide a regional perspective to brain age prediction and our results demonstrate that the elevated $\Delta$-Age in AD is accompanied by abnormalities in various regions of interest characteristic of AD. The description of datasets is included in subsection~\ref{data} in the Methods section. Next, we discuss how the analysis of outputs at the final layer of VNNs may capture the impact of age-related neurodegeneration on various regions of the brain. 

\subsection{Identification of Regions Associated with Neurodegeneration using VNN Architecture}
To start with, we used VNN as a regression model to fit multivariate cortical thickness data against chronological age for healthy controls (HC). Details on training and selected architecture of VNN models used for the brain age analysis are provided in the subsection~\ref{vnn_learn} in the Methods section. Thus, in our experiments, a VNN processed the cortical thickness data through a multi-layer architecture consisting of the anatomical covariance matrix derived from cortical thickness of the HC group. The final prediction by the VNN model was determined by unweighted aggregation of the final layer outputs which can be conceptualized as an unweighted mean of age predictions at individual brain regions. Therefore, VNN architecture allowed us to compute ``regional residuals" (scalar output at a given region derived from VNN final layer output - aggregated VNN output or age estimate formed by VNN) at each brain region to assess their contribution to the final output of VNN. We refer the reader to subsection~\ref{regbage} in the Methods section for concrete details regarding calculations of regional residuals.

The VNN model trained on the data from HC group according to the procedure in subsection~\ref{vnn_learn} captured the information about healthy aging from cortical thickness and associated anatomical covariance matrix. Note that this information may not be enough to predict chronological age accurately, i.e., we may not achieve a perfect fit on chronological age of HC group. However, our theoretical results regarding stability of VNNs in Theorem~\ref{thm_stability} and transferability of VNNs in Theorem~\ref{transferthm} dictate that the regression performance achieved by VNNs in this context is expected to be stable to perturbations in the anatomical covariance matrix due to change in sample size and transferable across datasets of different dimensionalities if their respective covariance matrices satisfy certain theoretical assumptions.

Next, we tested these trained VNN models on the combined groups of HC and AD by replacing the anatomical covariance matrix from the HC group with the anatomical covariance matrix from the combined group of HC and individuals with AD. Since capturing the accelerated aging is of interest, we hypothesized that the brain regions characteristic of AD (for instance, regions with higher cortical atrophy) with respect to HC group were likely to have elevated regional residuals for individuals with AD. Also, since outputs of VNNs are obtained via manipulation of the data according to principal components or eigenvectors of the covariance matrix (Section~\ref{mtvn}), we hypothesized that the regional residuals may be correlated with the principal components of the anatomical covariance matrix of the combined group. Furthermore, the stability of VNNs established in Theorem~\ref{thm_stability} implied that any group differences observed in regional residuals were expected to be stable to perturbations in the covariance matrix and thus, stable to the composition of the combined population of HC and individuals with AD used for estimating the anatomical covariance matrix.

The scale-free architecture of VNNs allowed us to gauge the spatial and statistical robustness of any observed elevated regional residuals in age-related neurodegeneration in scenarios when the transferability of performance was guaranteed (Theorem~\ref{transferthm} and Fig.~\ref{vnn_transfer_overview}) and when it was not guaranteed.  
Our discussion in Appendix~\ref{fig_converge} demonstrates that the cortical thickness datasets curated according to different parcellations of Schaefer's atlas (referred to as FTDC datasets and described in Section~\ref{data}) laid within the scope of theoretical guarantees on transferability as described in Fig.~\ref{vnn_transfer_overview}. By extension, the performance of VNN models trained for chronological age prediction of HC group was transferable across datasets organized according to different resolutions of Schaefer's atlas. Hence, assessments of the regions pertaining to elevated regional residuals across different scales of Schaefer's atlas could provide a proof of concept for the transferability property holding in brain age prediction. Experiments pertaining to this aspect are included in Appendix~\ref{explor}.


We focus on the setting where the theoretical guarantees of transferability of performance (Theorem~\ref{transferthm}) may not hold to assess whether transferring VNNs across datasets curated according to distinct brain atlases resulted in retaining the regional profiles pertinent to AD-related neurodegeneration. In this context,  we studied the regional profiles when VNNs were transferred from a dataset organized according to Schaefer's atlas to a dataset organized according to Destrieux (DKT)~\cite{destrieux2010automatic} atlas. The datasets in this context were derived from distinct populations. The dataset organized according to DKT atlas laid outside the purview of theoretical guarantees on transferability of VNNs trained on datasets organized according to different scales of Schaefer's atlas (see Fig.~\ref{cutmetrics} in Appendix~\ref{fig_converge} and associated discussion). Lack of theoretical guarantees on the VNN transferability between datasets organized according to DKT atlas and those according to Schaefer's atlas imply that VNNs may need to be augmented by a mapping that accounts for the differences between DKT atlas and Schaefer's atlas to achieve comparable performance on the chronological age prediction task. Hence, while the performance on the task of chronological age prediction was not guaranteed to be trivially transferable between datasets curated according to Schaefer's atlas and DKT atlas, we hypothesized that the observed effects of elevated regional residuals may be observed even after transferring VNNs if such effects were driven by cortical atrophy or neurodegeneration-related features in the anatomical covariance matrix. Therefore, evaluating the regional residuals after transferring VNNs between datasets curated according to distinct brain atlases could allow decoupling of the potential contributors of elevated brain age from the objective of achieving near perfect performance over chronological age prediction in the HC group~\footnote{We remark that the discussion here is not restricted to Schaefer's atlas and DKT atlas and could potentially be extended to any pair of datasets curated according to distinct brain atlases, such that, VNNs may not be trivially transferable between them.}. 

We further expanded the experiments above by investigating the correlations between regional residuals and clinical dementia rating (CDR) metrics. CDR sum of boxes scores are commonly used in clinical and research settings to stage dementia severity~\cite{o2008staging}. A higher CDR score is associated with more severe cognitive and functional status. If our hypothesis that the elevated regional residuals were driven by age-related neurodegeneration was valid, we expected to observe a significant alignment between the span of brain regions with elevated regional residuals and the span of brain regions whose regional residuals were correlated with CDR metrics. 




\subsection{Individual-level Brain Age Prediction}
Next, we focused on predicting the $\Delta$-Age for an individual from VNN outputs using a procedure that is consistent with other studies in the literature. The residuals evaluated via the difference between the chronological age and predicted age for individuals in the HC group by VNNs are prone to bias. Specifically, age for younger individuals tends to be over-estimated and that for older individuals tends to be under-estimated~\cite{beheshti2019bias,le2018nonlinear}. This phenomenon may appear, for example, when the correlation between the chronological age and predicted age is significantly smaller than 1. In order to correct for this age-related bias in the residuals, we used a linear regression based approach~\cite{de2020commentary} and evaluated $\Delta$-Age. We hypothesized $\Delta$-Age to be significantly elevated in the group of individuals with AD with respect to that for those in the HC group. The methodological details on age-bias correction and the procedure to evaluate $\Delta$-Age are included in subsection~\ref{brain_age} in the Methods section. Also, by extension of the transferability property  of VNNs, we have previously reported similar distributions of $\Delta$-Age across datasets curated according to different scales of a multi-scale brain atlas in~\cite{sihag2022predicting}.

\section{Methods}\label{methods}
\subsection{Data}\label{data}
We consider datasets from two independent populations in this paper, as described below.

\noindent
{\bf Multi-scale FTDC Datasets.} These datasets consist of the cortical thickness data extracted at different resolutions from healthy controls (HC; $n = 105$, age = $62.6 \pm 7.62$ years, 57 females). 
For each individual, the cortical thickness data was curated according to multi-resolution Schaefer atlas~\cite{schaefer2018local}, at 100 parcel, 300 parcel, and 500 parcel resolutions with finer resolution cortical thickness estimates with increasing number of parcellations.  The ANTs cortical thickness pipeline~\cite{tustison2014large,das2009registration} was used to derive mean cortical thickness within each atlas parcel using 3T T1-weighted MRIs (~1mm isotropic resolution). We report results on three datasets: FTDC100, FTDC300 and FTDC500, that constitute the cortical thickness datasets corresponding to 100, 300 and 500 cortical thickness features, respectively. Also, the FTDC100, FTDC300, and FTDC500 datasets are jointly referred to as FTDC datasets.   


\noindent


\noindent 

\noindent
{\bf OASIS-3 Dataset}. This dataset was derived from publicly available
freesurfer estimates of cortical thickness (hosted on \url{central.xnat.org}), as previously reported~\cite{lamontagne2019oasis}, and comprised of cognitively normal individuals (HC; $n = 652$, age = $67.76\pm 7.88$ years, $382$ females) and individuals with AD dementia diagnosis and at various stages of cognitive decline ($n=209$, age = $74.61 \pm 7.13$ years, $102$ females). The cortical thickness features were curated according to the Destrieux (DKT) atlas~\cite{destrieux2010automatic}(consisting of $148$ cortical regions). In the context of transferability, OASIS-3 provided a dataset curated according to a distinct brain atlas than the multi-scale Schaefer's atlas for FTDC datasets and hence, allowed us to investigate the VNN transferability beyond the setting with multi-resolution datasets. For clarity of exposition of the brain age prediction method, any dementia staging to subdivide the group of individuals with AD dementia diagnosis into mild cognitive impairment (MCI) or AD was not performed and we use the label AD+ to refer to this group. The individuals in AD+ group were significantly older than those in HC group (t-test: $p$-value = $4.15\times 10^{-27}$).  The boxplots for the distributions of chronological age for HC and AD+ groups are included in Fig.~\ref{age_dist}. For $206$ individuals in the AD+ group, the CDR sum of boxes scores evaluated within one year (365 days) from the MRI scan were available (CDR sum of boxes = $3.38\pm 1.73$). The CDR sum of boxes scores for this population were evaluated according to~\cite{morris1993clinical}.

The data from individuals in HC group for FTDC and OASIS-3 were leveraged to investigate transferability of VNNs as governed by the theoretical framework provided in Section~\ref{vnn_intro}. Furthermore, we primarily used the OASIS-3 dataset to derive conclusions about $\Delta$-Age as a marker of accelerated aging in the AD+ group. The observations made after transferring VNNs between FTDC and OASIS-3 datasets were used to gain methodological insights regarding the role of training VNNs to predict chronological age in the brain age prediction framework. Besides the datasets described above, we also extended the transferability aspect of VNNs to the brain age prediction framework on a small, multi-scale dataset consisting of individuals with AD diagnosis (independent of OASIS-3) that was curated according to different scales of the Schaefer's atlas (Appendix~\ref{explor}) in order to demonstrate the spatial robustness of the regional profiles obtained by our brain age prediction framework across multiple scales and validation of the findings on OASIS-3 dataset.

\subsection{VNN Learning}\label{vnn_learn}
We use the VNN model for a regression task where a multi-variate feature set is regressed to a scalar quantity. Note that the VNN output of the architecture represented by $\Phi(\bx;\hat\bC,\cH)$\footnote{We use the notation $\hat\bC$ for covariance matrix in the Methods section as the VNN architecture is modeled on the sample covariance matrix in practical implementation.} for one $m$-dimensional input is of dimension $m \times F$ if the VNN architecture has $F$ $m$-dimensional outputs in the final layer. The regression output is determined by a readout layer which evaluates an unweighted mean of all outputs of the final layer of VNN. Therefore, the regression output for an input $\bx$ is given by
\begin{align}
    \hat y = \frac{1}{F m}\sum\limits_{j = 1}^m\sum\limits_{f=1}^F[\Phi(\bx;\hat\bC,\cH)]_{jf}\;.
\end{align}
Prediction using unweighted mean at the output implies that the VNN model exhibits permutation-invariance (i.e., the final output is independent of the permutation of the input features and covariance matrix) and transferability. Although the performance of the VNN model could potentially be improved by adding a learnable or an adaptive readout function (e.g., weighted mean or a single layer perceptron) that maps the final layer outputs of VNN to scalar $\hat y$ via a learnable mapping, our subsequent experiments implied a negative impact on the interpretability of VNN model in the context of brain age prediction. For a regression task, the training dataset $\{\bx_i, y_i\}_{i=1}^n$ is leveraged to learn the filter taps in $\cH$ for the VNN $\Phi(\cdot;\hat\bC,\cH)$ such that they minimize the training loss, i.e.,
\begin{align}
    \cH_{\sf opt} = \min_{\cH} \frac{1}{n}\sum\limits_{i=1}^n \ell(\hat y_i, y_i)\;,
\end{align}
where 
\begin{align}
    \hat y_i = \frac{1}{F m} \sum\limits_{j =1}^m\sum\limits_{f=1}^F[\Phi(\bx_i;\hat\bC,\cH)]_{jf}\;,
\end{align}
and $\ell(\cdot)$ is the mean-squared error (MSE) loss function. 

In our experiments, we trained four sets of VNN models; one each for the HC group in FTDC100, FTDC300, FTDC500, and OASIS-3 datasets. The training process was similar for all VNNs. We randomly split the dataset into an approximately $90/10$ train/test split. Thus, the test sets for FTDC datasets consisted of $10$ individuals and that in OASIS-3 dataset consisted of $65$ individuals. The sample covariance matrix was evaluated using all samples in the training set ($n=95$ for FTDC datasets and $n=587$ for OASIS-3 dataset) and we had the sample covariance matrix $\hat\bC$ of size $m\times m$ (where $m = 100$ for FTDC100, $m=300$ for FTDC300, $m=500$ for FTDC500, and $m=148$ for OASIS-3). Furthermore, for all datasets, $\hat\bC$ was normalized such that its maximum eigenvalue was $1$. Next, the training set was randomly split internally, such that, the VNN was trained with respect to the mean squared error loss between the predicted age and the true age in $n = 84$ samples for FTDC datasets and $n=513$ samples for OASIS-3 dataset. The loss was optimized using batch stochastic gradient descent with Adam optimizer available in PyTorch library~\cite{paszke2019pytorch} for up to $100$ epochs. The batch size was $34$ for FTDC100 dataset, $8$ for FTDC300 dataset, $12$ for FTDC500 dataset, and $78$ for OASIS-3 dataset. The VNN model with the best minimum mean squared error performance on the remaining samples in the training set (which acted as a validation set) was included in the set of nominal models for this permutation of the training set.  For each dataset, we trained and validated the VNN models over $100$ permutations of the complete training set of $n = 95$ samples for each of the FTDC datasets and $n=587$ samples for OASIS-3 dataset, thus, leading to $100$ trained VNN models (also referred to as nominal models) per dataset.

The hyperparameters for the VNN architecture and learning rate of the optimizer were chosen according to a hyperparameter search procedure using the package Optuna~\cite{akiba2019optuna}. For FTDC100, the VNN had a $L = 2$-layer architecture, with a filter bank such that we had $F =  26$ and $2$ filter taps in each layer. The learning rate for the optimizer was $0.058$. The number of learnable parameters for this VNN was $1456$. For FTDC300, the VNN had a $L = 2$-layer architecture, with a filter bank such that we had $F =  39$ and $3$ filter taps in the first layer and $2$ filter taps in the second layer. The learning rate for the optimizer was $0.0241$. The number of learnable parameters for this VNN was $3237$. For FTDC500, the VNN model had a $L = 2$-layers with a filter bank such that we had $F = 27$ and $4$ filter taps in the first layer and $2$ filter taps in the second layer. The number of learnable parameters for this VNN was $1620$.  The learning rate for the Adam optimizer was set to $0.0631$. For OASIS-3, the VNN model had a $L = 2$-layers with a filter bank such that we had $F = 5$ and $6$ filter taps in the first layer and $10$ filter taps in the second layer. The learning rate for the Adam optimizer was set to $0.059$. The number of learnable parameters for this VNN was $290$.

Note that the parameters of all VNNs were learned on the HC group in their respective datasets and no subsequent training was performed for brain age prediction. The above procedure for training the VNNs by splitting the datasets into training/validation/test sets was performed to ensure that the VNNs were not overfit on the training set. The performance of VNNs on test sets for different datasets are reported in Appendix~\ref{vnn_test_perf}. However, our results primarily focus on the settings of transferability and brain age prediction, both of which are reported on the complete datasets in the results section.

\subsection{Transferability of VNNs}\label{vnn_transfer}
In our experiments, we empirically studied the transferability of the VNN models across different datasets described in Section~\ref{data}. In general, transferring a VNN model from dataset A to dataset B implies that the VNN was trained for an inference task on dataset A and is being tested for the same inference task on dataset B. The scale-free aspect of VNN architecture allows transferring of VNNs between datasets of different dimensionalities.  

We first describe the procedure for evaluating VNN transferability on multi-scale cortical thickness datasets derived from the same population. Let $\hat\bC_{m_1}$  denote the sample covariance matrix of size $m_1\times m_1$ from the dataset on which the VNN was trained. Here, $\hat\bC_{m_1}$ was normalized such that its maximum eigenvalue was $1$ in order to reconcile with the definition of graphon in Definition~\ref{grphon}.  Consider an individual whose chronological age is $y$ and cortical thickness data available as $\bx_{m_1}$ of size $m_1\times 1$ and $\bx_{m_2}$ of size $m_2\times 1$. Therefore, the VNN model prediction for the individual formed by the trained VNN model is given by
\begin{align}\label{teq1}
    \hat y_1 = \frac{1}{Fm_1} \sum\limits_{j=1}^{m_1} \sum\limits_{f=1}^F [\Phi(\bx_{m_1};\hat\bC_{m_1},\cH)]_{jf}\;,
\end{align}
where the set of filter taps $\cH$ are determined by training on the cortical thickness dataset with $m_1$ features and we use the notation $\Phi(\bx_{m_1};\hat\bC_{m_1},\cH)$ to denote the output at the final layer of the VNN. To empirically validate the theoretical results on the transferability of VNNs, we aim to evaluate the change in performance when the covariance matrix $\hat\bC_{m_1}$ is replaced with $\hat\bC_{m_2}$, where $\hat\bC_{m_2}$ is generated from dataset with $m_2$ cortical thickness features. 
Therefore, the predicted age based on the cortical thickness features $\bx_{m_2}$ for the same individual is given by
\begin{align}
    \hat y_{2} = \frac{1}{Fm_2} \sum\limits_{j=1}^{m_2} \sum\limits_{f=1}^F [\Phi(\bx_{m_2};\hat\bC_{m_2},\cH)]_{jf}\;,
\end{align}
where the filter taps $\cH$ are the same as those in~\eqref{teq1}. If the VNN model was transferable from a dataset with $m_1$ cortical thickness features to that with $m_2$ features, we expected the predictions $\hat y_1$ and $\hat y_2$ to be close. 

The procedure to investigate VNN transferability between datasets of different dimensionalities and distinct populations is similar as above with the following extensions. In principle, to evaluate the transferability from dataset B to dataset A, we compared the performance of VNNs that were trained on dataset A and the VNNs that were transferred from dataset B to dataset A. Clearly, if the VNNs that were transferred from dataset B to dataset A achieved comparable performance as that of the VNNs that were trained on dataset A, we could conclude that VNNs exhibited transference from dataset B to dataset A.

\subsection{Brain Age Prediction}
The VNN models trained as regression models for predicting chronological age using the cortical thickness data from healthy controls were expected to capture the effect of healthy aging in cortical thickness features. Next, we describe our strategy to characterize brain age with regional interpretability in the context of AD. We focus on the brain age results determined for OASIS-3. Additional observations made on multi-scale datasets in the context of brain age are included in Appendix~\ref{explor}.

\subsubsection{Regional Analyses of VNN Residuals in Age-Related Neurodegeneration}\label{regbage}
The VNN architecture allows us to associate a scalar output with each dimension among the $m$ dimensions in the final layer. Specifically, we have
\begin{align}\label{trvnn1}
    \bp_i = \frac{1}{F} \sum\limits_{f=1}^F [\Phi(\bx_i;\hat\bC_m,\cH)]_f\;,
\end{align}
where $\bp_i$ is the vector denoting the mean of all final layer outputs associated with filters in the filter bank at the final layer. Note that the mean of all elements of $\bp_i$ is the prediction $\hat y_i$ formed in~\eqref{teq1}. In the context of cortical thickness datasets, we can associate each element of $\bp_i$ with a distinct brain region. Therefore, the vector $\bp_i$ is a vector of ``regional contributions" to the output $\hat y_i$ by the VNN. The parameters $\cH$ were learnt over the HC group as described previously and kept unchanged in the subsequent analyses. The subsequent details in this section reflect that we primarily focused on the OASIS-3 dataset to study brain age. For OASIS-3 dataset, we use the notation $\hat\bC_{148}$ for the covariance matrix formed by the cortical thickness features from HC group. 

Next, we leveraged~\eqref{trvnn1} to study the effect of neurodegeneration on brain regions. For this purpose, in the OASIS-3 dataset, we evaluated the covariance matrix $\hat\bC_{148}^{\sf AD+}$  from the combined cortical thickness data of HC and AD+ groups. 
As a consequence of the stability of Theorem~\ref{transferthm}, we expect the inference drawn from VNNs to be stable to be stable to perturbations in the covariance matrix. Therefore, subsequent evaluations using VNNs are expected to be stable to the composition of combined HC and AD+ groups used to estimate the anatomical covariance matrix $\hat\bC_{148}^{\sf AD+}$. 

For every individual in the combined dataset of HC and AD+ groups, we processed their cortical thickness data $\bx$ through the model $\Phi(\bx;\hat\bC_{148}^{\sf AD+}, \cH)$ where parameters $\cH$ were learnt in the regression task on the data from HC group as described previously. Hence, the vector of mean of all final layer outputs for cortical thickness input $\bx$ is given by
\begin{align}\label{trvnn2}
    \bp = \frac{1}{F} \sum\limits_{f=1}^F [\Phi(\bx;\hat\bC_{148}^{\sf AD+},\cH)]_f\;,
\end{align}
and VNN output is given by
\begin{align}\label{interpret}
    \hat y = \frac{1}{148} \sum\limits_{j=1}^{148} [\bp]_{j}\;.
\end{align}
Furthermore, we define the residual for feature  $a$  (or brain region represented by feature $a$ in this case) as
\begin{align}\label{res_dist}
    [\br]_a \dff [\bp]_a - \hat y\;. 
\end{align}
Thus, equation~\ref{res_dist} allows us to characterize the residuals with respect to the VNN output $\hat y$ at the regional level for individual brain regions for an individual with cortical thickness data $\bx$. Henceforth,  we refer to the residuals evaluated according to~\eqref{res_dist} as ``regional residuals". Therefore, when an individual was predicted to have a brain age higher than their chronological age due to neurodegenerative condition, we hypothesized such an observation to be an aggregated effect of contributions from certain biologically plausible brain regions. The brain regions contributing to the observed higher brain age could be characterized at a regional level by the analysis of regional residuals as defined in~\eqref{res_dist}. Thus, the elements of the residual vector $\br$ can potentially act as a biomarker that can enable the isolation of brain regions affected due to age-related neurodegeneration.

In our experiments, for a given VNN model, we evaluated the residual vector $\br$ for every individual in the dataset.  Also, for experiments on OASIS-3 dataset, we denote the population of residual vectors for healthy controls as $\br_{\sf HC}$, individuals in AD+ group as $\br_{\sf AD+}$. 
The length of the residual vectors is the same as the number of cortical thickness features in the dataset. Since each element of the residual vector is associated with a distinct brain region, we performed ANOVA to test for group differences between individuals in HC and AD+ groups. Also, since elevation in $\Delta$-Age is the biomarker of interest in this analysis, we hypothesized that the brain regions that exhibited higher means for regional residuals for AD+ group than HC group would be the most relevant to capturing accelerated aging. Hence, we report our results only for brain regions that showed elevated regional residual distribution in AD+ group with respect to HC group. Further, the group difference between AD+ and HC groups in the residual vector element for a brain region was deemed significant if it met the following criteria: i) the corrected $p$-value (Bonferroni correction) for the clinical diagnosis label in the ANOVA model was smaller than $0.05$, and  ii) the uncorrected $p$-value for clinical diagnosis label in ANCOVA model with age and sex as covariates was smaller than $0.05$. 
An example for this regional analysis is included in Appendix~\ref{regional_illust}. 

Recall that 100 distinct VNN models were trained as regression models on different permutations of the training set of cortical thickness features from HC group. We leveraged these trained models to establish the robustness of observed group differences in the distributions of regional residuals. This procedure is described next.

\noindent
{\bf  Robustness of findings from regional analyses.} We performed the regional analysis described above corresponding to each trained VNN model and tabulated the number of VNN models for which a brain region was deemed significant in the regional analysis described above. A higher number of VNN models isolating a brain region as significant suggested higher robustness of the effect observed for that brain region. For instance, a brain region with that exhibited elevated regional residual in AD+ group with respect to HC group across nearly all the 100 trained VNN models was likely to be a highly robust contributor to elevated $\Delta$-Age in AD+ group. We used the {\sf fsbrain} package in R to project the robustness of significantly elevated regional residuals for a brain region on the brain template~\cite{schafer2020fsbrain}.

The scale-free architecture of VNNs is facilitated by the coVariance filter (Definition~\ref{def1}). Also, the non-adaptive readout function allows us to transfer VNNs to process datasets of different dimensionalities without any changes to the architecture. We used these properties of our proposed framework to investigate ``transferability of interpretability" across different datasets irrespective of whether they laid within the purview of theoretical guarantees on the transferability of performance of VNNs. This aspect is described next.

\noindent
{\bf Transferability of interpretability.} Note that the analysis  of regional residuals is independent of the performance of VNN model in the regression task. 
Hence, we leveraged the scale-free aspect of VNNs to evaluate whether the statistical effects that facilitated certain brain regions to be deemed significant in the above analyses were preserved after transferring to a dataset of different dimensionality. For this purpose, we transferred a VNN model trained on FTDC datasets to OASIS-3 dataset and plotted the brain regions deemed significant from the analyses of regional residuals of AD+ and HC groups in OASIS-3 on a brain template. If accelerated aging characteristic of AD was the driving factor behind the observed statistical differences in regional residuals of AD+ and HC groups, we expected to observe qualitatively consistent findings on OASIS-3 dataset for VNNs trained on OASIS-3 dataset and those transferred from FTDC datasets to OASIS-3 dataset. This observation would provide evidence for ``transferability of interpretability" for VNNs in brain age across different datasets as VNN models trained on FTDC datasets could transfer their ability to construct interpretable regional residuals from FTDC to OASIS-3. Here, we clarify that qualitatively consistency refers to observing significant group differences between AD+ and HC groups in the same direction (but not necessarily the same effect size) for regional residuals obtained from VNNs trained on OASIS-3 and FTDC datasets. By investigating transferability of interpretability, we could also potentially decouple the contributing factors behind the observed accelerated aging effect in AD+ group from the expectation of achieving a near-perfect performance on chronological age prediction for HC group in existing literature~\cite{bashyam2020mri, yin2023anatomically}. We hypothesized that the transferability of interpretability could be driven by a combination of cortical atrophy and the ability of VNNs to exploit eigenvectors (principal components) of the covariance matrix $\hat\bC_{148}^{\sf AD+}$ that were most relevant to brain age prediction.   


\subsubsection{Individual-level Brain Age Prediction}\label{brain_age}
Existing studies primarily focus on the gap between estimated brain age and the chronological age as a biomarker of pathology~\cite{franke2012longitudinal}, where the brain age is a scalar quantity. VNNs provide an interpretable and systemic approach to brain age via regional analyses as described in Section~\ref{regbage}. However, using VNNs, we can also obtain a scalar estimate for the brain age through a procedure consistent with the existing studies in this domain. To evaluate the brain age from VNN regression output, we first addressed the systemic bias in the gap between $\hat y$ and $y$, where the age may be underestimated for older individuals and overestimated for younger individuals~\cite{beheshti2019bias}. Such bias can confound the interpretations of brain age. Therefore, to correct for this age-driven bias, we adopted a linear regression model based approach to correct any age bias in the VNN age estimates~\cite{de2020commentary,beheshti2019bias}. Under this approach, we followed the following bias correction steps on the VNN estimated age $\hat y$ to obtain the brain age $\hat y_{\sf B}$ for an individual with chronological age $y$ and cortical thickness data~$\bx$:

\noindent
{\bf Step 1.} Fit a linear regression model on the complete dataset to determine coefficients $\alpha$ and $\beta$ in the following linear model:
\begin{align}\label{s1}
     \hat y - y = \alpha y + \beta\;. 
\end{align}
\noindent
{\bf Step 2.} Evaluate brain age as follows:
\begin{align}\label{s2}
   \hat y_{\sf B} = \hat y - (\alpha y + \beta)\;. 
\end{align}
The gap between $\hat y_{\sf B}$ and $y$ is typically the biomarker of interest and is defined below. For an individual with cortical thickness $\bx$ and chronological age $y$, the brain age gap $\Delta$-Age is defined as
\begin{align}
     \Delta\text{-Age}\triangleq \hat y_{\sf B} - y\;,
\end{align}
where $\hat y_{\sf B}$ is determined from the VNN age estimate $\hat y$ from $\Phi(\bx;\bC,\cH)$ and $y$ according to steps in~\eqref{s1} and~\eqref{s2}. The age-bias correction in~\eqref{s1} and~\eqref{s2} was performed for only HC group to account for bias in the VNN estimates due to healthy aging and then applied to the AD+ group. Further, the distributions of $\Delta$-Age were obtained for all individuals in HC and AD+ groups. $\Delta$-Age for AD+ group was expected to be elevated as compared to HC group as a consequence of elevated regional residuals derived from the VNN model. To elucidate this, we consider a toy example where we have two individuals of the same chronological age $y$ with one belonging to the AD+ group and another to the HC group. Equation \eqref{s2} suggests that their corresponding VNN outputs (denoted by $\hat y_{\sf AD+}$ for individual in the AD+ group and $\hat y_{\sf HC}$ for individual in the HC group) are corrected for age-bias using the same term $\alpha y + \beta$. Hence, $\Delta$-Age for the individual in the AD+ group will be elevated with respect to that from the HC group only if the VNN prediction $\hat y_{\sf AD+}$ is elevated with respect to  $\hat y_{\sf HC}$. Since the VNN predictions  $\hat y_{\sf AD+}$  and  $\hat y_{\sf HC}$ can be perceived as unweighted aggregations of the estimates at the regional level (see~\eqref{interpret}), higher $\hat y_{\sf AD+}$ with respect to $\hat y_{\sf HC}$ is a direct consequence of the regional residuals (see~\eqref{res_dist}) being elevated in AD+ group with respect to HC group. When the individuals in this example have different chronological age, the age-bias correction is expected to remove any variance due to chronological age in $\Delta$-Age. Since the age distributions of AD+ and HC group are different, we also verified that the differences in $\Delta$-Age for AD+ and HC group were not driven by age or gender differences via ANCOVA with age and sex as covariates.


\section{Results}\label{results}

\subsection{Transferability of VNNs for regression task}\label{vnn_regress}
We evaluated the transferability of VNN models trained as regression models on the cortical thickness data of the HC group from FTDC datasets across different resolutions of Schaefer's atlas. To begin with, we remark that the series of covariance matrices formed by cortical thickness features extracted according to $100-500$ parcellation for HC group in FTDC datasets was converging (Appendix~\ref{fig_converge}). This assessment was pertinent as our theoretical results in Theorem~\ref{transferthm} hold for a converging sequence of covariance matrices. 
Furthermore, we also evaluated the distance of covariance matrices from FTDC datasets  from the covariance matrices derived from the HC group of OASIS-3 dataset ($\hat\bC_{148}$ for DKT atlas). Here, the distances between the covariance matrix derived from OASIS-3 dataset and those from FTDC datasets were significantly greater than the pairwise distances for the covariance matrices associated with different resolutions of Schaefer's atlas (Fig.~\ref{cutmetrics} in Appendix~\ref{fig_converge}). This observation could potentially be explained by the differences in constructions of Schaefer's atlas and DKT atlas. Specifically, DKT atlas is an anatomic atlas that maps both gyral and sulcal regions~\cite{destrieux2010automatic}. In contrast, Schaefer's atlas is derived from a gradient weighted Markov random field model using functional MRI data~\cite{schaefer2018local} that does not have anatomic boundaries per se. Due to these significant differences in construction, we did not expect the anatomical covariance matrices from the data curated according to DKT atlas to be a part of the converging sequence of covariance matrices from FTDC datasets. Thus, based on the discussion thus far, we expected the transferability property of VNNs to hold for statistical inference tasks on FTDC datasets. We also performed exploratory analysis for transferability from FTDC to OASIS-3 dataset for VNNs trained on FTDC datasets and vice-versa to assess the degradation in regression performance.

In order to investigate transferability of VNNs, we trained the VNN models for a regression task between cortical thickness and chronological age for individuals in HC group according to the procedure described in Section~\ref{vnn_learn} for FTDC100, FTDC300, FTDC500, and OASIS-3 datasets. This resulted in $100$ nominal VNN models for each dataset (each model trained on a different permutation of the training set). The readout layer in the VNNs was non-adaptive and it evaluated the unweighted mean of the outputs of the final VNN layer to form an estimate for chronological age. Therefore, the trained VNN could readily process a dataset with different number of features without any retraining or alteration in the architecture. The performance outcomes were quantified in terms of mean absolute error (MAE) and Pearson's correlation between the VNN output and ground truth.
 
\begin{table}[!htbp]
\caption{Transferability across datasets (MAE for VNN regression outputs with respect to the ground truth).}\label{transfer_tbl1}
\renewcommand{\arraystretch}{1.5}
\centering
{\small
\begin{tabular}{|c| c| c| c|c|c|}
\hline
\multirow{2}{*}{\backslashbox{Training}{Testing}} & \multirow{2}{*}{FTDC100 (HC)} & \multirow{2}{*}{FTDC300 (HC)} & \multirow{2}{*}{FTDC500 (HC)} & \multirow{2}{*}{OASIS-3 (HC)} \\ 
& & & &  \\
 \hline
FTDC100 (HC) & \cellcolor{red!25}{$5.39 \pm 0.084$} & $5.5 \pm 0.101$ & $5.61 \pm 0.132$&   $7.55\pm 0.356$ \\
 \hline
 FTDC300 (HC) & $5.39 \pm 0.193$ & \cellcolor{red!25}{$5.41 \pm 0.167$} &  {$5.47\pm 0.169$} & $7.22\pm 0.4$ \\
 \hline
FTDC500 (HC) & $5.43 \pm 0.2$ & $5.38 \pm 0.15$ &  \cellcolor{red!25}{$5.4\pm 0.169$} & $7.14\pm 0.6$ \\
 \hline
OASIS-3 (HC)  & $6.82 \pm 0.365$ & $6.98\pm 0.42$ &  {$6.47\pm 0.34$} & \cellcolor{red!25}{$5.72\pm 0.076$} \\ 
 \hline
\end{tabular}}
\end{table}

\begin{table}[!htbp]
\caption{Transferability across datasets (Pearson's correlation between VNN outputs and ground truth).}\label{transfer_tbl2}
\renewcommand{\arraystretch}{1.5}
\centering
{\small
\begin{tabular}{|c| c| c| c|c|c|}
\hline
\multirow{2}{*}{\backslashbox{Training}{Testing}} & \multirow{2}{*}{FTDC100 (HC)} & \multirow{2}{*}{FTDC300 (HC)} & \multirow{2}{*}{FTDC500 (HC)} & \multirow{2}{*}{OASIS-3 (HC)} \\ 
& & & & \\
 \hline
FTDC100 (HC)& \cellcolor{red!25}{$0.49 \pm 0.017$} & $0.47 \pm 0.018$ & $0.468 \pm 0.018$ & $0.387\pm0.021$ \\ 
 \hline
  FTDC300 (HC) & $0.498 \pm 0.05$ & \cellcolor{red!25}{$0.49 \pm 0.042$} &  {$0.486\pm 0.04$} & $0.374\pm 0.06$\\ 
  \hline
FTDC500 (HC) & $0.51 \pm 0.021$ & $0.509 \pm 0.02$ & \cellcolor{red!25}{$0.51\pm0.021$} & $0.374\pm 0.031$ \\
 \hline
OASIS-3 (HC)  & $0.454 \pm 0.006$ & $0.437\pm 0.004$ &  {$0.432\pm 0.003$} & \cellcolor{red!25}{$0.43\pm 0.012$} \\ 
 \hline
\end{tabular}}
\end{table}

\begin{figure}[!htbp]
  \centering 
  \includegraphics[scale=0.45]{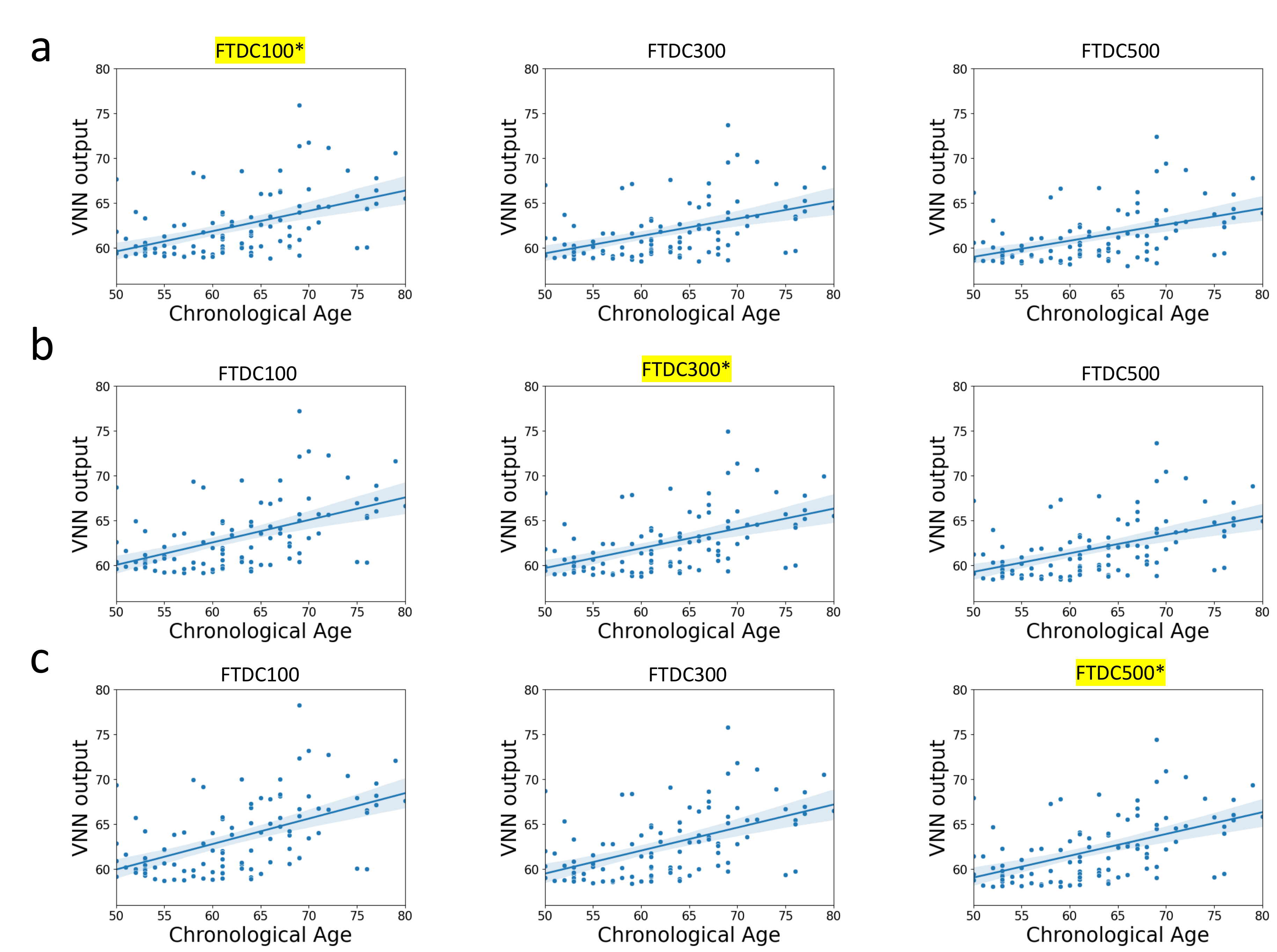}
   \caption{{\bf VNN regression outputs versus chronological age after transferring VNNs across FTDC datasets.} Panel~{\bf a} displays the results obtained by averaging the outputs of all 100 VNN models that were trained on FTDC100 dataset and tested on FTDC datasets curated according to different resolutions of Schaefer's atlas. Mean of regression outputs for all 100 VNN outputs versus the ground truth (chronological age) are displayed when the VNNs process FTDC100 dataset ($m=100$; marked by an asterisk since VNNs were trained on this dataset), FTDC300 dataset ($m=300$) and FTDC500 dataset ($m=500$). Panel~{\bf b} and Panel~{\bf c} displays similar results as in Panel~{\bf a} for the VNNs that were trained on FTDC300 and FTDC500 datasets, respectively. }
   \label{vnn_tranfer_fig}
\end{figure}

We tabulate MAE in Table~\ref{transfer_tbl1} and Pearson's correlation between ground truth and VNN output in Table~\ref{transfer_tbl2}. For both tables, the row ID provides the dataset on which VNN models were trained and the column ID indicates the dataset for which the VNN performance is reported (after transferring the VNNs if training and testing datasets are different). For instance, the element with row ID ``FTDC100" and column ID ``FTDC300" in Table~\ref{transfer_tbl1} represents the mean and standard deviation of MAE evaluated on FTDC300 dataset ($m=300$) for the 100 nominal VNN models trained on FTDC100 dataset ($m=100$). The elements with same row ID and column ID in Table~\ref{transfer_tbl1} and Table~\ref{transfer_tbl2} provide the baseline performance to gauge performance after transferring VNNs.

The results in Table~\ref{transfer_tbl1} and Table~\ref{transfer_tbl2} show that the performance of VNNs in terms of MAE and correlation between VNN output and ground truth was preserved after transferring VNNs across FTDC datasets that were curated according to different resolutions of Schaefer's atlas. The transferability of VNNs across FTDC datasets was corroborated by Fig.~\ref{vnn_tranfer_fig}, which illustrates the plots of the means of the outputs of $100$ VNN models that were trained on FTDC100 dataset (or FTDC300 or FTDC500 datasets) and used to process FTDC100, FTDC300, and FTDC500 datasets. We also remark that this experiment is not feasible for PCA-regression models as the principal components and the regression model from one dataset cannot be naively transferred to process another dataset that has a different number of features. Hence, the observations on FTDC datasets in Table~\ref{transfer_tbl1}, Table~\ref{transfer_tbl2}, and Fig.~\ref{vnn_tranfer_fig} validate our theoretical results regarding transferability of VNNs. The plot for VNN output versus chronological age for the HC group in OASIS-3 dataset is included in Fig.~\ref{brain_age_AD} in Appendix~\ref{brain_age_supp}.

Moreover, Table~\ref{transfer_tbl1} and Table~\ref{transfer_tbl2} also report the performance of VNN models trained on FTDC datasets after transferring to OASIS-3 dataset and vice-versa. As expected from the observations in Fig.~\ref{cutmetrics}b, the VNN performances degraded significantly when processing the datasets curated according to DKT atlas. 
However, the VNN models trained on FTDC100, FTDC300 or FTDC500 datasets retained significant correlation between predicted age and true age for OASIS-3 dataset (Table~\ref{transfer_tbl2}) and achieved an MAE of $<8$ years (Table~\ref{transfer_tbl1}). Similar observations were true for VNN models trained on OASIS-3 dataset when transferred to FTDC datasets. These observations indicated that the outputs of VNN models trained on FTDC datasets may not be scaled appropriately to transfer to OASIS-3 dataset or vice-versa and thus, a mapping or a transformation between Schaefer's atlas and DKT atlas is necessary to transfer regression models optimally. Nevertheless, our experiments imply that the VNN models trained on FTDC datasets (or OASIS-3 dataset) possessed at least some information about chronological age that could be transferred when processing the OASIS-3 datasets (or FTDC datasets). This aspect is explored further in our results in Section~\ref{vnn_age_eig}. Although not the focus of our paper, the observations above provide some context into why fine-tuning pre-trained deep learning models for a specific application may be useful for statistical inference from neuroimaging data~\cite{yang2022data,wein2021graph,ardalan2022transfer}.

\subsection{VNN regression model outputs for HC group in OASIS-3 are correlated with the first eigenvector of~$\hat\bC_{148}$}\label{vnn_age_eig}

Our results in Lemma~\ref{lm1} suggested that VNN based statistical inference draws conceptual similarities with PCA-driven analysis. Hence, we further investigated whether the regression performance by VNNs in Table~\ref{transfer_tbl1} could be characterized by contributions of the principal components of the anatomical covariance matrix. We focus our discussion here on the VNNs trained on OASIS-3 dataset and VNNs transferred from FTDC100 dataset to OASIS-3 dataset (among the FTDC datasets, FTDC100 dataset had the largest ratio between the number of samples and number of features). However, the observations made by experiments on VNNs trained on FTDC300 or FTDC500 datasets were similar to that for those trained on FTDC100 and are included in Appendix~\ref{inner_prod_all}.

Recall that the final regression output by VNNs is formed by an unweighted average function as a readout function. Thus, we can equivalently represent the functionality of the readout as a simple aggregation of the contributions of different features or brain regions to the final estimate formed by the VNN (see~\eqref{trvnn1}). Hence, for every individual, we evaluated the mean of the inner products (also equivalently referred to as dot product between vectors) between the vectors of contributions of every brain region with the principal component of the covariance matrix $\hat\bC_{148}$ for all $100$ VNN models. Note that a vector of regional contributions was of the same length as the number of cortical thickness features (i.e., 148 for OASIS-3) and therefore, each element of this vector was associated with a distinct brain region.  We use the notation $\bp_{\sf HC}$ to represent the population of vectors obtained from the HC group in OASIS-3. To evaluate the inner product, we used $\bar\bp_{\sf HC}$, which was obtained from  $\bp_{\sf HC}$ after normalization (norm = 1). We denote the population of inner products across  the HC group in OASIS-3 by $|\!<\!\bar \bp_{\sf HC}, \bv_i\!>\!|$ for an eigenvector $\bv_i$ of $\hat\bC_{148}$. Note that since all vectors in $\bar \bp_{\sf HC}$ were normalized to have norm $1$ and the eigenvectors of $\hat\bC_{148}$ were of length $1$ by default, $|\!<\!\bar \bp_{\sf HC}, \bv_i\!>\!|$ represented the population of the cosine of the angles between the vectors in $\bp_{\sf HC}$ and eigenvectors $\bv_i$ of $\hat\bC_{148}$ across the HC group in OASIS-3.   

Figure~\ref{vnn_hc_eig_fig}a plots the mean of the inner products observed across the HC group for the first $30$ eigenvectors of $\hat\bC_{148}$ for VNNs that were trained on OASIS-3 dataset. The alignment between the first eigenvector of $\hat\bC_{148}$ and the vectors of regional contributions to the VNN output was significantly stronger as compared to other eigenvectors ($0.991 \pm 0.0003$ across the HC group) with relatively smaller associations for eigenvectors $\bv_2$ ($0.04\pm 0.0029$), $\bv_3$ ($0.0746 \pm 0.0023$), and $\bv_4$ ($0.0716 \pm 0.0026$). Thus, it could be concluded that the VNNs primarily leveraged the information along the first principal component  of  $\hat\bC_{148}$ (i.e., $\bv_1$) to achieve the performance in Table~\ref{transfer_tbl1} and Table~\ref{transfer_tbl2}. Figure~\ref{vnn_hc_eig_fig}b illustrates the projection of $\bv_1$ on a brain template. In Fig.~\ref{vnn_hc_eig_fig}b, $\bv_1$ predominantly included bilateral anatomic brain regions in the parahippocampal gyrus, precuneus, inferior medial temporal gyrus, and precentral gyrus. 

Next, we discuss the correlations between the eigenvectors of $\hat\bC_{148}$ 
 and the regional contributions derived from VNNs that were transferred from FTDC100 to OASIS-3. Recall from Table~\ref{transfer_tbl1} and Table~\ref{transfer_tbl2} that the VNNs that were transferred from FTDC100 to OASIS-3 had a diminished transference of performance in the task of chronological age prediction on the OASIS-3 dataset. However, similar to Fig.~\ref{vnn_hc_eig_fig}a, the first principal component $\bv_1$ had the largest correlation with the regional contributions derived from the VNNs in this setting ($0.985 \pm 0.0015$) and relatively smaller associations were observed for $\bv_3$ ($0.072 \pm 0.004$) 
 and $\bv_4$ ($0.111\pm 0.005$). Figure~\ref{vnn_hc_eig_fig}c displays the mean of the inner products observed across the HC group for the first $30$ eigenvectors of $\hat\bC_{148}$ obtained by the VNNs that were transferred from FTDC100 to OASIS-3 dataset. Thus, despite diminished quantitative transferability between FTDC and OASIS-3 datasets (Table~\ref{transfer_tbl1} and Table~\ref{transfer_tbl2}), the VNNs that were transferred from the FTDC100 dataset to OASIS-3 dataset leveraged the first eigenvector of $\hat\bC_{148}$ to form the chronological age predictions for HC group in OASIS-3 dataset. This observation indicated that the VNNs trained to predict chronological age for one dataset could gain the ability to leverage the relevant eigenvectors of the anatomical covariance matrix when transferred to process another dataset irrespective of the transferability of performance over the task of chronological age prediction.
\begin{figure}[!htbp]
  \centering 
  \includegraphics[scale=0.5]{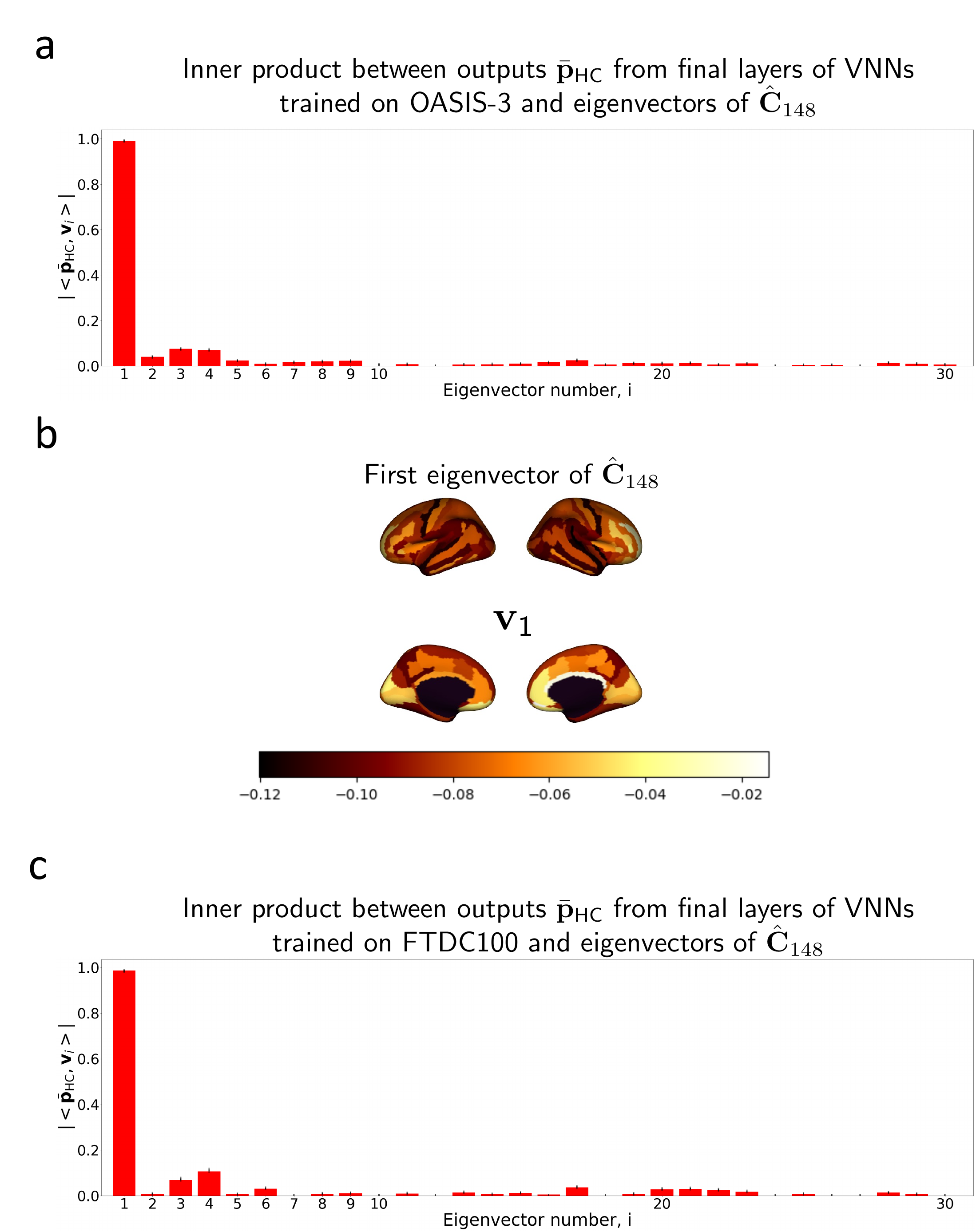}
   \caption{{\bf Inner product between the normalized vector of regional contributions to the VNN outputs ($\bar\bp_{\sf HC}$) and eigenvectors of $\hat\bC_{148}$ (anatomical covariance matrix for HC group in OASIS-3).} Panel~{\bf a} illustrates a bar plot for $|\!<\!\bar\bp_{\sf HC}, \bv_i\!>\!|$ for $i\in \{1,\dots, 30\}$, where $\bv_i$ is the $i$-th eigenvector (principal component) of covariance matrix $\hat\bC_{148}$ and associated with $i$-largest eigenvalue in terms of magnitude and the vectors of regional contributions, $\bar\bp_{\sf HC}$ were obtained by VNNs that were trained on OASIS-3 dataset. The inner product results for eigenvectors with coefficient of variation greater than $30\%$ across the HC group of OASIS-3 were excluded (and hence, their respective entries set as $0$). For every individual in HC group, the associations between their corresponding vector of regional contributions, $\bar\bp_{\sf HC}$ and eigenvectors of $\hat\bC_{148}$ were evaluated over $100$ nominal VNN models. The first eigenvector ($\bv_1$) had the largest association with $\bar\bp_{\sf HC}$ ($0.991 \pm 0.0003$ across HC group). The eigenvector $\bv_1$ is plotted on a brain template in Panel~{\bf b}. Panel~{\bf c} displays a bar plot for $|\!<\!\bar\bp_{\sf HC}, \bv_i\!>\!|$ for the first $30$ eigenvectors of $\hat\bC_{148}$ and the vectors of regional contributions, $\bar\bp_{\sf HC}$ in the scenario when VNNs were transferred from FTDC100 to OASIS-3 dataset.}
   \label{vnn_hc_eig_fig}
\end{figure}

\subsection{Brain Age Prediction using VNNs in OASIS-3}
Next, we performed the regional analyses of residuals obtained from VNN models for individuals in HC and AD+ groups according to the procedure discussed in Section~\ref{regbage}. We remark that several existing studies on brain age prediction have utilized models that achieved better MAE on the HC group than our results in Table~\ref{transfer_tbl1} and Table~\ref{transfer_tbl2}~\cite{yin2023anatomically,bashyam2020mri,couvy2020ensemble,liem2017predicting}. 
However, our primary focus here is on demonstrating that the outputs of the final layers of VNNs trained on data from HC group enabled us to associate contributions to brain regions that were characteristic of neurodegeneration in AD. 
Thus, this section summarizes the insights provided by VNNs that have neither been explored nor feasible for most existing brain age evaluation frameworks based on black box implementation of deep learning models. 

To start with, we leveraged the VNN models trained on data from HC group in OASIS-3 in Section~\ref{vnn_regress} to process a combined cortical thickness dataset from HC and AD+ groups. Also, in this scenario, the covariance matrix was formed from the combined dataset ($\hat\bC_{148}^{\sf AD+}$ from cortical thickness data from HC and AD+ groups). This choice was made in order to incorporate the neurodegeneration-driven changes in the anatomical covariance into the model. Next, for every VNN model, we evaluated the distribution of regional residuals in AD+ and HC groups as described in Section~\ref{regbage}. An example of the outputs of the regional analysis of a VNN model is included in Appendix~\ref{regional_illust}. 

We note that the regional residuals derived from VNNs are independent of the MAE performance of VNN models in predicting chronological age for the individuals in  HC group. 
Therefore, besides analyzing brain age in OASIS-3 dataset, we also performed exploratory regional analyses to assess the interpretability offered by VNN models that were trained on FTDC datasets and transferred to process OASIS-3 dataset. We aimed to discern the interpretability offered by VNNs by analyzing the group differences in regional residuals between AD+ and HC groups and the association of regional residuals with CDR sum of boxes in the AD+ group.

\subsubsection{Regional analyses of outputs of VNNs reveal brain regions characteristic of Alzheimer's disease}\label{results_regional}
Our results for the OASIS-3 dataset showed that the brain regions characteristic of AD had significant differences in regional residual distributions (defined in~\eqref{res_dist}) for AD+ group as compared to HC group. Note that the results presented in this section were checked for robustness using the family of 100 nominal models trained to predict chronological age of individuals in the HC group in OASIS-3 (described in Section~\ref{vnn_regress}). 

Figure~\ref{roi_AD}a illustrates the significant group differences in cortical thickness for AD+ and HC groups (t-test, t-statistic for regions with Bonferroni corrected $p$-value $<0.05$ are projected on the brain template). Since AD+ group is characterized by accelerated loss of cortical thickness in various brain regions, we expected the outcomes of the analysis of regional residuals to have significant overlap with cortical atrophy patterns in Fig.~\ref{roi_AD}a. Figure~\ref{roi_AD}b displays the robustness (determined via analyses of 100 VNN models) for various brain regions in having an elevated regional effect in their corresponding residual elements for AD+ group with respect to HC group. The significance of brain regions in each VNN model was determined via comparisons between HC and AD+ groups and from VNN final layer outputs, where the parameters of the VNNs were determined via training to predict chronological age for the HC group in OASIS-3. The most significant regions with elevated regional residuals in AD+ with respect to HC were concentrated in bilateral inferior parietal, temporal, entorhinal, parahippocampal, precuneus, subcallosal, and precentral regions.  All these brain regions, except for precentral and subcallosal, mirrored the cortical atrophy in Fig.~\ref{roi_AD}a, and these regions are known to be highly  interconnected with hippocampus~\cite{lindberg2012subcallosal}. 



Our findings here are also consistent with the existing studies that study anatomic changes in AD. Inferior parietal and medial temporal regions are well known regions of atrophy for the typical variant of Alzheimer's disease~\cite{ossenkoppele2015behavioural,ferreira2020biological}. Cortical atrophy in these regions is also accompanied by tau protein deposition~\cite{hansson2017tau}. Precuneus and parahippocampal regions are implicated in the early stages of AD~\cite{ryu2010measurement, echavarri2011atrophy} and precuneus is also a region of interest in individuals with early-onset AD~\cite{karas2007precuneus}. The aforementioned brain regions are also assigned interpretability in the study in~\cite{yin2023anatomically}. 

Although the results in Fig.~\ref{roi_AD}b provided a meaningful regional profile for AD+ group, we further performed exploratory analysis to check whether the regional residuals had any clinical significance. To this end, we evaluated the correlations between CDR sum of boxes and the regional residuals derived from final layer VNN outputs for the AD+ group for all 100 VNN models. 
Figure~\ref{roi_AD}c plots the means of the Pearson's correlations observed between regional residuals and CDR sum of boxes across the 100 VNN models for AD+ group on the brain template. Interestingly, the brain regions with the most significant correlations aligned with the most robust brain regions in the regional profile corresponding to elevated regional residuals in the AD+ group in Fig.~\ref{roi_AD}b. We clarify that our analysis in this context was not restricted to the highlighted brain regions in Fig.~\ref{roi_AD}b. Thus, the regional profile extracted from VNN final layer outputs in Fig.~\ref{roi_AD}b also captured information about the dementia severity. Next, we checked if the regional findings in Fig.~\ref{roi_AD}b and Fig.~\ref{roi_AD}c translated into elevation in $\Delta$-Age and the correlation between $\Delta$-Age and CDR sum of boxes for AD+ group.

\begin{FPfigure}
  \centering 
  \includegraphics[scale=0.45]{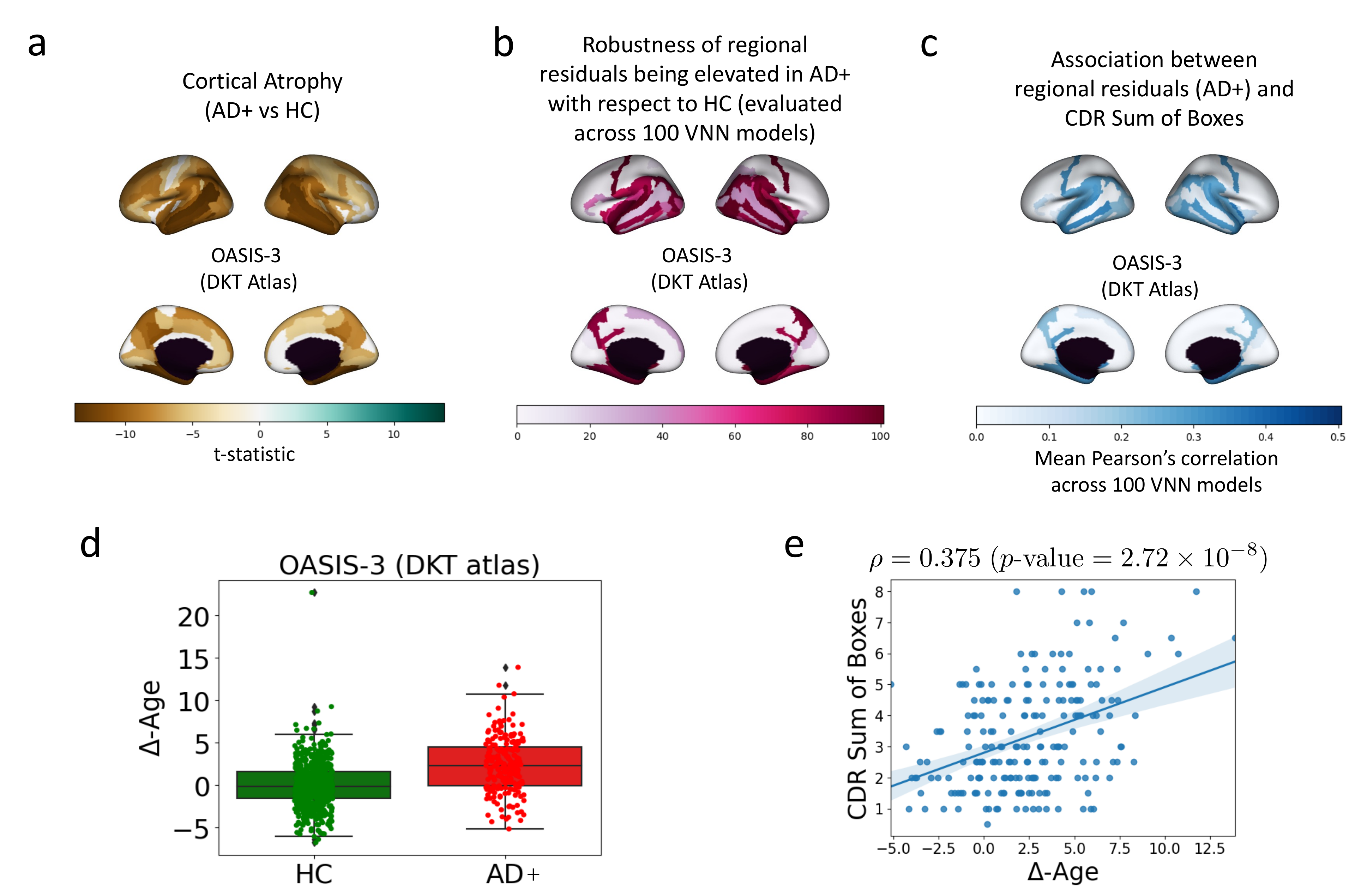}
   \caption{{\bf Interpretable $\Delta$-Age evaluation in OASIS-3 dataset.} Panel~{\bf a} displays the results of group differences in cortical thickness between AD+ and HC groups. Regions with significant differences (two-sided t-test, Bonferroni corrected $p$-value $<0.05$) are identified and the corresponding $t$-statistics are projected on a brain template. Negative $t$-statistic for a brain region suggests that the AD+ group had significant cortical atrophy in that region as compared to HC group. Panel~{\bf b} displays the robustness of the significantly elevated regional residuals for AD+ group with respect to HC group for different brain regions. For every VNN model in the set of $100$ nominal VNN models that were trained on HC group in OASIS-3, the analyses of regional residuals helped isolate brain regions that corresponded to significantly elevated regional residuals in AD+ group with respect to HC group (ANOVA: Bonferroni corrected $p$-value $<0.05$, ANCOVA with age and sex as covariates: $p$-value $< 0.05$). After performing this experiment for $100$ VNN models, the robustness of the observed significant effects in a brain region was evaluated by calculating the number of times a brain region was identified to have significantly elevated regional residuals in AD+ group with respect to HC group. The number of times a brain region was linked with significantly elevated regional residuals in AD+ group with respect to HC group is projected on the brain template and can be perceived as a marker of its robustness as a contributor to accelerated aging in AD+ group.  Panel~{\bf c} projects the mean Pearson's correlation between regional residuals (derived for each VNN model in the set of 100 nominal VNN models) and CDR sum of boxes for AD+ group on the brain template. Panel~{\bf d} displays the distribution of $\Delta$-Age for HC and AD+ groups. The gap between the mean of $\Delta$-Age for AD+ and HC is $2.36$ years and significant (Cohen's $d$ = 0.797). The elevated brain age effect here is characterized by regional profile in Panel~{\bf b}. Panel~{\bf e} displays the scatter plot for CDR sum of boxes versus $\Delta$-Age in AD+ group.  The correlation between $\Delta$-Age and CDR sum of boxes could be attributed to the correlations observed between regional residuals and CDR sum of boxes in Panel~{\bf c}.}
   \label{roi_AD}
\end{FPfigure}

\subsubsection{$\Delta$-Age is elevated in AD+ group and correlated with CDR}
We evaluated the $\Delta$-Age for HC and AD+ groups in OASIS-3 according to the procedure specified in Section~\ref{brain_age}. We also investigated the Pearson's correlation between $\Delta$-Age and CDR sum of boxes scores in AD+ group. Figure~\ref{roi_AD}d illustrates the distributions for $\Delta$-Age for HC and AD+ groups ($\Delta$-Age for HC: $0\pm 2.66$ years, $\Delta$-Age for AD+: $2.36\pm 3.22$ years). The difference in $\Delta$-Age for AD+ and HC groups was significant (Cohen's $d$ = 0.797, ANCOVA with age and sex as covariates: $p$-value = $3.98 \times 10^{-22}$, partial $\eta^2 = 0.103$). Also, the $p$-value associated with age in ANCOVA was $0.61$ and hence, group difference in $\Delta$-Age was not driven by the difference in the distributions of chronological age for HC and AD+ groups. Figure~\ref{roi_AD}e plots the scatter plot between CDR sum of boxes score and $\Delta$-Age for the AD+ group. The Pearson's correlation between CDR sum of boxes score and $\Delta$-Age was $0.375$ ($p$-value $= 2.72 \times 10^{-8}$), thus, implying that the $\Delta$-Age evaluated for AD+ group captured information about dementia severity. Hence, as expected, the $\Delta$-Age for AD+ was likely to be larger with an increase in CDR sum of boxes scores. For instance, the mean of $\Delta$-Age for individuals with CDR sum of boxes greater than 4 was $3.76$ years, and for those with CDR sum of boxes smaller than or equal to 4 was $1.79$ years.

Given that the age-bias correction procedure is a linear transformation of VNN outputs, it can readily be concluded that the statistical patterns for regional residuals in Fig.~\ref{roi_AD}b and Fig.~\ref{roi_AD}c lead to elevated $\Delta$-Age in Fig.~\ref{roi_AD}d and correlation between $\Delta$-Age and CDR sum of boxes scores in Fig.~\ref{roi_AD}e. Thus, a major implication of the findings in Fig.~\ref{roi_AD} is that certain regions contribute abnormally in the context of AD-related neurodegeneration as compared to HC group to the age estimate formed by the VNN models and therefore, the analysis of regional residuals derived from the outputs at the final layer of a VNN model provides a feasible way to characterize accelerated biological aging in AD+ group with a regional profile. Additional figures and details pertaining to VNN outputs and brain age before and after the age-bias correction was applied are included in Appendix~\ref{brain_age_supp} and they demonstrate the distributions of $\Delta$-Age and brain age across HC and AD+ groups in OASIS-3.

We have also previously reported in~\cite{sihag2022predicting} that the gap between $\Delta$-Age for AD, MCI and HC groups was preserved across different resolutions of Schaefer's atlas after transferring VNN models. The results in our previous work in~\cite{sihag2022predicting} were a direct consequence of the transferability of VNNs across different scales of Schaefer's atlas. For completeness, we further demonstrate that there was spatial consistency between regional profiles for AD that were evaluated across different resolutions of Schaefer's atlas for a distinct dataset in Appendix~\ref{explor}.

\subsection{Robustness of regional profiles of elevated $\Delta$-Age in AD+ group hinges on the correlations between regional residuals and principal components of anatomical covariance matrix $\hat\bC_{148}^{\sf AD+}$}
Thus far, we have independently shown that VNNs trained as regression models are transferable across multi-scale datasets without any re-training and VNNs provide an interpretable perspective to accelerated brain age in neurodegeneration. However, despite lack of quantitative transferability between FTDC and OASIS-3 datasets, our results in Fig.~\ref{vnn_hc_eig_fig} showed that the VNNs that were transferred from FTDC100 to OASIS-3 and the VNNs trained on OASIS-3 exploited similar eigenvectors of the anatomical covariance matrix to form predictions for chronological age for HC group.  Here, we explore whether the regional interpretability provided by VNNs for $\Delta$-Age depends on the principal components of the anatomical covariance matrix~$\hat\bC_{148}^{\sf AD+}$.

 \subsubsection{Regional residuals derived from VNNs trained on OASIS-3 are correlated with principal components of anatomical covariance matrix}
We started by investigating the relationship between regional residuals derived from VNNs trained on OASIS-3 and the principal components of $\hat\bC_{148}^{\sf AD+}$. For this purpose, we evaluated the inner product of normalized residual vectors (norm = 1) obtained from VNNs trained on OASIS-3 and the eigenvectors of the covariance matrix $\hat\bC_{148}^{\sf AD+}$ for the individuals in AD+ group to determine whether any specific eigenvectors (principal components) of $\hat\bC_{148}^{\sf AD+}$ were instrumental to recover the findings in Fig.~\ref{roi_AD}b. The normalized residual vector is denoted by $\bar\br_{\sf AD+}$. For every individual, the mean of the absolute value of the inner product $|\!<\!\bar\br_{\sf AD+}, \bv_i\!>\!|$ (where $\bv_i$ is the $i$-th eigenvector of $\hat\bC_{148}^{\sf AD+}$) was evaluated for the 100 VNN models that were used to derive the results in Fig.~\ref{roi_AD}. Figure~\ref{corr_eig_plots_oasis}a plots the mean inner product for eigenvectors associated with $50$ largest eigenvalues of $\hat\bC_{148}^{\sf AD+}$. The three largest mean correlations with the regional residuals in AD+ group were observed for the fourth principal component of $\hat\bC_{148}^{\sf AD+}$ (${|\!<\!\bar\br_{\sf AD+}, \bv_4\!>\!| = 0.6\pm 0.006}$), third principal component ($|\!<\!\bar\br_{\sf AD+}, \bv_3\!>\!| = 0.384\pm 0.007$), and the first principal component  ($|\!<\!\bar\br_{\sf AD+}, \bv_1\!>\!| = 0.294\pm 0.0003$). These principal components are plotted on a brain template in Fig.~\ref{corr_eig_plots_oasis}b. Interestingly, the largest association was observed for the fourth eigenvector of $\hat\bC_{148}^{\sf AD+}$, which spanned the subcallosal and medial frontal regions in the right hemispheres. The third eigenvector spanned the subcallosal region in the right hemisphere and temporal pole and parahippocampal region in both hemispheres. The first principal component of $\hat\bC_{148}^{\sf AD+}$ spanned brain regions similar to that by the first principal component of $\hat\bC_{148}$ (Fig.~\ref{vnn_hc_eig_fig}b). These observations suggested that the observed regional effects in subcallosal, parahippocampal and temporal pole regions were the most significant AD-related contributors to the observed elevated $\Delta$-Age effect in Fig.~\ref{roi_AD}. These brain regions were also implicated as contributors to elevated brain age in AD in an independent dataset in Appendix~\ref{explor}.


\begin{figure}[!htbp]
  \centering
  \includegraphics[scale=0.43]{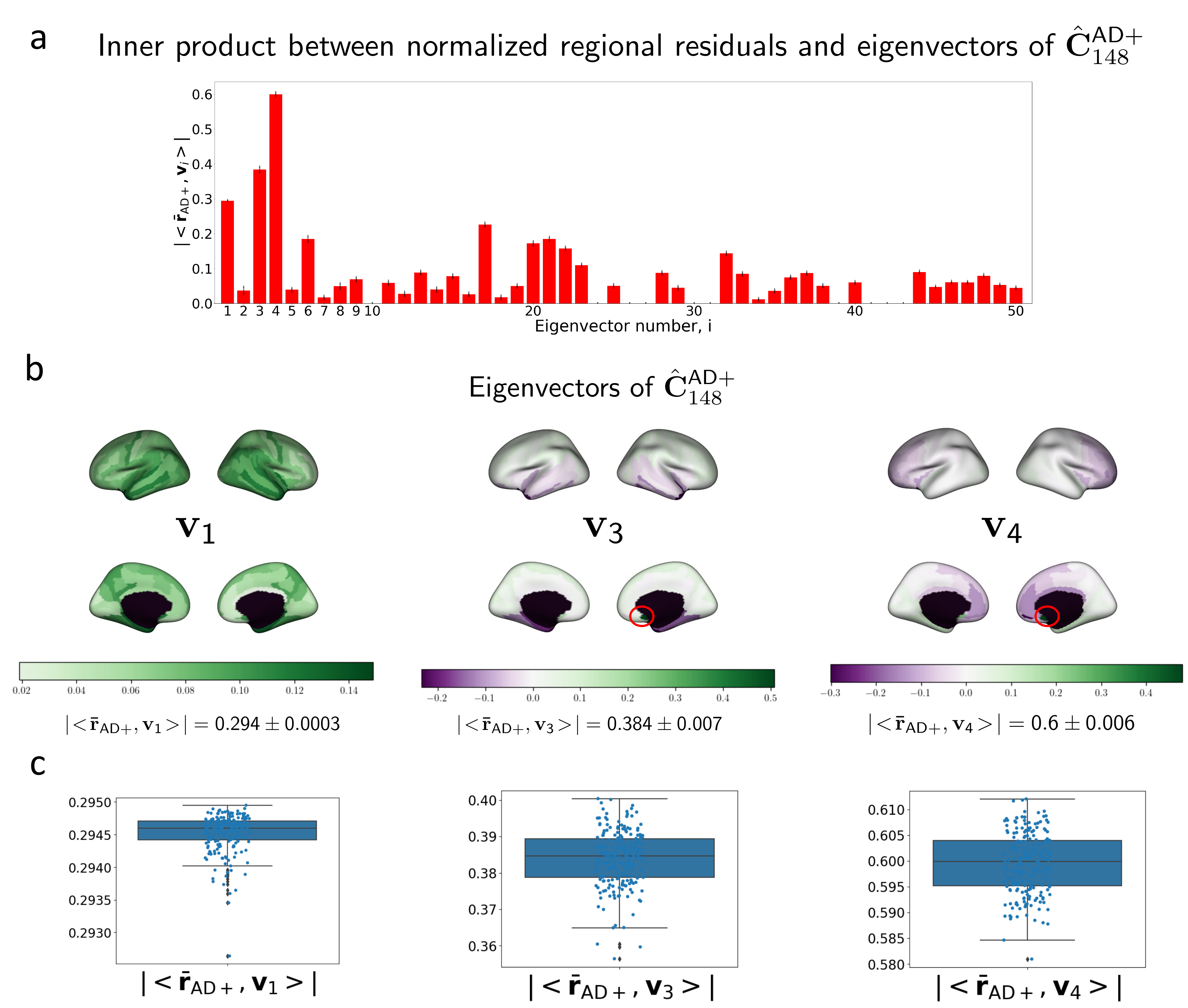}
   \caption{{\bf Mean inner product between the normalized vector of regional residuals (norm = 1) of VNN outputs (VNNs trained on OASIS-3) obtained from AD+ group and the eigenvectors of $\hat\bC_{148}^{\sf AD+}$ (covariance matrix of combined HC and AD+ group), and eigenvector plots.} Panel~{\bf a} illustrates a bar plot for $|\!<\!\bar\br_{\sf AD+},\bv_i\!>\!|$ for $i\in \{1,\dots, 50\}$, where $\bv_i$ is the $i$-th eigenvector of covariance matrix $\hat\bC_{148}^{\sf AD+}$ associated with its $i$-largest eigenvalue. The bars are evaluated from the mean of  $|\!<\!\bar\br_{\sf AD+},\bv_i\!>\!|$ obtained for individuals in the AD+ group (results for eigenvectors associated with coefficient of variation of $|\!<\!\bar\br_{\sf AD+},\bv_i\!>\!|$ larger than $30\%$ across the AD+ group excluded). For every individual in AD+ group, the associations of its regional residuals with eigenvectors of $\hat\bC_{148}^{\sf AD+}$ were evaluated as the mean of those over $100$ nominal VNN models (trained on the OASIS-3 dataset).  The eigenvectors associated with top three largest values for $|\!<\!\bar\br_{\sf AD+},\bv_i\!>\!|$ are plotted on the brain template in Panel~{\bf b}. Subcallosal region in the right hemisphere was associated with the element with the largest magnitude in $\bv_3$ and $\bv_4$ and is highlighted with a red circle in the corresponding plots. Panel~{\bf c} displays the boxplots for the distributions of $|\!<\!\bar\br_{\sf AD+},\bv_i\!>\!|$ for $\bv_1$, $\bv_3$, and $\bv_4$ across the AD+ group.}
   \label{corr_eig_plots_oasis}
\end{figure}

\subsubsection{Regional residuals derived from VNNs transferred from FTDC100 to OASIS-3 are correlated with principal components of $\hat\bC_{148}^{\sf AD+}$}


For this set of experiments, we evaluated the regional residuals for all individuals in the OASIS-3 dataset using VNNs that were trained on FTDC100 dataset and transferred to process OASIS-3. Figure~\ref{corr_eig_plots_ftdc_oasis} displays the bar plot for the mean of the inner product between the regional residuals for AD+ group and first $30$ eigenvectors of $\hat\bC_{148}^{\sf AD+}$. Similar to Fig.~\ref{corr_eig_plots_oasis}a, the regional residuals derived from VNNs that were transferred from FTDC100 to OASIS-3 had the largest correlations with the first, third, and fourth eigenvectors of the anatomical covariance matrix $\hat\bC_{148}^{\sf AD+}$ (${|\!<\!\bar\br_{\sf AD+}, \bv_4\!>\!| = 0.56\pm 0.013}$, ${|\!<\!\bar\br_{\sf AD+}, \bv_3\!>\!| = 0.362\pm 0.014}$, ${|\!<\!\bar\br_{\sf AD+}, \bv_1\!>\!| = 0.29\pm 0.004}$ ). These observations implied that the VNN models trained on FTDC100 dataset were able to leverage the relevant eigenvectors for brain age prediction in OASIS-3 dataset in the formation of regional residuals. From the findings in Fig.~\ref{corr_eig_plots_oasis}a and Fig.~\ref{corr_eig_plots_ftdc_oasis}, we can conclude that the VNN models trained to predict chronological age in both OASIS-3 and FTDC100 were able to leverage the principal components of $\hat\bC_{148}^{\sf AD+}$ that were associated with the regional profile of elevated $\Delta$-Age in Fig.~\ref{roi_AD}. The first, third, and fourth eigenvectors of $\hat\bC_{148}^{\sf AD+}$ were also the three most significantly associated with the regional residuals derived from VNNs that were transferred from FTDC300 or FTDC500 datasets to OASIS-3 dataset. The results in this context are included in Fig.~\ref{inner_prod_reg_profile} in Appendix~\ref{inner_prod_all}.

\begin{figure}[t]
  \centering
  \includegraphics[scale=0.14]{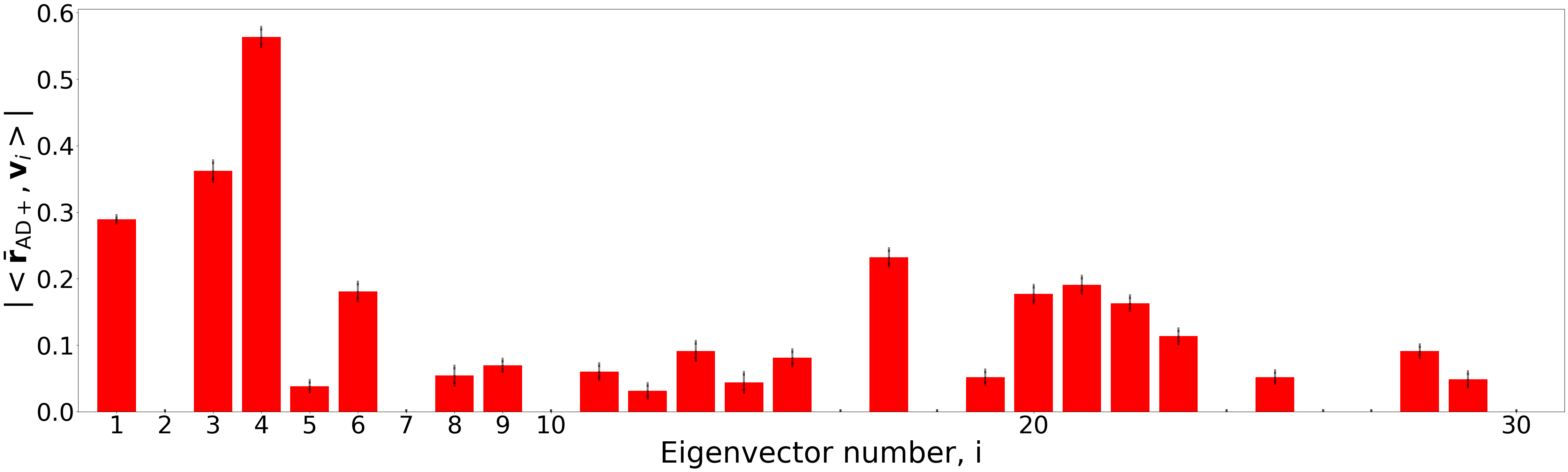}
   \caption{{\bf Mean inner product between the normalized vector of regional residuals (norm = 1) of VNN outputs (transferred from FTDC100 to OASIS-3) obtained from AD+ group and the eigenvectors of $\hat\bC_{148}^{\sf AD+}$ (covariance matrix of combined HC and AD+ group).} This figure illustrates a bar plot for $|\!<\!\bar\br_{\sf AD+},\bv_i\!>\!|$ for $i\in \{1,\dots, 30\}$, where $\bv_i$ is the $i$-th eigenvector of covariance matrix $\hat\bC_{148}^{\sf AD+}$ associated with its $i$-largest eigenvalue. The bars are evaluated from the mean of  $|\!<\!\bar\br_{\sf AD+},\bv_i\!>\!|$ obtained for individuals in the AD+ group (results for eigenvectors associated with coefficient of variation of $|\!<\!\bar\br_{\sf AD+},\bv_i\!>\!|$ larger than $30\%$ excluded). For every individual in AD+ group, the association of its regional residuals with eigenvectors of $\hat\bC_{148}^{\sf AD+}$ were evaluated over $100$ nominal VNN models (trained on the FTDC100 dataset). }
   \label{corr_eig_plots_ftdc_oasis}
\end{figure}

\begin{figure}[!htbp]
  \centering
  \includegraphics[scale=0.5]{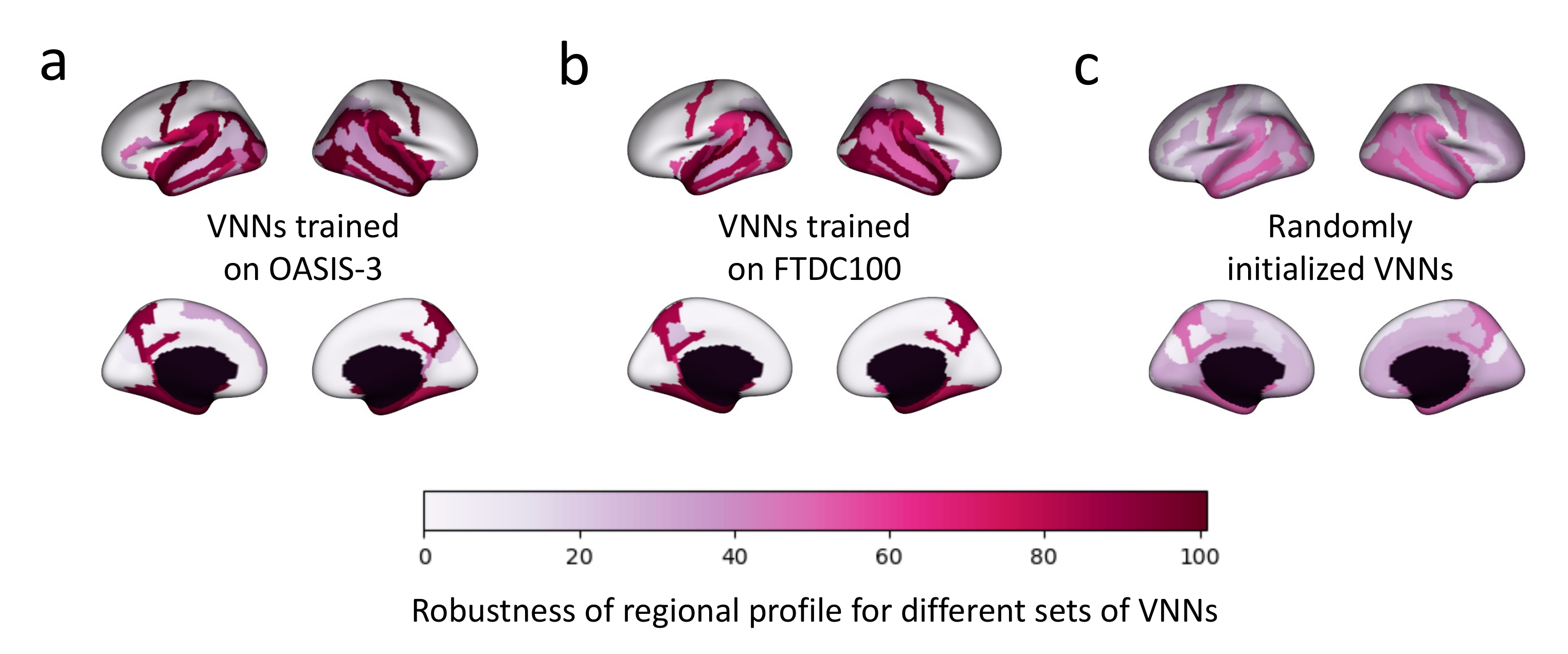}
   \caption{{\bf Regional profiles derived from the robustness of regional residuals being elevated for AD+ group with respect to HC group in OASIS-3 using different sets of VNNs.} Panel~{\bf a} projects the robustness of the regional residuals being elevated for AD+ with respect to HC group across the 100 VNN models on a brain template, where the VNNs were trained on the OASIS-3 dataset. The number of VNN models for which a brain region was deemed to have a significantly higher regional residual for AD+ group with respect to HC group quantifies the robustness of that brain region as a contributor to the observed elevated $\Delta$-Age in AD+ with respect to HC. Panel~{\bf b} displays the regional profiles obtained from the analysis of regional residuals for AD+ and HC groups, where the regional residuals were derived from $100$ VNNs trained to predict chronological age for HC group in FTDC100. Panel~{\bf c} displays the corresponding results derived from $100$ VNNs that were randomly initialized. }
   \label{oasis_ftdc_transfer}
\end{figure}
\subsubsection{Regional profiles for brain regions with elevated regional residuals in AD+ group are qualitatively preserved when VNNs are transferred
from FTDC100 to OASIS-3}
Next, we checked whether the consistency in correlations with principal components in Fig.~\ref{corr_eig_plots_oasis}a and Fig.~\ref{corr_eig_plots_ftdc_oasis} resulted in qualitative consistency in the regional profiles derived from the VNNs trained on FTDC100 dataset for regions associated with elevated regional residuals in AD+ group with respect to HC group in OASIS-3. Figure~\ref{oasis_ftdc_transfer}b illustrates the projection of the robustness of the observed significantly higher regional residuals for AD+ group with respect to HC group on a brain template, where the robustness was evaluated using $100$ nominal VNN models in a similar fashion as for the results in Fig.~\ref{roi_AD}b. For clarity, the regional profile for elevated $\Delta$-Age obtained from the VNNs that were trained on OASIS-3 dataset have been included in Fig.~\ref{oasis_ftdc_transfer}a. 

The regional profiles obtained using VNNs trained on FTDC100 in Fig.~\ref{oasis_ftdc_transfer}b were largely consistent with that in Fig.~\ref{oasis_ftdc_transfer}a. This observation implied that the VNNs trained on FTDC100 dataset possessed the ability to extract regional profiles characteristic of accelerated aging in OASIS-3 dataset. This observation could be attributed to the fact that the VNNs trained on FTDC100 were able to exploit similar principal components of $\hat\bC_{148}^{\sf AD+}$ as the VNNs that were trained on OASIS-3 dataset.

By training a VNN to predict chronological age in healthy individuals, we aim to gain the information about healthy aging, which is expected to be instrumental in detecting patterns of accelerated aging. Our results thus far have shown that the markers of accelerated aging identified by VNNs trained in this manner can be linked to the principal components of the anatomical covariance matrix irrespective of quantitative transferability. Thus, we conjecture that training the VNNs to predict chronological age helps fine tune their ability to manipulate the input data using principal components that were relevant to the elevated $\Delta$-Age effect in AD+ group. To further support this conjecture, we evaluated the regional profiles for brain regions that exhibited elevated regional residuals for AD+ group with respect to HC group in OASIS-3 using 100 VNNs that were randomly initialized (i.e., not trained whatsoever) and had the same architecture as the VNNs that were trained on OASIS-3. Figure~\ref{oasis_ftdc_transfer}c demonstrates the regional profiles derived from randomly initialized VNNs on OASIS-3 dataset. Figure~\ref{oasis_ftdc_transfer}c shows the robustness of the elevated regional residuals for AD+ group with respect to HC group was severely diminished with respect to the parallel results in Fig.~\ref{oasis_ftdc_transfer}a and Fig.~\ref{oasis_ftdc_transfer}b.


Our experiments here have demonstrated that the ability of VNNs to extract information about accelerated aging from cortical thickness data relies on their ability to exploit the eigenvectors or principal components of $\hat\bC_{148}^{\sf AD+}$ and not necessarily the performance on the task of predicting chronological age of healthy individuals. Thus, training the VNNs to predict chronological age in HC group helped fine tune their parameters to exploit the eigenvectors or principal components of $\hat\bC_{148}^{\sf AD+}$ relevant to $\Delta$-Age in AD+ group. The observations above facilitate the decoupling of the task of brain age prediction from the objective of achieving near perfect performance on the task of chronological age prediction in the HC group.

\noindent
{\bf Additional Experiments}: The associations between regional residuals and eigenvectors of $\hat\bC_{148}^{\sf AD+}$ and the regional profiles for brain regions with elevated regional residuals in AD+ group when VNNs were transferred from FTDC300 or FTDC500 are included in Fig.~\ref{inner_prod_reg_profile} in Appendix~\ref{inner_prod_all}. The results in Appendix~\ref{inner_prod_all} further demonstrate that the most significant correlations observed in Fig.~\ref{vnn_hc_eig_fig}a for VNNs trained on OASIS-3 were largely preserved when VNNs were transferred from FTDC300 or FTDC500 datasets to OASIS-3. Additional experiments in Appendix~\ref{el_stab} help establish the stability of the regional profiles in Fig.~\ref{roi_AD}b to perturbations in the covariance matrix~$\hat\bC_{148}^{\sf AD+}$ and demonstrate that the VNN models used here were not overfit on the composition of the population used to estimate the anatomical covariance matrix. In Appendix~\ref{vnn_adaptive}, we report that improving the performance of VNNs on the chronological age prediction task using an adaptive readout may penalize the regional interpretability offered by VNNs in the context of brain age (besides sacrificing the scale-free characteristic and the transferability guarantees for VNNs).

\section{Discussion}\label{discussion}
Graph convolution operator on covariance matrix, termed as a coVariance filter, forms the backbone of VNN architecture. The coefficients of the coVariance filter characterize its ability to manipulate the data according to the eigenspectrum of the covariance matrix to achieve a learning objective. Thus, statistical inference using VNNs draws similarities with PCA-driven statistical approaches. However, PCA conventionally operates within the feature space of a given dataset and hence, does not provide any notion of similarity between principal components of datasets with different number of features. In this paper, we have studied the key property of transferability of VNN models, which allows VNNs to be transferable between datasets with similar characteristics but different number of features. The notion of similarity between datasets consisting of different number of features is borrowed from the existing theory of graphons that studies limits of dense graphs~\cite{lovasz2012large}. Specifically, our theoretical results show that if there exists a sequence of covariance matrices that converges to a continuous limit object in the asymptote of infinite number of features, the VNNs can be transferred between any two covariance matrices of such a sequence for statistical inference. The underlying theoretical results rely on the convergence of the eigenspectrum of a continuous approximation of covariance matrices, which result in convergence of the coVariance filter outputs for covariance matrices belonging to a converging sequence, and subsequently, the convergence of VNN outputs. Our experiments pertain to dense anatomical covariance matrices and therefore, graphon model-based analyses were certainly appropriate to study transferability of VNNs. We also note that the parallel results on transferability of GNNs in~\cite{ruiz2021graphon} are restricted to graphs with no self-loops or unlabeled graphs, and these results can be recovered from Theorem~\ref{transferthm} when we have $\zeta = 1$ (i.e., all features in the dataset have equal variance). Furthermore, we remark that sparse covariance matrices are also of practical interest as they can help manage computational complexity~\cite{bien2011sparse}. Therefore, studying VNN transferability over sparse covariance matrices is a future direction of interest. 

We also observed that functions of VNN outputs were correlated with several eigenvectors of the underlying covariance matrix (Fig.~\ref{corr_eig_plots_oasis}), thus, validating our hypothesis based on Lemma~\ref{lm1} that VNNs manipulate the data according to the eigenspectrum of the covariance matrix. Hence, information processing with VNNs draws similarities with PCA based analysis. Unlike PCA, VNNs are guaranteed to be stable to perturbations in the underlying covariance matrix (Theorem~\ref{thm_stability}). Previously, we have empirically validated the stability of regression performance of VNNs to perturbations in the covariance matrix in~\cite{sihag2022covariance}. Figure~\ref{oasis_stability} demonstrates that the stability of VNNs to perturbations in the covariance matrix also extended to the observations relevant to gauging the interpretability of VNN outputs in the context of brain age.

From a data analytics perspective, we empirically showed that VNNs were able to extract information about chronological age from cortical thickness features of healthy controls and this information was transferable with minimal difference across cortical thickness datasets curated according to different scales of Schaefer's atlas. These observations were predicted by our theoretical results as the covariance matrices derived from cortical thickness features at different scales of Schaefer's atlas formed a converging sequence. Although we used Schaefer's atlas to demonstrate the transferability of VNNs within the scope of our theoretical results, our findings are not restricted to Schaefer's atlas and can readily be extended to any other brain atlases that accommodate multi-scale parcellations. We also note that this transferability of performance seemingly came at the cost of performance. Specifically, in Appendix~\ref{vnn_adaptive}, we report that the regression performance of VNNs could be improved significantly using adaptive readouts. Use of adaptive readouts sacrifices end-to-end transferability and scale-free aspect of VNNs if the number of learnable features in the readout function is dependent on the dimensionality of the data. Therefore, further study of VNNs with adaptive readouts may be necessary for performance-oriented applications.

Despite growing interest in similar multi-scale datasets in neuroscience~\cite{zeighami2021association,luppi2021combining, royer2022open}, the analyses strategies that optimally identify or leverage the redundancy in such datasets across different scales are currently lacking. In this context, VNNs also provide a novel analysis framework that can readily transfer across multiple scales. Analyses on datasets curated according to different brain atlases also revealed the limits of transferability of VNNs (Table~\ref{transfer_tbl1} and Table~\ref{transfer_tbl2}). For instance, the VNNs trained on cortical thickness features from FTDC100 datasets (Schaefer's atlas) did not retain regression performance on OASIS-3 dataset. These observations imply that VNNs may need to be augmented with mappings between brain atlases to optimally generalize their performance on chronological age prediction beyond the multi-scale cortical thickness datasets.

\textcolor{black}{In our experiments, the VNNs were trained on datasets extracted from a population of older adults (age $>$ 50 years). However, the changes in the brain are not uniform across the anatomical regions for different age groups - some regions grow more in development~\cite{lenroot2006brain, reardon2018normative} and atrophy more in aging than other regions~\cite{coffey1992quantitative, raz1997selective}. Thus, factors related to the impacts of development and neurodegeneration on the brain may render the VNNs used in this paper inadequate for capturing aging related information from cortical thickness data for a population with a broader or a different age profile than the datasets studied here. From a technical perspective, such biological factors may lead to violation of the assumptions necessary for the theoretical guarantees on the transferability of VNNs when assessing transference between neuroimaging datasets consisting of healthy controls from different age groups. Moreover, similar considerations may be necessary for longitudinal analyses where the VNNs trained on the data from a specific age group may not be equipped to perform accurate inference at all stages of the lifespan of an individual.}


While our focus in this paper has been strictly on cortical thickness datasets that were derived from structural MRI, there exist multiple neuroimaging modalities, such as diffusion MRI, functional MRI, and electroencephalogram (EEG) that provide distinct insights into brain activity and structure. These modalities typically have smaller signal-to-noise ratio (SNR) as compared to structural MRI~\cite{polzehl2016low, bennett2010reliable}. Hence, it is imperative to investigate the transferability of VNNs over different modalities of neuroimaging in order to gauge widespread applicability of their transferability across multi-scale datasets.




Brain age prediction task provides a unique statistical challenge as the brain age has no ground truth and the machine learning models used must be able to capture the changes driven by age-related neurodegeneration that lead to accelerated aging. In this context, our study of VNNs provides a foundational contribution to the methodology of brain age prediction. Our results showed that besides facilitating transferability, the non-adaptive readout function was instrumental to characterizing the regional interpretability of VNN models in the brain age prediction task. Specifically, the final regression output could be written as a mean of entities of an $m$-dimensional vector (as in~\eqref{interpret}) and therefore, allowed us to evaluate the contribution of each brain region to the final regression output. Thus, VNN models could also provide a regional profile if an elevated brain age was observed in morbidity. Furthermore, the regional profiles extracted by VNN models were correlated with certain principal components of the anatomical covariance matrix. Importantly, training the VNNs to predict chronological age helped fine tune their parameters to exploit the relevant principal components of the anatomical covariance matrix. Thus, the role of the age-bias correction step was restricted to projecting the VNN outputs onto a space where one could observe biological aging with respect to the chronological age from a layman's perspective. 

Interestingly, the VNNs transferred from FTDC datasets to OASIS-3 datasets were able to exploit the same brain age-related principal components of the anatomical covariance matrix of OASIS-3 dataset as the VNNs trained on OASIS-3 dataset, which resulted in consistent regional profiles for the brain regions with elevated regional residuals in the AD+ group for the two sets of VNNs. This observation further implied that the regional profiles identified by VNNs were robust to various factors characterizing the heterogeneity across FTDC and OASIS-3 datasets (such as distinct quality of neuroimaging data, contrasts across scanners in MRI acquisition, and different cohort compositions). Thus, our findings suggest that the convolution operation modeled by coVariance filters of VNNs provides a useful analytic tool to derive interpretable, spatially robust, and reproducible information from cortical thickness datasets that is relevant for brain age prediction.

The results in Appendix~\ref{vnn_adaptive} showed that improving the performance on the chronological age prediction task may not necessarily improve brain age prediction, either in terms of higher $\Delta$-Age or better characterization of brain regions perceived to be the contributors to elevated $\Delta$-Age. Brain age is a coarse metric that is expected to be elevated as compared to chronological age in various neurodegenerative conditions but may not have enough discriminability to discriminate between them. Therefore, by associating $\Delta$-Age with a regional profile, VNNs also provide a feasible tool to distinguish pathologies if the distributions of $\Delta$-Age for them are overlapping.

Currently, there is no clear benchmark for an acceptable performance of a machine learning model on the chronological age prediction task in order for it to predict brain age as most existing approaches in this domain lack interpretability. For instance, the study in ~\cite{butler2021pitfalls} recognizes brain age prediction using `loosely fitted' models on chronological age as a potential pitfall in the analyses while another reports better brain age prediction using a `moderately' fitted deep learning model~\cite{bashyam2020mri}. The study in~\cite{yin2023anatomically} associates voxel-wise interpretability to brain age but the link between the interpretability patterns and accuracy of chronological age prediction by their approach is not discussed. One motivation to perform brain age prediction with models that have high quality fit on the chronological age is presumably to mitigate the impact of age-bias correction step on the final result~\cite{butler2021pitfalls}. This criticism appears to be justified when brain age estimation approaches cannot isolate the abnormalities that lead them to predict elevated $\Delta$-Age. However, the simplicity of VNN models allows us to analyze the deviations in the intermediate steps of brain age estimation before age-bias correction is applied. In this context, we note that VNNs with unweighted mean as a readout function could not achieve perfect prediction of chronological age for healthy controls and yet, their brain age predictions in individuals with AD were associated with robust, transferable regional profiles. 
Thus, VNNs seem to be methodologically adept at finding neurodegeneration-driven factors that contribute to elevated $\Delta$-Age. The insights provided above in the context of VNNs and brain age are simply infeasible for brain age prediction approaches if they rely on complex and non-transparent deep learning models despite having millions of learnable parameters.


Based on the discussion thus far, we conclude that learning to predict chronological age in healthy controls is instrumental for VNNs to provide interepretability to elevated $\Delta$-Age. However, a near-prefect chronological age prediction for healthy controls by itself may not be a determinant of the quality of brain age prediction in neurodegeneration. From a broader perspective, brain age prediction even for healthy controls is a complex task due to various factors that can contribute to accelerated aging in the absence of an adverse health condition~\cite{ronan2016obesity, mareckova2020maternal,westlye2012effects}.    While we do not claim that the VNNs provide the `best' brain age prediction on any metric, our experiments have convincingly demonstrated that VNNs are able to extract the sufficient, robust information from cortical thickness datasets for anatomically interpretable and justified brain age prediction in neurodegeneration. It is possible that VNN based brain age predictions and associated regional interpretations could further be optimized or improved upon in some manner. For instance, incorporating the metrics of aging from DNA methylation aging~\cite{horvath2013dna} in the training of VNNs is a promising future direction that can help expand our understanding of aging.

Existing studies, including this paper, fall short at concretely defining the notion of optimal brain age prediction. However, we note that fine-tuning the fit of machine learning models on chronological age in order to gain the desired $\Delta$-Age in neurodegeneration or correlations of $\Delta$-Age with auxiliary measures potentially makes such models overfit on the $\Delta$-Age itself. Hence, a larger focus is needed on principled statistical approaches for brain age prediction that can capture the factors that lead to accelerated aging. Locally interpretable and theoretically grounded deep learning models such as VNNs can provide a feasible, promising future direction to build statistically and conceptually legitimate brain age prediction models in broader contexts.

\section{Data and Code Availability}\label{data_code}
MRI and clinical data for individuals in the FTDC datasets may be requested through \url{https://www.pennbindlab.com/data-sharing} and upon review by the University of Pennsylvania Neurodegenerative Data Sharing Committee, access will be granted upon reasonable request. OASIS-3 dataset is publicly available and hosted on~\url{central.xnat.org}. Code for demonstrating transferability of VNNs and brain age evaluation is available at~\url{https://github.com/pennbindlab/VNN_Brain_Age}. Requests for details regarding IDs of individuals in OASIS-3 and source data for all figures may be sent to \url{sihags@pennmedicine.upenn.edu}. 

\section{Acknowledgements}
The MRI data for FTDC datasets were provided by the Penn Frontotemporal Degeneration Center (NIH AG066597) and Penn Institute on Aging. Cortical thickness data
were made available by Penn Image Computing and Science Lab at University of Pennsylvania. OASIS-3 data were provided by Longitudinal Multimodal Neuroimaging: Principal Investigators: T. Benzinger, D. Marcus, J. Morris; NIH P30 AG066444, P50 AG00561, P30 NS09857781, P01 AG026276, P01 AG003991, R01 AG043434, UL1 TR000448, R01 EB009352.

\clearpage

\begin{appendices}

\section{An Abstract Overview of VNN-based Brain Age Prediction}\label{layman_vnn}
Figure~\ref{brain_age_tut} provides an abstract overview of the general procedure of evaluating brain age using machine learning (ML) models. From Fig.~\ref{brain_age_tut}, we note that if the ML model is a black box, it may be infeasible to capture the contributors to elevated age-gap in Step 3. Furthermore, in this context, it is also unclear whether age-bias correction step influences final $\Delta$-Age prediction through some statistical artifact~\cite{butler2021pitfalls}. Hence, it can be desirable to minimize the role of age-bias correction in $\Delta$-Age evaluation by selecting an ML model that achieves a near perfect fit on chronological age of healthy controls in Step 1. However, there is no guarantee that achieving an `perfect fit' on true age of healthy controls will enable the ML model to capture the impact of neurodegeneration in individuals with neurodegeneration.
\begin{figure}[!htbp]
  \centering
  \includegraphics[scale=0.45]{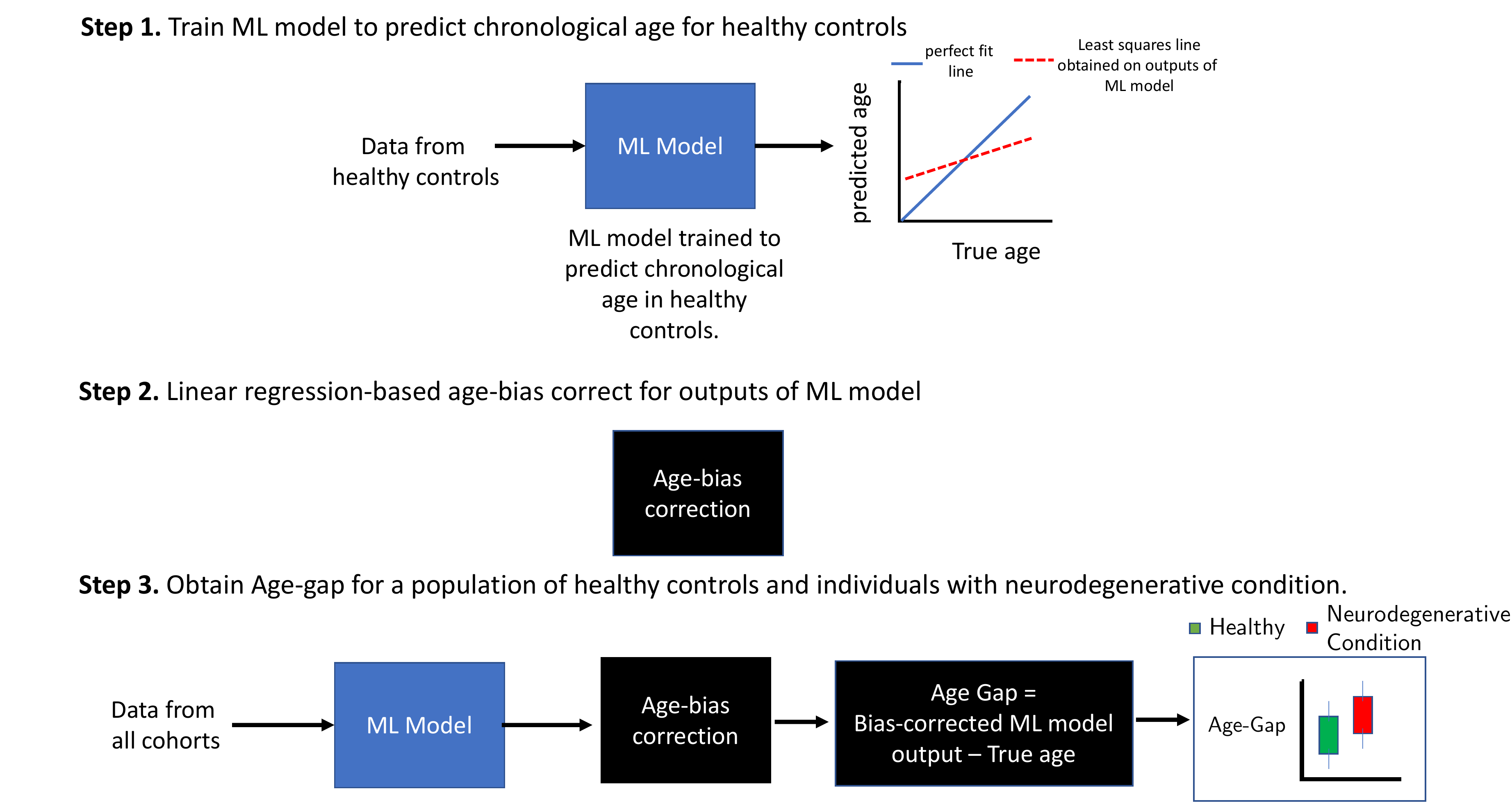}
   \caption{{\bf A general overview of brain age evaluation using machine learning algorithms in the existing literature.} {\bf Step 1} consists of training a machine learning (ML) model to predict chronological age (true age) for healthy controls. If the correlation between predicted age and true age is smaller than $1$, an age-bias exists in ML model outputs as the age for older individuals tends to be under-estimated and that for younger individuals tends to be over-estimated. To correct for this bias, a linear regression based model is applied on the ML model outputs in {\bf Step 2}. Under the hypothesis that ML model can capture accelerated aging in age-related neurodegeneration, it is expected that $\Delta$-Age for individuals with neurodegeneration will be significantly higher than those of healthy controls ({\bf Step 3}).}
   \label{brain_age_tut}
\end{figure}

VNNs allow us to analyze the contribution of each feature (brain region) to the final output. Hence, by analyzing the elevations in contributions of different brain regions via studying group differences in regional residuals, we are able to characterize the brain regions that contribute to accelerated aging (or larger $\Delta$-Age) in individuals with age-related neurodegeneration (Fig.~\ref{brain_age_vnn_tut}). Thus, we can verify that VNNs captured neurodegeneration-driven effects that eventually led to elevated $\Delta$-Age for an individual. Our experiments show that VNNs do not obtain a perfect fit on chronological age of healthy individuals. Hence, age-bias correction is important to appropriately project the VNN model outputs via a linear model into an appropriate space such that a clinician can observe an elevated $\Delta$-Age effect in individuals with neurodegenerative condition (AD in this paper). Based on these observations, we remark that VNNs provide an interpretable framework for brain age prediction. 

\begin{figure}[!htbp]
  \centering
  \includegraphics[scale=0.5]{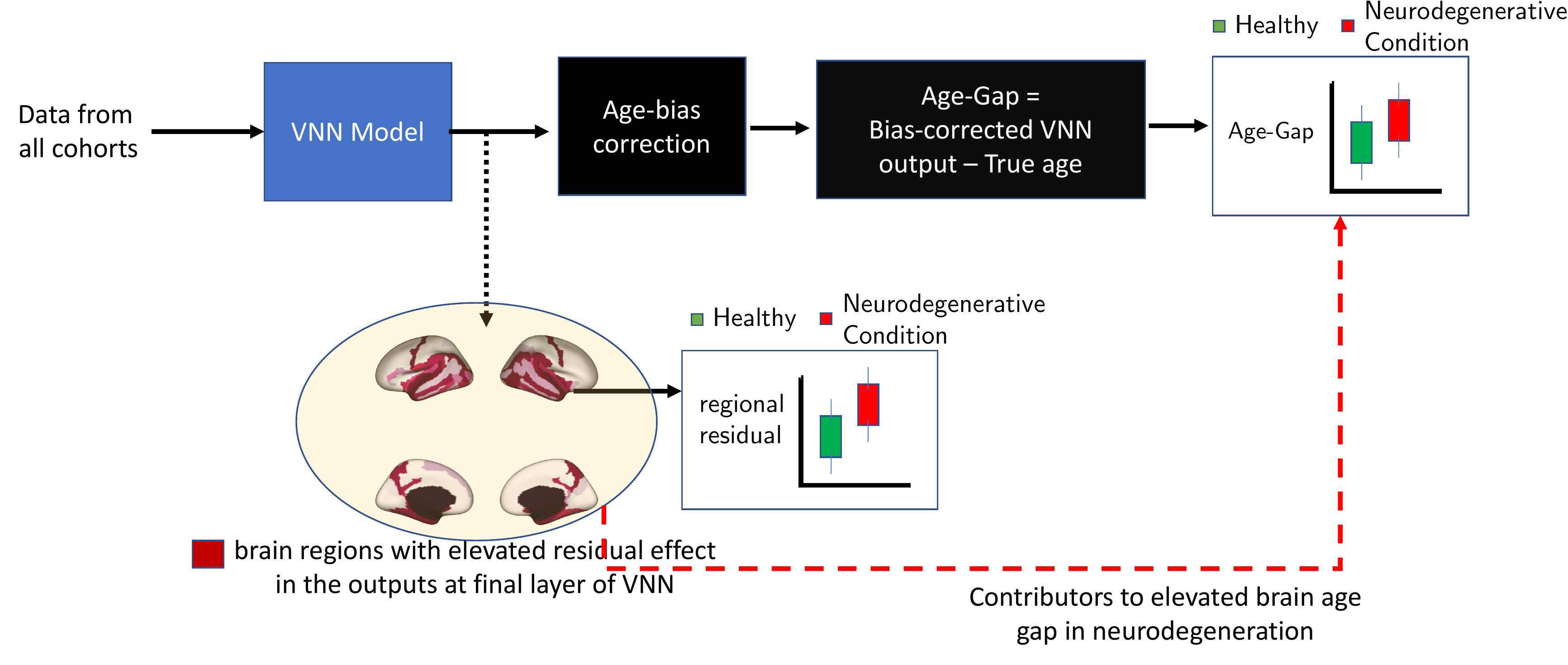}
   \caption{{\bf Interpretability offered by VNNs in brain age prediction.} By analyzing the final layer outputs of VNNs, we can isolate brain regions that have larger regional residuals for individuals with AD with respect to healthy controls. Furthermore, the elevated regional residuals in these brain regions eventually contribute to elevated $\Delta$-Age after age-bias correction.}
   \label{brain_age_vnn_tut}
\end{figure}

\clearpage
\section{Proof of Theorem~\ref{thm_stability}}\label{pf_thm1}
We reprise the proof from~\cite{sihag2022covariance}. To start with, we note that the coVariance filters with respect to $\hat\bC$ and $\bC$ are given by
\begin{align}
    \bH(\hat\bC) = \sum\limits_{k=0}^{m}h_k\hat\bC^k \quad \text{and}\quad  \bH(\bC) = \sum\limits_{k=0}^{m}h_k\bC^k\;.
\end{align}
We start by characterizing the perturbation of sample covariance matrix $\hat\bC$
 with respect to $\bC$ in Lemma~\ref{lm1}. To this end, we define
 \begin{align}
     \bE \dff \hat\bC - \bC\;,
 \end{align}
and $\bI_m$ as an $m\times m$ identity matrix. Also, the eigenvalue decomposition of sample covariance matrix $\hat\bC$ is given by
\begin{align}\label{eigensample}
    \hat \bC = \hat\bV \hat\Lambda \hat\bV^{\sf T}\;,
\end{align}
where $\hat\bV = [\hat\bv_1,\dots,\hat\bv_m]$ is the matrix constituted by orthonormal eigenvectors of $\hat \bC$ and $\hat\Lambda = {\sf diag}(\hat\lambda_1,\dots,\hat\lambda_m)$ is the diagonal matrix of eigenvalues of $\hat \bC$, such that, $\hat\lambda_1\geq \hat\lambda_2\dots\geq \hat\lambda_m$. Clearly, the eigenvalues and eigenvectors of $\hat\bC$ are estimates of the eigenvalues and eigenvectors of the true covariance matrix $\bC$. 
\begin{lemma}\label{lm}
Consider an ensemble covariance matrix $\bC$ with the eigendecomposition in~\eqref{sample_eig} and a sample covariance matrix $\hat\bC$ with the eigendecomposition in~\eqref{eigensample}. For any eigenvalue $\lambda_i > 0$ of $\bC$, the perturbation $\bE$ satisfies
\begin{align}
    \bE \bv_i = \beta_i\delta\bv_i + \delta\lambda_i \bv_i + ( \delta\lambda_i\bI_m -\bE) \delta\bv_i
\end{align}
where 
\begin{align}
    \beta_i \dff  (\lambda_i \bI_{m} - \bC),\quad \delta\bv_i \dff \hat\bv_i-\bv_i,\quad \delta\lambda_i \dff \hat\lambda_i-\lambda_i\;.
\end{align}

\end{lemma}
\begin{proof}

Note that from the definition of eigenvectors and eigenvalues, we have
\begin{align}\label{egic}
   \hat\bC \hat\bv_i = \hat\lambda_i\hat\bv_i\;.
\end{align}
We can rewrite~\eqref{egic} in terms of perturbations with respect to the ensemble covariance matrix $\bC$ and the outputs of its eigendecomposition as follows:
\begin{align}\label{prt}
   (\hat\bC -\bC)  (\bv_i + \delta\bv_i) + \bC (\bv_i + \delta\bv_i) = (\lambda_i + \delta\lambda_i)(\bv_i + \delta\bv_i)\;,
\end{align}
where we have used $\hat\lambda_i = \lambda_i + \delta\lambda_i$ and $\hat\bv_i = \bv_i + \delta\bv_i$. Using the fact that $\bC \bv_i = \lambda_i\bv_i$ and rearranging the terms in~\eqref{prt}, we have
\begin{align}\label{prt2}
    (\hat\bC -\bC) \bv_i = (\lambda_i \bI_{m} - \bC)\delta\bv_i + \delta\lambda_i  (\bv_i + \delta\bv_i) - (\hat\bC -\bC) \delta\bv_i\;.
\end{align}
By setting $\bE = \hat\bC -\bC$ and $\beta_i = \lambda_i \bI_{m} - \bC$, we can rewrite~\eqref{prt2} as
\begin{align}
    \bE\bv_i = \beta_i\delta\bv_i + \delta\lambda_i \bv_i + (\delta\lambda_i\bI_m - \bE) \delta\bv_i\;.
\end{align}
\end{proof}
\noindent
We next establish the first order approximation for $\hat\bC^k $ in terms of $\bC$ and $\bE$. The first order approximation of $\hat\bC^k$ is given by
\begin{align}\label{pf2}
    (\bC + \bE)^k = \bC^k + \sum\limits_{r=0}^{m} \bC^r \bE \bC^{k-r-1} + \tilde\bE\;,
\end{align}
where $\tilde \bE$ satisfies $\|\tilde \bE\| \leq \sum\limits_{r=2}^k {k\choose r}\|\bE\|^r\|\bC\|^{k-r}$. Using~\eqref{pf2}, we have
\begin{align}
    \bH(\hat\bC) - \bH(\bC) &= \sum\limits_{k=0}^{m} h_k [(\bC + \bE)^k - \bC^k]\;,\\
    &= \sum\limits_{k=0}^{m} h_k \sum\limits_{r=0}^{k-1} \bC^r \bE \bC^{k-r-1} + \tilde \bE\;,\label{pf3}
\end{align}
where $\tilde\bE$ satisfies $\|\tilde\bE\|^2 = {\cal O}(\|\bE\|^2)$~\cite{gama2020stability}. The focus of our subsequent analysis will be the first term in~\eqref{pf3}. For a random data sample $\bx = [x_1,\dots,x_m]^{\sf T}$, such that, $\|\bx\|_2 \leq 1$ and $\bx \in \mR^{m\times 1}$,  its Fourier transform with respect to $\bC$ is given by $\tilde\bx = \bV^{\sf T} \bx$, where $\tilde\bx = [\tilde x_1,\dots,\tilde x_m]^{\sf T}$~\cite{sihag2022covariance}. The relationship $\tilde \bx $ and $\bx$ can be expressed as
\begin{align}\label{pf5}
   \bx = \sum\limits_{i=1}^m \tilde x_i \bv_i\;.
\end{align}
Multiplying both sides in~\eqref{pf3} by $\bx$ and by leveraging~\eqref{pf5}, we get
\begin{align}
    [\bH(\hat\bC) - \bH(\bC) ]\bx &= \sum\limits_{k=0}^{m} h_k \sum\limits_{r=0}^{k-1} \bC^r \bE \bC^{k-r-1} \bx + \tilde \bE\bx\;,\\
    &= \sum\limits_{i=1}^m \tilde x_i \sum\limits_{k=0}^{m} h_k \sum\limits_{r=0}^{k=1} \bC^r \bE \bC^{k-r-1}  \bv_i + \tilde \bE\bx\;,\label{t1}\\
    &=  \sum\limits_{i=1}^m \tilde x_i \sum\limits_{k=0}^{m} h_k \sum\limits_{r=0}^{k-1} \bC^r\lambda_i^{k-r-1} \bE\bv_i + \tilde \bE\bx\;,\label{t2}
\end{align}
where we have used $\bC \bv_i = \lambda_i\bv_i$ in the transition from~\eqref{t1} to~\eqref{t2}. 
We focus only on the first term in~\eqref{t2} and leverage the result from Lemma~\ref{lm1} that expands $\bE \bv_i$ to get
\begin{align}\label{pf6}
    \sum\limits_{i=1}^m \tilde x_i \sum\limits_{k=0}^{m} h_k \sum\limits_{r=0}^{k-1} \bC^r\lambda_i^{k-r-1} \bE\bv_i &= \underbrace{ \sum\limits_{i=1}^m \tilde x_i \sum\limits_{k=0}^{m} h_k \sum\limits_{r=0}^{k-1} \bC^r\lambda_i^{k-r-1} \beta_i\delta\bv_i}_\text{Term 1} \nonumber\\
    & \enskip + \underbrace{ \sum\limits_{i=1}^m \tilde x_i \sum\limits_{k=0}^{m} h_k \sum\limits_{r=0}^{k-1} \bC^r\lambda_i^{k-r-1} \delta\lambda_i\bv_i }_\text{Term 2}\nonumber\\
    &\enskip +  \underbrace{ \sum\limits_{i=1}^m \tilde x_i \sum\limits_{k=0}^{m} h_k \sum\limits_{r=0}^{k-1} \bC^r\lambda_i^{k-r-1} ( \delta\lambda_i\bI_m -\bE)\delta\bv_i }_\text{Term 3}\;.
\end{align}
Next, we analyze term 1, term 2, and term 3 in~\eqref{pf6} separately. 

\noindent
\textbf{Analysis of Term 1 in~\eqref{pf6}}. In the analysis of term 1, we start by noting that
\begin{align}
    \beta_i &= \lambda_i\bI_m - \bC \;,\\
    &= \sum\limits_{j = 1}^m (\lambda_i - \lambda_j)\bv_j\bv_j^{\sf T}\;,\label{t11}\\
    &=  \bV (\lambda_i\bI_m - \Lambda)\bV^{\sf T}\;.\label{t12}
\end{align}
Using~\eqref{t12} and $\delta\bv_i = \bu_i - \bv_i$ in term 1 in~\eqref{pf6}, we have
\begin{align}
     \sum\limits_{i=1}^m \tilde x_i \sum\limits_{k=0}^{m} h_k \sum\limits_{r=0}^{k-1} \bC^r\lambda_i^{k-r-1} \bV (\lambda_i\bI_m - \Lambda)\bV^{\sf T} (\hat\bv_i - \bv_i)\;.\label{t13}
\end{align}
Using $\bC^r = \bV \Lambda^r \bV^{\sf T}$ in~\eqref{t13}, term 1 in~\eqref{pf6} is equivalent to
\begin{align}
    &\sum\limits_{i=1}^m \tilde x_i \sum\limits_{k=0}^{m} h_k \sum\limits_{r=0}^{k-1} \lambda_i^{k-r-1} \bV \Lambda^r (\lambda_i\bI_m - \Lambda)\bV^{\sf T} (\hat\bv_i - \bv_i)\;,\\
    &= \sum\limits_{i=1}^m \tilde x_i \bV \bL_i \bV^{\sf T} (\hat\bv_i - \bv_i)\;,\label{pf7}
\end{align}
where $\bL_i$ is a diagonal matrix whose $j$-th element is given by
\begin{align}
    [\bL_i]_j &= \sum\limits_{k=0}^{m} h_k \sum\limits_{r=0}^{k-1} (\lambda_i - \lambda_j) \lambda_i^{k-r-1} \lambda_j^r\;,\\
    &= \sum\limits_{k=0}^{m} h_k (\lambda_i - \lambda_j)\frac{\lambda_i^k - \lambda_j^k}{\lambda_i - \lambda_j}\;,\\
    & = \sum\limits_{k=0}^{m}h_k \lambda_i^k - \sum\limits_{k=0}^{m}h_k \lambda_j^k\;,\\
    &= h(\lambda_i) - h(\lambda_j)\;,\label{t15}
\end{align}
where $ h(\lambda_i)$ is the frequency response of the coVariance filter and is defined in~\eqref{vvf}. Therefore, we have $\bL_i = {\sf diag}([h(\lambda_i) - h(\lambda_j)]_j)$. Next, in~\eqref{pf7}, we note that 
\begin{align}\label{t14}
    \bV^{\sf T}(\hat\bv_i - \bv_i) = [\bv_1^{\sf T}(\hat\bv_i - \bv_i) , \cdots, \bv_m^{\sf T}(\hat\bv_i - \bv_i)]^{\sf T}\;.
\end{align}
Using~\eqref{t14} and~\eqref{t15} in~\eqref{pf7} and $\bv_j^{\sf T}\bv_i = 0, \forall j\neq i$, we deduce that the term 1 in~\eqref{pf6} is equivalent to
\begin{align}\label{t1f}
    \sum\limits_{i=1}^m \tilde x_i \bV \bL_i \bV^{\sf T} (\hat\bv_i - \bv_i) = \sum\limits_{i=1}^m \tilde x_i \bV \bJ_i\;,
\end{align}
where the $j$-th element of $\bJ_i$ is given by
\begin{align}
    [\bJ_i]_j = \begin{cases}
    &0\;, \quad \text{if }\enskip j = i\;,\\
    &(h(\lambda_i) - h(\lambda_j) )\bv_j^{\sf T} \hat\bv_i\;, \enskip \text{otherwise}
    \end{cases}\;.
\end{align}
For the stability analysis, we are interested in the norm of term 1. Therefore, by noting the equivalence between the term 1 in~\eqref{pf6} and~\eqref{t1f}, after taking the norm, we have
\begin{align}
    \left\lVert\sum\limits_{i=1}^m \tilde x_i \sum\limits_{k=0}^{m} h_k \sum\limits_{r=0}^{k-1} \bC^r\lambda_i^{k-r-1} \beta_i\delta\bv_i\right\rVert_2 &=\left\lVert\sum\limits_{i=1}^m \tilde x_i \bV \bJ_i\right\rVert_2 \;,\\
    & \leq \sqrt{m} \sum\limits_{i=1}^m |\tilde x_i| \max_{j, i\neq j} |h(\lambda_i)- h(\lambda_j)| |\bv_j^{\sf T} \hat\bv_i|\;.
\end{align}
Note that $\bv_j^{\sf T} \hat\bv_i$ is the inner product between the eigenvector $\bv_j$ of the ensemble covariance matrix $\bC$ and the eigenvector $\hat\bv_i$ of the sample covariance matrix $\hat\bC$. The bounds on $\bv_j^{\sf T} \hat\bv_i$ in terms of the number of data samples $n$ have been studied in the existing literature. Here, we leverage the result from~\cite[Theorem 4.1]{loukas2017close} to conclude that if ${\sf sgn}(\lambda_j-\lambda_i)2 \hat\lambda_j > {\sf sgn}(\lambda_j-\lambda_i)(\lambda_j-\lambda_i)$ for $\lambda_i\neq \lambda_j$, the condition
\begin{align}\label{term1r}
     \left\lVert\sum\limits_{i=1}^m \tilde x_i \sum\limits_{k=0}^{m} h_k \sum\limits_{r=0}^{k-1} \bC^r\lambda_i^{k-r-1} \beta_i\delta\bv_i\right\rVert_2 \leq \sqrt{m}\sum\limits_{i=1}^m |\tilde x_i|\max_{j, i\neq j} |h(\lambda_i)- h(\lambda_j)| \frac{2k_i}{n^{1/2  -\varepsilon}|\lambda_i - \lambda_j|}\;,
\end{align}
is true with probability at least $\left(1 - \frac{1}{n^{2\varepsilon}}\right)$ for some $\varepsilon \in (0,1/2]$, where $k_i = \Big(\mE[\|\bX\bX^{\sf T} \bv_i\|_2^2] - \lambda_i^2\Big)^{\frac{1}{2}}$. Furthermore, we note that the condition~$ {\sf sgn}(\lambda_j-\lambda_i)2 \hat\lambda_j > {\sf sgn}(\lambda_j-\lambda_i)(\lambda_j-\lambda_i)$ is satisfied with probability at least $1 - \frac{2k_i^2}{|\lambda_i - \lambda_j|}$~\cite[Corollary 4.2]{loukas2017close}, which via a union bound and first order approximation from Taylor series implies that~\eqref{term1r} is true with probability at least $1 - \frac{1}{n^{2\varepsilon}} - \frac{2\kappa m}{n}$ for $\kappa$ defined as
\begin{align}\label{addefs}
     \kappa \dff \max_{i,j: \lambda_i \neq \lambda_j} \frac{k_i^2}{|\lambda_i - \lambda_j|} \quad\text{where} \quad k_{\sf min} \dff \min_{i\in \{1,\dots,m\}, \lambda_i>0} k_i \;.
 \end{align}
Therefore, for a coVariance filter with the property
\begin{align}\label{fltrprop}
    \frac{|h(\lambda_i)-h(\lambda_j)|}{|\lambda_i-\lambda_j|}\leq \frac{M}{k_i}\;,
\end{align}
for some real constant $M > 0$, the condition in~\eqref{term1r} is equivalent to
\begin{align}
     \left\lVert\sum\limits_{i=1}^m \tilde x_i \sum\limits_{k=0}^{m} h_k \sum\limits_{r=0}^{k-1} \bC^r\lambda_i^{k-r-1} \beta_i\delta\bv_i\right\rVert_2 \leq \frac{2\sqrt{m}  M}{n^{\frac{1}{2}-\varepsilon}}\sum\limits_{i=1}^m |\tilde x_i|\;,
\end{align}
which holds with probability at least  $1 - \frac{1}{n^{2\varepsilon}} - \frac{2\kappa m}{n}$. Furthermore, note that $\sum\limits_{i=1}^m |\tilde x_i| \leq \sqrt{m}\|\bx\|_2$. When the random sample $\bx$ satisfies $\|\bx\|_2 \leq 1$, we have
\begin{align}\label{final_pf1}
    \mP\left( \left\lVert\sum\limits_{i=1}^m \tilde x_i \sum\limits_{k=0}^{m} h_k \sum\limits_{r=0}^{k-1} \bC^r\lambda_i^{k-r-1} \beta_i\delta\bv_i\right\rVert_2 \leq \frac{2}{n^{\frac{1}{2}-\varepsilon}}m M \right)  \geq 1-\frac{1}{n^{2\varepsilon}} - \frac{2\kappa m}{n}\;,
\end{align}
for any $\varepsilon\in (0,1/2]$.

\noindent
\textbf{Analysis of Term 2 in~\eqref{pf6}}.
Using $\bC \bv_i = \lambda_i\bv_i$, we note that term 2 in~\eqref{pf6} is equivalent to
\begin{align}
    \sum\limits_{i=1}^m \tilde x_i \sum\limits_{k=0}^{m} h_k \sum\limits_{r=0}^{k-1} \bC^r\lambda_i^{k-r-1} \delta\lambda_i\bv_i &= \sum\limits_{i=1}^m \tilde x_i \sum\limits_{k=0}^{m} h_k \sum\limits_{r=0}^{k-1} \lambda_i^{k-1} \delta\lambda_i\bv_i \;,\\
    &=  \sum\limits_{i=1}^m \tilde x_i \sum\limits_{k=0}^{m}k h_k \lambda_i^{k-1} \delta\lambda_i\bv_i \;,\\
    &=  \sum\limits_{i=1}^m \tilde x_i h'(\lambda_i) \delta\lambda_i\bv_i \;.\label{simpl}
\end{align}
Next, using Weyl's theorem~\cite[Theorem 8.1.6]{golub2013matrix}, we note that  $\|\bE\| \leq \alpha$ implies that $|\delta\lambda_i| \leq \alpha$ for any $\alpha>0$. For a random instance $\bx$ of a random vector $\bX$ whose  probability distribution is supported within a ball of radius $1$ w.l.o.g, such that, $\|\bx\|_2 \leq 1$, we have
\begin{align}\label{alpha}
    \mP\left(\bE \leq B\left(\frac{\|\bC\|\sqrt{\log m + u}}{\sqrt{n}} + \frac{(1 + \|\bC\|)(\log m+u)}{n}\right)\right) \geq 1 - 2^{-u}\;,
\end{align}
for some constant $B > 0$ and $u > 0$. The result in~\eqref{alpha} follows directly from~\cite[Theorem 5.6.1]{vershynin2018high}. Therefore, using~\eqref{simpl}, we have
\begin{align}
    \left\lVert\sum\limits_{i=1}^m \tilde x_i \sum\limits_{k=0}^{m} h_k \sum\limits_{r=0}^{k-1} \bC^r\lambda_i^{k-r-1} \delta\lambda_i\bv_i\right\rVert_2 &\leq \sum\limits_{i=1}^m |\tilde x_i| |h'(\lambda_i)| |\delta\lambda_i| \|\bv_i\|_2\;.
\end{align}
Using~\eqref{alpha}, $|h'(\lambda_i)|\leq M/k_{\sf min}$ (where $k_{\sf min} = \min_{i\in \{1,\dots,m\}, \lambda_i>0}k_i$) from~\eqref{fltrprop}, and $\|\bv_i\|_2 = 1$, we have
\begin{align}\label{final_pf2}
    &\mP\left( \left\lVert\sum\limits_{i=1}^m \tilde x_i \sum\limits_{k=0}^{m} h_k \sum\limits_{r=0}^{k-1} \bC^r\lambda_i^{k-r-1} \delta\lambda_i\bv_i\right\rVert_2\right.\nonumber\\
    &\quad \left.\leq \frac{A}{k_{\sf min}}\sqrt{m}M \left(\frac{\|\bC\|\sqrt{\log m + u}}{\sqrt{n}} + \frac{(1 + \|\bC\|)(\log m+u)}{n}\right)\right) \geq 1 - 2^{-u}\;,
\end{align}
for some constant $A>0$ and $u>0$.

\noindent
\textbf{Analysis of Term 3 in~\eqref{pf6}}. We remark that the term 3 in~\eqref{pf6} consists of second order error terms that diminish faster with the number of samples $n$ as compared to term 1 and term 2. This can be concluded from the following observations. Firstly, from~\eqref{alpha} and Weyl's theorem, we note that $\| \delta\lambda_i\bI_m -\bE\| \leq 2\|\bE\|$ and $\|\bE\| \simeq {\cal O}(1/\sqrt{n})$ with high probability. Secondly, the upper bound on the term $\delta \bv_i$ also has a similar scaling behavior as $\bE$~\cite{yu2015useful}. 

Therefore, the second order error term $( \delta \lambda_i\bI_m -\bE)\delta\bv_i$ diminishes at a rate faster than ${\cal O}(1/\sqrt{n})$, which is faster as compared to terms 1 and 2, that individually scale as ${\cal O}(1/n^{1/2 - \varepsilon})$  for $\varepsilon \in (0,1/2]$ and ${\cal O}(1/\sqrt{n})$, respectively. Finally, by noting that the condition on $\|[\bH(\hat\bC)-\bH(\bC)]\bx\|_2$ reduces to the condition on operator norm  $\|[\bH(\hat\bC)-\bH(\bC)]\|$ over the search space $\max_{\bx \in \mR^{m\times1},\|\bx\|_2\leq 1}\{\|[\bH(\hat\bC)-\bH(\bC)]\bx\|_2\}$ and that the terms scaling at $1/\sqrt{n}$ or slower in~\eqref{final_pf1} and~\eqref{final_pf2} dominate the scaling behavior of the upper bound on $\|\bH(\hat\bC)-\bH(\bC)\|$, we arrive at the following theorem.
\begin{theorem}[Stability of coVariance Filter]\label{filterstab}
Consider a random vector $\bX \in \mR^{m\times 1}$ , such that, its corresponding covariance matrix is given by $\bC = \mE[(\bX - \mE[\bX]) (\bX - \mE[\bX])^{\sf T}]$. For a sample covariance matrix $\hat\bC$ formed using $n$ i.i.d instances of $\bX$ and a random instance $\bx$ of $\bX$, such that, $\|\bx\|_2 \leq 1$ and under assumption~\eqref{fltrprop}, the following holds with probability at least $(1 - {n^{-2\varepsilon}} - 2\kappa m/n)(1-1/n)$ for any $\varepsilon \in (0,1/2]$:
\begin{align}\label{filterstab_rslt}
    \left\lVert \bH(\hat\bC) - \bH(\bC)\right\rVert = \frac{M}{n^{\frac{1}{2} - \varepsilon}}\cdot{\cal O}\left(m + \frac{\sqrt{m}\|\bC\|\sqrt{\log mn}}{k_{\sf min}n^{\varepsilon}}\right)\;.
\end{align}
\end{theorem}
\noindent 
Thus, the right-hand side of~\eqref{filterstab_rslt} provides the structure of~$\alpha_n$ in Theorem~\ref{thm_stability}. From Theorem~\ref{filterstab}, we note that $ \|\bH(\hat\bC) - \bH(\bC)\|$ decays with the number of samples $n$ at least at the rate of $1/n^{\frac{1}{2} - \varepsilon}$. Thus, we conclude that the stability of the coVariance filter improves as the number of samples $n$ increases. This observation is along the expected lines as the estimate $\hat \bC$ becomes closer to the ensemble covariance matrix $\bC$ by the virtue of the law of large numbers. We refer the reader to~\cite{sihag2022covariance} for further discussion on the implication of the assumption in~\eqref{fltrprop} on filter design. The rest of the proof follows directly from~\cite[Theorem 3]{sihag2022covariance}. 

\section{Graphon Information Processing}\label{gip}
The theory of graphons has previously been leveraged to study the transferability of GNNs between graphs in the same graphon family~\cite{ruiz2020graphon}.  The proof of Theorem~\ref{transferthm} relies on establishing the transferability of VNNs between datasets in the setting where their corresponding covariance matrices belong to a converging sequence characterized by a graphon. Our main objective in this section is to show that data processing over coVariance filter can equivalently be represented in the continuous domain using its graphon approximation. Establishing this property will ultimately allow us to compare VNNs instantiated on covariance matrices derived from datasets with different numbers of features. We begin with some preliminaries for graphons. 
\subsection{Preliminaries.}
The definition of a graphon is re-stated below.
\begin{definition}[Graphon]\label{grphon2}
A graphon is a bounded, symmetric, measurable function ${\bW: [0,1]^2 \mapsto [-1,1]}$.   
\end{definition} 
\noindent
Using the theory of convergence of graphons and interpreting the covariance matrix as a weighted graph representation of data, a graphon $\bW$ exists as a limiting object for the sequence of graphon approximations $ \{\bW_{\bC_m}\}$ if the sequence of covariance matrices $\{\bC_m\}$ converges in the \emph{cut distance}~\cite{borgs2008convergent}. In Remark~\ref{grphnlimit}, we formalize the sufficient condition for existence of a limit object for a given sequence of covariance matrices. The statement in Remark~\ref{grphnlimit} is an extension of~\cite[Corollary 3.9]{borgs2008convergent} to our setting where covariance matrices are viewed as weighted graphs.
\begin{remark}[Graphon as limit object~\cite{borgs2008convergent}]\label{grphnlimit}
    A sequence of covariance matrices $\{\bC_m\}$ is deemed convergent if they form a Cauchy sequence with respect to the cut distance $\delta_{\square}$, where the cut distance $\delta_{\square} (\bC_{m_1}, \bC_{m_2})$ between covariance matrices $\bC_{m_1}$ and $\bC_{m_2}$  is defined in~\eqref{cutdist_eq} in Appendix~\ref{cutdist}.
    Furthermore, for any convergent sequence of covariance matrices $\{\bC_{m}\}$, the corresponding sequence of graphon approximations $\{\bW_{\bC_m}\}$ converges to a graphon. 
\end{remark}
\noindent A distinct feature of the cut distance is that it allows the comparison of covariance matrices of different sizes. Hence, all covariance matrices whose graphon approximations converge to a graphon can be considered to be a part of that graphon family. Moreover, graphon $\bW$ can be interpreted as the schema for which the covariance matrix $\bC_m$ represents the covariance realization at resolution $m$.


\noindent
\subsection{Information Processing with Graphons}
We next show that a coVariance filter $\bH(\bC_m)$ can be equivalently represented in the continuous domain using convolution operations over graphon representations $\bW_{\bC_m}$. Given a coVariance filter output $\bz = \bH(\bC_m)\bx$, the continuous representation of $\bx$ is $y_{\bx}$ and that of $\bC_m$ is $\bW_{\bC_m}$. The operation $\bC\bx$ is fundamental to the convolution operation in $\bH(\bC_m)\bx$ and therefore, we first provide its continuous equivalent. For $\bs = \bC\bx$, the $i$-th element of $\bs$ is given by
\begin{align}\label{sw}
    [\bs]_i = \sum\limits_{j=0}^m [\bC_m]_{ij} [\bx]_j\;.
\end{align}
Thus, $[\bs]_i$ is a linear combination of elements in $\bx$ according to the $i$-th row of $\bC_m$. In the continuous space, we 
can equivalently write~\eqref{sw} as
\begin{align}\label{sw2}
    y_{\bs}(u) = \int_{0}^1 \bW_{\bC_m}(u,v) y_{\bx}(v) {\sf d} v\;,
\end{align}
where $y_{\bx}$ is the continuous representation of $\bx$ obtained according to the intervals defined in~\eqref{interval1}. Note that $y_{\bs}$ is a continuous representation of $\bs$, i.e., they satisfy $y_{\bs}(u) = [\bs]$ for $u\in \cU_i$. Hence, $y_{\bs}$ and $\bs$ can be recovered from each other. This observation can be extrapolated to define the continuous equivalent of a coVariance filter. This is feasible because we can write the entity $\bC_m^k\bx$ in $\bH(\bC)$ in a recursive form. Specifically, if we have $\bs_k = \bC_m^k\bx$, then we can rewrite $\bs_k$ as
\begin{align}
    \bs_k = \bC_m \bs_{k-1}\;,
\end{align}
where $\bs_0 = \bx$. Thus, using the same reasoning that established the equivalence between~\eqref{sw} and~\eqref{sw2}, we conclude that the continuous representation $y_{\bs_k}$ of $\bs_k$ can be recovered via the following operation
\begin{align}\label{cop}
    y_{\bs_k}(u) = \int_0^1 \bW_{\bC_m}(u,v) y_{\bs_{k-1}}(v) {\sf d}v\;.
\end{align}
Since the coVariance filter output $\bz$ is a weighted aggregation of the terms $\bs_k$, we can write its continuous representation $y_{\bz}$ as
\begin{align}\label{cov_cont}
    y_{\bz}(u) = \sum\limits_{k=0}^K h_k y_{\bs_k}(u)\;.
\end{align}
Using the mathematical steps leading up to~\eqref{cov_cont}, we have shown that the continuous representation of the covariance filter output $\bz$ can be recovered via the convolution operations over the graphon representation $\bW_{\bC_m}$ in~\eqref{sw2} and~\eqref{cop}. Also, $\bz$ and $y_{\bz}$ are operationally interchangeable. Moreover, we can also extrapolate this correspondence between $\bz$ and $y_{\bz}$ to covariance perceptrons and VNNs with multi-layer architecture and MIMO information processing. The extension of this observation to coVariance perceptron and a basic VNN is trivial as the coVariance output is evaluated after application of pointwise non-linearity $\sigma$ on $\bz$ and a basic VNN is formed by stacking multiple coVariance perceptrons and number of inputs and outputs at each layer (i.e., $F$) being set to $1$.

We use the notation $\bx_m$ to denote an input vector with $m$ features. Thus, if VNN output $\Phi(\bx_m;\bC_m,\cH)$ is of size $m\times 1$ and we have $F=1$ and number of layers $L$, its continuous approximation $y_{\Phi(\bx_m;\bC_m,\cH)}$ can be recovered by a convolutional architecture instantiated on $\bW_{\bC_m}$ with input $y_{\bx_m}$. For a VNN with MIMO processing, each VNN layer has multiple $m$-dimensional inputs and multiple $m$-dimensional outputs. Thus, we can equivalently define an architecture capable of performing MIMO processing that is instantiated on $\bW_{\bC_m}$ and $\bx_m$ and produces multiple continuous representations as the output. Such an architecture has previously been studied in the form of graphon neural networks~\cite{ruiz2021graphon}. In this context, we define the model $\tilde\Phi(y_{\bx_m}; \bW_{\bC_m}, \cH)$ that is modeled via convolution operations over $\bW_{\bC_m}$ in~\eqref{cop} and has the same architecture as the VNN $\Phi(\bx_m;\bC_m,\cH)$. Note that the outputs of $\tilde\Phi(y_{\bx_{m}}; \bW_{\bC_m}, \cH)$ are continuous representations of the outputs of VNN $\Phi(\bx_m;\bC_m,\cH)$ (see also Fig.~\ref{vnn_transfer_overview} for an illustration). Thus, we can investigate the transferability of parameters $\cH$ between VNNs instantiated on covariance matrices $\bC_{m_1}$ and $\bC_{m_2}$ by analyzing the difference between $\tilde\Phi(y_{\bx_{m_1}}; \bW_{\bC_{m_1}}, \cH)$ and $\tilde\Phi(y_{\bx_{m_2}}; \bW_{\bC_{m_2}}, \cH)$.

In this context, our analysis hinges on the setting in which the graphon approximations $\bW_{\bC_{m_1}}$ and $\bW_{\bC_{m_2}}$ belong to a sequence of graphon approximations $\{\bW_{\bC_m}\}$ that converges to a graphon $\bW$. Thus, we also consider an information processing architecture $\tilde\Phi(y; \bW, \cH)$ instantiated on graphon $\bW$, such that $y$ and continuous representations $y_{\bx_m}$ always satisfy $y_{\bx_m}(\rho_i) = y(\rho_i), \forall i\in \{1,\dots,m\}$. Here, we can also interpret $\tilde\Phi(y; \bW, \cH)$ as a generative model with $\tilde\Phi(y_{\bx_m}; \bW_{\bC_m}, \cH)$ being an instance of $\tilde\Phi(y; \bW, \cH)$ at resolution $m$. Thus, our analysis of transferability of VNNs also includes the study of convergence of outputs from $\tilde\Phi(y_{\bx_m}; \bW_{\bC_m}, \cH)$ with that from $\tilde\Phi(y; \bW, \cH)$.

To this end, we now formally define a convolution filter over a graphon and characterize its frequency response. We denote the $k$-hop aggregation (analogous to $\bC^k \bx$) on $\bW_{\bC_m}$ and continuous representation $y_{\bx_m}$ by the operator $T_{\bW_{\bC_m}}^k y_{\bx_m}$ that is given by
\begin{align}
    (T_{\bW_{\bC_m}}^k y_{\bx_m})(u) \dff \int_0^1 \bW_{\bC_m}(u,v) (T_{\bW_{\bC_m}}^{k-1} y_{\bx_m})(v) {\sf d}v\;,\label{graphonconv}
\end{align}
for any $k>1$, where 
\begin{align}
    (T_{\bW_{\bC_m}} y_{\bx_m})(u) \dff \int_0^1 \bW_{\bC_m}(u,v) y_{\bx_m}(v) {\sf d}v\;.
\end{align}
Thus, based on the discussion above, $T^k_{\bW_{\bC_m}} y_{\bx_m}$ and $\bC_m^k\bx_m$ are operationally interchangeable.  We can also define $k$-hop aggregation over $\bW$ using the operator $T_{\bW} y$ when $y$ is related to $y_{\bx_m}$ by $ y_{\bx_m}(\rho_i) = y(\rho_i)$, where $\rho_i$ is defined in~\eqref{interval1}. Thus, graphon $\bW$ and the continuous representation $y$ can be seen as generative models for covariance matrix $\bC_m$ and data point $\bx_m$. This observation is in parallel to that in the context of graphs and graphons~\cite{ruiz2021graphon}. We denote the graphon filter for a set of filter taps $\cH = \{h_k\}_{k=0}^K$ by $ \Psi(y;\bW, \cH) :[0,1]\rightarrow \mR$, which is defined as 
\begin{align}
    \Psi(y;\bW, \cH) (u) &\dff \sum\limits_{k=0}^K  h_k (T_{\bW}^k y)(u)\;.\label{g1}
\end{align}
Similar to coVariance filter, we can characterize the frequency response of a graphon filter via using eigendecomposition of $\bW$ in~\eqref{g1}. Because $\bW$ is bounded and symmetric, we can express the spectral decomposition of $\bW$ as
\begin{align}
    \bW(u,v) = \sum\limits_{i\in \mZ\backslash \{0\}} \eta_i \Gamma_i(u)\Gamma_i(v)\;, \label{spectralgraphon}
\end{align}
where $\eta_i, \forall i\in \mZ\backslash\{0\}$ are eigenvalues and  $\Gamma_i$ are the eigensignals of $\bW$. Therefore, we can re-write~\eqref{g1} as
\begin{align}
   \Psi(y;\bW, \cH) (u) &= \sum\limits_{i\in \mZ\backslash\{0\}} \sum\limits_{k=0}^K h_k  \eta_i^k \Gamma_i(u) \int_0^1 \Gamma_i(v)  y(v) {\sf d}v\;,\label{g2}\\
     &= \sum\limits_{i\in \mZ\backslash\{0\}}  \tilde h(\eta_i) \Gamma_i(u) \int_0^1 \Gamma_i(v)  y(v) {\sf d}v\;,\label{g3}
\end{align}
for $u\in [0,1]$. Note that~\eqref{g2} follows from~\eqref{g1} using~\eqref{spectralgraphon} and~\eqref{graphonconv}, and  we have used the definition $ \tilde h(\eta) \dff  \sum_{k=0}^K h_k  \eta^k$ in~\eqref{g3}. The term~$ \tilde h(\eta_i)$ characterizes the frequency response of a graphon filter and depends on the filter taps $\{ h_k\}$ and the graphon eigenvalues. The analysis of $\|\Psi(y_{\bx_m};\bW_{\bC_m},\cH) - \Psi(y;\bW,\cH)\|_2$ in Appendix~\ref{pf_thm2} reveals that the closeness between the graphon filter approximation  $\Psi(y_{\bx};\bW_{\bC_m}, \cH)$  and the graphon filter $\Psi(y;\bW, \cH)$  scales inversely with the dimension $m$.


\section{Cut Distance}\label{cutdist}
Here, we borrow the definition of cut distance between two covariance matrices $\bC_{m_1}$ and $\bC_{m_2}$ with $m_1\neq m_2$ from that for weighted graphs in~\cite{borgs2008convergent}. First,  $\bC_{m_1}$ and $\bC_{m_2}$ are normalized such that ${\sf tr}(\bC_{m_1}) = 1$ and ${\sf tr}(\bC_{m_2}) = 1$. Next, we define a family of matrices $m_1\times m_2$ sized matrices ${\cal K}(\bC_{m_1}, \bC_{m_2})$, such that, for every element ${\bf K} \in {\cal K}(\bC_{m_1}, \bC_{m_2})$, we have
\begin{align}
    \sum\limits_{j=1}^{m_2}[{\bf K}]_{ij} = [\bC_{m_1}]_{ii} \quad \text{and}\quad   \sum\limits_{i=1}^{m_1}[{\bf K}]_{ij} = [\bC_{m_2}]_{jj} \;.
\end{align}
The family ${\cal K}(\bC_{m_1}, \bC_{m_2})$ is referred to as `fractional overlay' that conceptually defines the mapping between individual features of $\bC_{m_1}$ and $\bC_{m_2}$. Since the covariance matrix provides a weighted graph representation with features as nodes, the fractional overlay describes the overlap between the nodes of two weighted graphs with different number of nodes. Next, by leveraging a member ${\bf K}$ of fractional overlay ${\cal K}(\bC_{m_1}, \bC_{m_2})$, we define matrices $\bC_{m_1}[{\bf K}]$ and $\bC_{m_2}[{\bf K}^{\sf T}]$ on the set of pairs in $[m_1]\times[m_2]$, where $[m_1] = [1,\dots,m_1]$. Hence, the size of matrices $\bC_{m_1}[{\bf K}]$ and   $\bC_{m_2}[{\bf K}]$ is $m_1m_2\times m_1m_2 $. In both $\bC_{m_1}[{\bf K}]$ and $\bC_{m_2}[{\bf K}^{\sf T}]$, the weight at the diagonal element associated with the element $(i,j)$ in $[m_1]\times [m_2]$ is $[{\bf K}]_{ij}$. Furthermore, for $i,j \in [m_1]$ and $a,b\in [m_2]$, the off-diagonal element associated with the paired element $((i,a),(j,b))$ is $[\bC_{m_1}]_{ij}$ in $\bC_{m_1}[{\bf K}]$ and $[\bC_{m_2}]_{ab}$ in $\bC_{m_2}[{\bf K}^{\sf T}]$. Since $\bC_{m_1}[{\bf K}]$ and $\bC_{m_2}[{\bf K}^{\sf T}]$ derived from $\bC_{m_1}$ and $\bC_{m_2}$ are of the same size, the distance between them can be readily defined. The cut distance between $\bC_{m_1}$ and $\bC_{m_2}$ is equivalent to the cut distance between $\bC_{m_1}[{\bf K}]$ and $\bC_{m_2}[{\bf K}^{\sf T}]$ and is defined as
\begin{align}\label{cutdist_eq}
    \delta_{\square}(\bC_{m_1},\bC_{m_2})\dff \min_{{\bf K} \in {\cal K}(\bC_{m_1},\bC_{m_2})}d_{\square}(\bC_{m_1}[{\bf K}],\bC_{m_2}[{\bf K}^{\sf T}]])\;,
\end{align}
where $d_{\square}$ is a distance metric between matrices of the same size and defined next. For two matrices $\bC_m$ and $\bD_m$ with same diagonal elements, the distance $d_{\square}(\bC_m,\bD_m)$ is defined as
\begin{align}
    d_{\square}(\bC_m,\bD_m) \dff \max_{S,T\in [m]} \frac{1}{{\sf tr}(\bC_m)^2} |e_{\bC_m}(S,T) - e_{\bD_m}(S,T) |\;,
\end{align}
where 
\begin{align}
    e_{\bC_m}(S,T)  \dff \sum\limits_{i\in S, j\in T} [\bC_m]_{ii} [\bC_m]_{jj} [\bC_m]_{ij}\;.
\end{align}
\section{Proof of Theorem~\ref{transferthm}}\label{pf_thm2}
In Theorem~\ref{transferthm}, we compare the continuous representations of the $f$-th outputs of VNNs $\Phi(\bx_{m_1};\bC_{m_1},\cH)$ and $\Phi(\bx_{m_2};\bC_{m_2},\cH)$. Our discussion in Appendix~\ref{gip} showed that these continuous representations appear naturally as the outputs of the architectures $\tilde\Phi(y_{\bx_{m_1}};\bW_{\bC_{m_1}},\cH)$ and $\tilde\Phi(y_{\bx_{m_1}};\bW_{\bC_{m_1}},\cH)$ instantiated on graphon approximations $\bW_{\bC_{m_1}}$ and $\bW_{\bC_{m_2}}$, respectively. Therefore, our subsequent analysis is focused on the comparisons between their constituent graphon filters (see Appendix~\ref{gip} for definition) that eventually enables us to establish the convergence between $f$-th outputs of $\tilde\Phi(y_{\bx_{m_1}};\bW_{\bC_{m_1}},\cH)$ and $\tilde\Phi(y_{\bx_{m_1}};\bW_{\bC_{m_1}},\cH)$. We refer the reader to Appendix~\ref{gip} for the details on information processing architecture defined on the graphon approximations and understanding of the relationship between $\tilde\Phi(y_{\bx_{m_1}};\bW_{\bC_{m_1}},\cH)$ and corresponding VNN $\Phi(\bx_{m_1};\bC_{m_1},\cH)$.

We begin by establishing various results pertaining to the comparisons between $\bW$ and $\bW_{\bC_m}$, $y$ and $y_{\bx_m}$, and difference between eigenvalues of two distinct graphons. We leverage the $(\Omega,\zeta)$-dominant property of sequence of covariance matrices $\{\bC_m\}$ in~\eqref{dom_prop} and the Lipschitz condition of graphon in~\eqref{lips} to establish the following result. 
\begin{lemma}\label{lm3}
    Given an $\alpha$-Lipschitz graphon $\bW$ and $\bW_{\bC_m}$ as graphon representation of a $(\Omega,\zeta)$-dominant covariance matrix $\bC_m$, we have
    \begin{align}
        \|\bW - \bW_{\bC_m}\|_2 \leq {\frac{\alpha\Omega^{3/2}}{m^{3\zeta/2-1}}}\;.
    \end{align}
\end{lemma}
\begin{proof}
    From the construction of $\bW_{\bC_m}$, we have
    \begin{align}
         \|\bW - \bW_{\bC_m}\|_2 &= \Big(\int_0^1\int_0^1 \|\bW(u,v) - \bW_{\bC_m}(u,v)\|^2{\sf d}u {\sf d}v \Big)^{\frac{1}{2}}\;,\\
         &= \Big(\sum\limits_{i,j}\int_{\cU_i}\int_{\cU_j} \|\bW(u,v) - \bW_{\bC_m}(u,v)\|^2{\sf d}u {\sf d}v \Big)^{\frac{1}{2}}\;.
    \end{align}
    Without loss of generality, we assume that $\cU_1 = [0,\rho_1]$ is the largest interval. Using the $\alpha$-Lipschitz continuity of graphon $\bW_{\bC_m}$ and noting that $\bW_{\bC_m}(\rho_i,\rho_j) = \bW(\rho_i,\rho_j)$, we have
    \begin{align}
         \|\bW - \bW_{\bC_m}\|_2 &\leq \Big(\sum\limits_{i,j}\int_{I_i}\int_{I_j} \alpha^2 (|u| + |v|)^2 {\sf d}u {\sf d}v \Big)^{\frac{1}{2}}\;,\\
         &\leq \Big(m^2 \int_0^{\rho_1}\int_0^{ \rho_1} \alpha^2 (|u| + |v|)^2 {\sf d}u {\sf d}v \Big)^{\frac{1}{2}}\;,\\
         &\leq \Big(m^2\int_0^{\rho_1}\int_0^{\rho_1} \alpha^2 (|u| + |v|) {\sf d}u {\sf d}v \Big)^{\frac{1}{2}}\;,\\
         &\leq {\alpha m \rho_1^{3/2}}\;.
    \end{align}
   Using the assumption that $\bC_m$ is $(\Omega,\zeta)$-dominant, we have
   \begin{align}
        \|\bW - \bW_{\bC_m}\|_2 &\leq \frac{\alpha\Omega^{3/2}}{m^{3\zeta/2 - 1}}\;.
   \end{align}
\end{proof}
\noindent
Next, we characterize the difference between a graphon signal $y \in L_2([0,1])$ and approximation $y_{\bx_m}$ obtained from a random sample $\bx$ in Lemma~\ref{diff_sig}. For this purpose, we have the following assumption:
a graphon signal $y$ satisfies $|y(a) - y(b)|\leq {\alpha_2} |a-b|, \forall a, b \in[0,1]$. We term a graphon signal satisfying this property as $\alpha_2$-Lipschitz graphon signal.
\begin{lemma}\label{diff_sig}
    Given an $\alpha_2$-Lipschitz graphon signal $y$ and a graphon signal approximation $y_{\bx_m}$ obtained from $\bx_m \in \mR^{m\times 1}$, we have
    \begin{align}
        \|y - y_{\bx_m}\|_2 \leq  \frac{\alpha_2 \Omega^{3/2}}{m^{3\zeta/2 -1}}\;.
    \end{align}
\end{lemma}
\begin{proof}
Note that
\begin{align}
    \|y - y_{\bx_m}\|_2 &= \sum\limits_{\cU_i} \|y - y_{\bx_m}\|_{L_2[I_i]}\;,\\
    &=  \sum\limits_{i=1}^{m} \Big(\int_{\rho_{i-1}}^{\rho_{i}}(y(u)-y_{\bx_m}(u))^2 {\sf d}u\Big)^{\frac{1}{2}}\;,
\end{align}
where we have $\rho_{0} = 0$.
Using the Lipschitz property of graphon signal and $(\Omega,\zeta)$-property of $\bC_m$, we have
\begin{align}
    \|y - y_{\bx_m}\|_2 \leq m\Big(\alpha_2^2 \int_{0}^{\rho_1 }u^2 {\sf d}u\Big)^{\frac{1}{2}}
    &\leq \frac{\alpha_2 \Omega^{3/2}}{m^{3\zeta/2 -1}}\;.
\end{align}
    
\end{proof}
\noindent
Next, we state Proposition 4 from~\cite{ruiz2020graphon} that characterizes a bound on the difference between eigenvalues from two graphons.
\begin{lemma}[Proposition 4 from~\cite{ruiz2020graphon}]\label{prop4}
    Consider two graphons $\bW$ and $\bW'$ with set of eigenvalues $\{\eta_i\}_{i=1}^{\infty}$ and $\{\beta_i\}_{i=1}^{\infty}$, respectively. Then, for all $i\in \mZ^{+}$\;, we have
    \begin{align}
        |\eta_i - \beta_i| \leq \|T_{\bW-\bW'} \|_2 \leq \|\bW - \bW'\|_2\;.
    \end{align}
\end{lemma}
\noindent
Next, we leverage Lemmas~\ref{lm3},~\ref{diff_sig}, and~\ref{prop4} to bound the difference between graphon convolution $\Psi(y;\bW,\cH)$ and convolution by the approximation $\Psi(y_{\bx_m};\bW_{\bC_m},\cH)$ realized from graph filter over $\bC_m$. The subsequent line of analysis is similar to that in~\cite{ruiz2021graphon}. For this purpose, we have the following assumptions.
\begin{itemize}
    \item[A1] The graphon $\bW$ is $\alpha$-Lipschitz.
    \item[A2] The graphon signal $y$ is $\alpha_2$-Lipschitz.
    \item[A3] The covariance matrix $\bC_m$ belongs to a convergent $(\Omega,\zeta)$-dominant sequence of covariance matrices. Also, the corresponding sequence of graphon approximations belongs to the family of graphon $\bW$.
    \item[A4] The frequency response is band-limited, such that, $|\tilde h(\eta)| = 0$ for $\eta \leq \eta_{\sf c}$. Furthermore, we assume that $m_{\sf c}$ largest eigenvalues of graphon $\bW$ in terms of magnitude satisfy $|\eta|  > \eta_{\sf c}$ and the set of such eigenvalues is denoted by $\cC$. Also, the graphon filter is non-expanding and satisfies $|\tilde h(\eta)| \leq 1,\forall \eta$ and $|\tilde h(\eta_i) - \tilde h(\hat\eta_i)| \leq \alpha_3|\eta_i - \hat\eta_i|$.
\end{itemize}
In the following Lemma, we use the notations $\{\hat\eta_i\}$ and $\{\hat\Gamma_i\}$ for the set of eigenvalues and eigenfunctions, respectively, of $\bW_{\bC_m}$.
\begin{lemma}[Transferability of Graphon Filters]\label{graphon_stability}
    For a convolution $\Psi(y;\bW,\cH)$ and its approximation~$\Psi(y_{\bx_m};\bW_{\bC_m},\cH)$, under the assumptions A1-A4 and for $\|y\|_2\leq 1$, we have
    \begin{align}
        \|\Psi(y;\bW,\cH) - \Psi(y_{\bx_m};\bW_{\bC_m},\cH)\|_2 \leq \frac{ \Omega^{3/2}}{m^{3\zeta/2-1}} \left(\alpha_2+\alpha \Big[\alpha_3 + \frac{\pi m_{\sf c}}{2\Delta_c} \Big] \right)\;,
    \end{align}
    where $\Delta_c = \min_{i\neq j; i, j \in \cC} \{|\eta_i - \hat\eta_j|\}$.
\end{lemma}
\begin{proof}
    Note that we can rewrite $ \|\Psi(y;\bW,\cH) - \Psi(y_{\bx_m};\bW_{\bC_m},\cH)\|_2$ as
    \begin{align}
         \|\Psi(y;\bW,\cH) - \Psi(y_{\bx_m};\bW_{\bC_m},\cH)\|_2 &=  \|\Psi(y;\bW,\cH) - \Psi(y;\bW_{\bC_m},\cH)  \nonumber\\
         &\quad +\Psi(y;\bW_{\bC_m},\cH) - \Psi(y_{\bx_m};\bW_{\bC_m},\cH)\|_2 \;,
    \end{align}
    and using triangle inequality, we have
    \begin{align}\label{ger1}
         \|\Psi(y;\bW,\cH) - \Psi(y_{\bx_m};\bW_{\bC_m},\cH)\|_2 &\leq \underbrace{\|\Psi(y;\bW,\cH) - \Psi(y;\bW_{\bC_m},\cH)\|_2 }_\text{Term 1} \nonumber \\
         &\quad +\underbrace{\|\Psi(y;\bW_{\bC_m},\cH) - \Psi(y_{\bx_m};\bW_{\bC_m},\cH)\|_2}_\text{Term 2} \;.
    \end{align}
\noindent
Next, we analyze Terms 1 and 2 from \eqref{ger1} separately. 

\noindent
{\bf Analysis of Term 1}.  Using the expansion of $\Psi(y;\bW,\cH)$ and $\Psi(y;\bW_{\bC_m},\cH)$, we have
\begin{align}
    \|\Psi(y;\bW,\cH) - \Psi(y;\bW_{\bC_m},\cH)\|_2 = \Big(\int_0^1 f^2(v){\sf d}v\Big)^{1/2}\;,
\end{align}
where
\begin{align}\label{fv}
    f(v) = \sum\limits_{i=\in \mZ\backslash \{0\}}\Big[ \tilde h(\eta_i)\Gamma_i(v) \int_0^1 y(u)\Gamma_i(u) {\sf d}u - \tilde h(\hat\eta_i)\hat\Gamma_i(v) \int_0^1 y(u) \hat\Gamma_i(u) {\sf d}u\Big]\;.
\end{align}
By adding and subtracting $\tilde h(\hat\eta_i)\Gamma_i(v) \int_0^1 y(u)\Gamma_i(u) {\sf d}u $ in~\eqref{fv} and using the triangle inequality, we have
\begin{align}
     \|\Psi(y;\bW,\cH) - \Psi(y;\bW_{\bC_m},\cH)\|_2 \leq \Big(\int_0^1 f_1^2(v){\sf d}v\Big)^{1/2} + \Big(\int_0^1 f_2^2(v){\sf d}v\Big)^{1/2} = \|f_1\|_2 + \|f_2\|_2\;,
\end{align}
where
\begin{align}
    f_1(v) = \sum\limits_{i\in \mZ\backslash \{0\}}\Big[  (\tilde h(\eta_i) - \tilde h(\hat\eta_i))\Gamma_i(v) \int_0^1 y(u)\Gamma_i(u) {\sf d}u\Big]\;,
\end{align}
and
\begin{align}
    f_2(v) = \sum\limits_{i\in \mZ\backslash \{0\}}\Big[ \tilde h(\hat\eta_i)\Gamma_i(v) \int_0^1 y(u)(\Gamma_i(u) - \hat\Gamma_i(u)) {\sf d}u\Big] \;.
\end{align}
Using the Lipschitz property of graphon filter, we have $|\tilde h(\eta_i) - \tilde h(\hat\eta_i)| \leq \alpha_3|\eta_i - \hat\eta_i|$. Therefore, using Lemma~\ref{prop4} and Lemma~\ref{lm3}, we have 
\begin{align}\label{bnf1}
    \|f_1\|_2 \leq \frac{\alpha_3\alpha\Omega^{3/2}}{m^{3\zeta/2-1}}\;.
\end{align}
for any $y$ that satisfies $\|y\|_2 \leq 1$. To analyze $\|f_2\|_2$, we note that by Cauchy-Schwarz inequality, we have
\begin{align}
    \|f_2\|_2 &\leq \sum\limits_{i\in \mZ\backslash \{0\}} |\tilde h(\hat\eta_i)| \|\Gamma_i\|_2 \|y (\Gamma_i - \hat\Gamma_i)\|_2\;,\label{pl1}\\
    &\leq \sum\limits_{i\in \mZ\backslash \{0\}}|\tilde h(\hat\eta_i)| \|\Gamma_i - \hat\Gamma_i\|_2\;,\label{pl2}
\end{align}
where~\eqref{pl2} follows from~\eqref{pl1}, without loss of generality for $\|y\|_2=1$, $\|\Gamma_i\|_2=1$ and another application of Cauchy-Schwarz inequality. Next, we note that the integral operator $T_{\bW}$, such that, $(T_{\bW}y)(v) = \int_{0}^1 \bW(u,v) y(u) {\sf d}u$ is a self-adjoint Hilbert-Schmidt operator and $\bW$ admits the spectral decomposition with $\{\eta_i\}$ as eigenvalues and $\{\Gamma_i\}$ as eigensignals. Therefore, to analyze $\|\Gamma_i - \hat\Gamma_i\|_2$, we note that $\Gamma_i$ is projection of operator $T_{\bW}$ associated with eigenvalue $\eta_i$ and $\hat\Gamma_i$ is projection of operator $T_{\bW_{\bC_m}}$ associated with eigenvalue $\hat\eta_i$. By dividing the spectrum of $T_{\bW}$ as ${\sf spec}(T_{\bW}) = \{\eta_i\} \cup \{\eta_j\}_{j\neq i}$ and that of $T_{\bW_{\bC_m}}$ as ${\sf spec}(T_{\bW_{\bC_m}}) = \{\hat\eta_i\} \cup \{\hat\eta_j\}_{j\neq i}$, we apply Proposition 2.3 from~\cite{seelmann2014notes} to have
\begin{align}\label{pol1}
    \|\Gamma_i - \hat\Gamma_i\|_2 \leq \frac{\pi}{2} \frac{\|T_{\bW} - T_{\bW_{\bC_m}}\|_2}{d_i}\;,
\end{align}
where $d_i>0$ is a constant that satisfies $|\eta_i - \hat\eta_{i+1}| \geq d_i$, $|\eta_i - \hat\eta_{i-1}| \geq d_i$, $|\eta_{i+1} - \hat\eta_{i}| \geq d_i$, and $|\eta_{i-1} - \hat\eta_i| \geq d_i$. Using~\eqref{pol1}, Lemma~\ref{prop4} and Lemma~\ref{lm3} in~\eqref{pl2}, we have
\begin{align}\label{bnd1}
    \|f_2\|_2 \leq \frac{\pi \alpha\Omega^{3/2}}{2\Delta_cm^{3\zeta/2-1}} \sum\limits_{i\in \mZ\backslash{0}} |\tilde h(\hat\eta_i)|\;,
\end{align}
where $\Delta_c = \min_{i} d_i$. Next, we note that under Assumption A4, the frequency response of graphon filter is band-limited, i.e., we have $\sum\limits_{i\in\mZ\backslash{0}} |\tilde h(\hat\eta_i)|\leq m_{\sf c}$ when the frequency response is limited according to A4. We denote the set of $m_{\sf c}$ largest eigenvalues (in terms of magnitude) of $\bW$ by ${\cal C}$.  In this scenario, we can rewrite~\eqref{bnd1} as
\begin{align}\label{bnd2}
    \|f_2\|_2 \leq \frac{\pi \alpha \Omega^{3/2} m_{\sf c}}{2\Delta_cm^{3\zeta/2-1}}\;.
\end{align}
Clearly, there is a trade-off between $m_c$ and $\zeta$ as we must have $m_c < m^{3\zeta/2-1}$ and $\zeta > 2/3$ for~\eqref{bnd2} to have decreasing behavior in $m$. Equations~\ref{bnf1} and~\ref{bnd2} provide the upper bound on Term 1.

\noindent
{\bf Analysis of Term 2}. We can expand term 2 as
\begin{align}
    \|\Psi(y;\bW_{\bC_m},\cH) - \Psi(y_{\bx_m};\bW_{\bC_m}, \cH)\|_2 = \Big(\int_0^1g^2(v) {\sf d}v\Big)^{1/2}\;,
\end{align}
where
\begin{align}\label{t21}
    g(v) = \sum\limits_{i=1}^{\infty} \Big[ \tilde h(\eta_i) \hat\Gamma_i(v) \int_0^1 (y(u)-y_{\bx_m}(u))\hat\Gamma_i(u) {\sf d}u \Big]\;.
\end{align}
Therefore, using~\eqref{t21}, we have
\begin{align}
     \|\Psi(y;\bW_{\bC_m},\cH) - \Psi(y_{\bx_m};\bW_{\bC_m}, \cH)\|_2 =  \|\Psi(y-y_{\bx_m};\bW_{\bC_m},\cH)\|_2\;.
\end{align}
Note that for a frequency response that satisfies $\tilde h(\eta) < 1$, the graphon filter is non-expanding and therefore, we have
\begin{align}
    \|\Psi(y;\bW_{\bC_m},\cH) - \Psi(y_{\bx_m};\bW_{\bC_m}, \cH)\|_2 \leq \|y-y_{\bx_m}\|_2\;.
\end{align}
Using Lemma~\ref{diff_sig}, we have
\begin{align}\label{bndt3}
     \|\Psi(y;\bW_{\bC_m},\cH) - \Psi(y_{\bx_m};\bW_{\bC_m}, \cH)\|_2 \leq \frac{\alpha_2 \Omega^{3/2}}{m^{3\zeta/2 -1}}\;.
\end{align}
Therefore, by combining the upper bounds on Term 1 and Term 2 from~\eqref{bnf1},~\eqref{bnd2}, and~\eqref{bndt3}, the proof of Lemma~\ref{graphon_stability} is concluded. 
\end{proof}
\noindent 
Lemma~\ref{graphon_stability} establishes the transference between the graphon $\bW$ and the graphon approximation $\bW_{\bC_m}$ obtained from the covariance matrix $\bC_m$. We leverage the result in Lemma~\ref{graphon_stability} to establishing the transference for graphon neural networks in a similar setting. We denote the $f$-th output for graphon neural network $ \tilde\Psi(y;\bW_{\bC_m},\cH)$ with $F$ outputs in the final layer by $[\tilde\Psi(y;\bW_{\bC_m},\cH)]_f$. 
\begin{lemma}[Transferability of Graphon Neural Networks]\label{grphnvnn}
Consider a graphon neural network $\tilde\Phi(\cdot;\bW,\cH)$ with $L$ layers and $F$ outputs per layer and a VNN  $\Phi(\cdot;\bC_{m},\cH)$ with graphon neural network representation as $\tilde\Phi(\cdot;\bW_{\bC_m},\cH)$. If the covariance matrix $\bC_{m}$ belongs to a $(\Omega,\zeta)$-dominant sequence of covariance matrices and its graphon approximation $\bW_{\bC_m}$ belongs to a graphon family of $\alpha$-Lipschitz graphon $\bW$, then under the assumptions A1-A4, for $\|y\|_2\leq 1$ and $2/3<\beta\leq 1$, we have
\begin{align}
\|[\tilde\Phi(y;\bW,\cH)]_f -[\tilde\Phi(y_{\bx_m};\bW_{\bC_m},\cH)]_f \|_2 \leq {LF^{L}} \left(\frac{ \Omega^{3/2}}{m^{3\zeta/2-1}} \left[\alpha_2 + \alpha\left[\alpha_3 + \frac{\pi m_{\sf c}}{2\Delta_c}\right] \right]\right)\;.\label{pf_ty}
\end{align}

\end{lemma}
\noindent
The proof of Lemma~\ref{grphnvnn} leverages Lemma~\ref{graphon_stability} and accommodates the impact of multi-layer VNN architecture. We refer the reader to equations (23)-(28) in~\cite{ruiz2020graphon} for exact analytical steps. Finally, by applying the triangle inequality on~\eqref{pf_ty}, we establish the transference between graphon neural network approximations $\tilde\Phi(\cdot;\bW_{\bC_{m_1}},\cH)$ and $\tilde\Phi(\cdot;\bW_{\bC_{m_2}},\cH)$ for VNNs $\Phi(\cdot;{\bC_{m_1}},\cH)$ and $\Phi(\cdot;{\bC_{m_2}},\cH)$, respectively, and the proof of Theorem~\ref{transferthm} is concluded.  

\clearpage

\section{Convergence of covariance matrices from FTDC datasets}\label{fig_converge}
Cut metric allows the comparison of matrix representations of graphs of different sizes~\cite{lovasz2012large}. Here, an estimated cut norm between a pair of covariance matrices was evaluated using the {\sf cutnorm} package in python (implementation based on~\cite{alon2004approximating}). Here, we investigate the convergence of the covariance matrices formed by datasets curated according to different scales of the Schaefer's atlas. For this purpose, we also consider the datasets curated according to 200 and 400 parcellations of Schaefer's atlas in addition to the FTDC datasets. Figure~\ref{cutmetrics}a plots the cut norm evaluated for the series $\{\hat\bC_{100},\hat\bC_{200}, \hat\bC_{300},\hat\bC_{400}, \hat\bC_{500}\}$, where $\hat\bC_{m}$ is the anatomical covariance matrix constructed from the cortical thickness features curated according to $m$ parcellations of Schaefer's atlas. The cut norm between the consecutive members of this series reduced with an increase in the number of parcels, thus, consistent with the properties of a Cauchy sequence in this metric. 
\begin{figure}[htbp]
  \centering
  \includegraphics[scale=0.5]{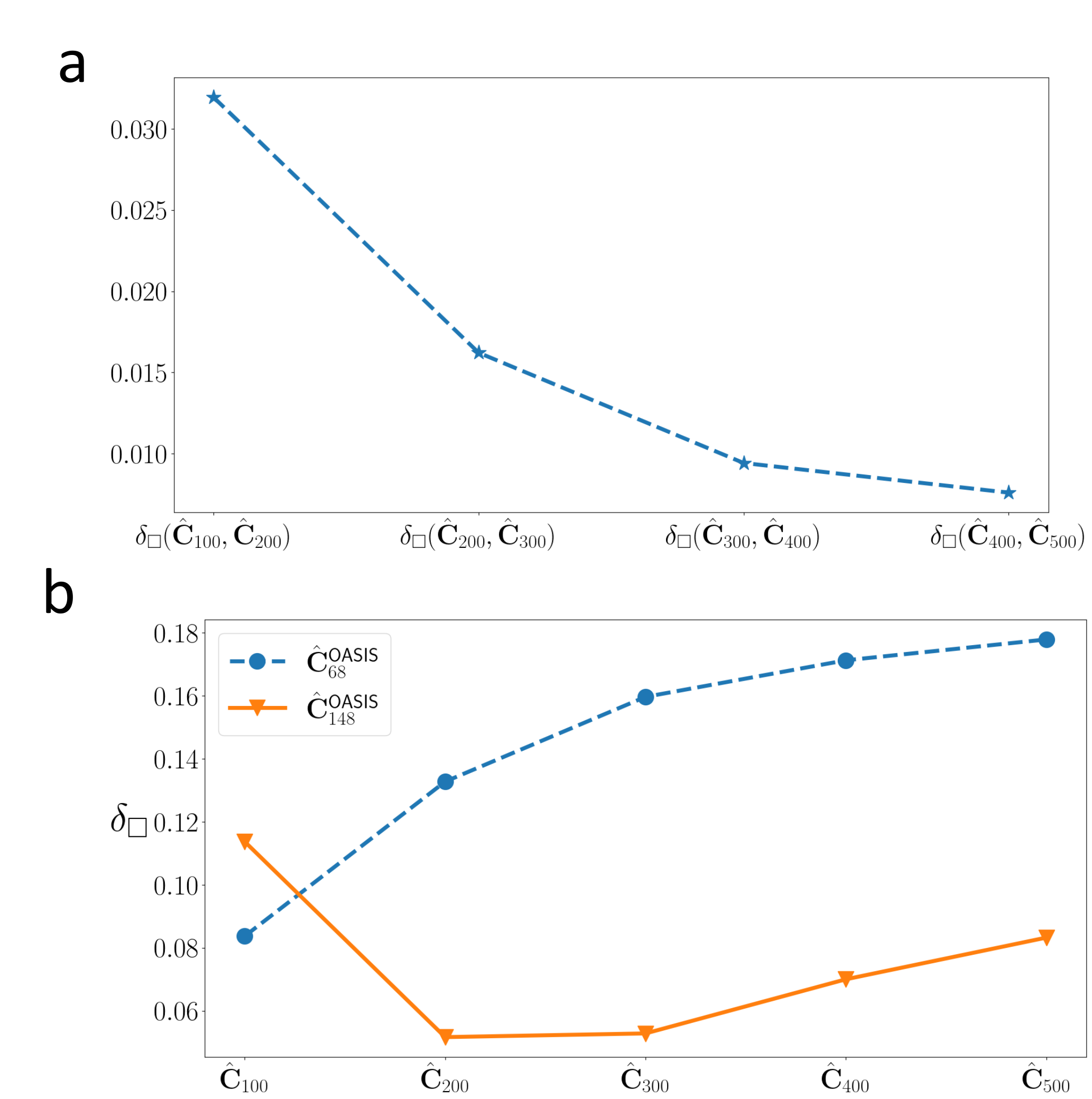}
   \caption{{\bf Cut norm between covariance matrices.} Panel~{\bf a} plots the estimated cut norm between consecutive elements of the series of covariance matrices $\{\hat\bC_{100},\hat\bC_{200}, \hat\bC_{300},\hat\bC_{400}, \hat\bC_{500}\}$. Panel~{\bf b} plots the estimated cut norm between the covariance matrices derived from OASIS-3 dataset according to DK and DKT atlases and the covariance matrices from FTDC datasets.}
   \label{cutmetrics}
\end{figure}
We also note that the distance between covariance matrices derived from OASIS-3 dataset and those from FTDC datasets was significantly greater than the distances between the covariance matrices associated with different resolutions of Schaefer's atlas (Fig.~\ref{cutmetrics}b). In Fig.~\ref{cutmetrics}b, the covariance matrix $\hat\bC_{148}^{\sf OASIS}$ was estimated from the cortical thickness features curated according to DKT atlas for HC group. The  $\hat\bC_{68}^{\sf OASIS}$ was estimated from the cortical thickness features curated according to Desikan-Killiany atlas~\cite{desikan2006automated} for the same group. According to Theorem~\ref{transferthm}, the VNN performance is transferable across datasets whose covariance matrices are part of a converging sequence. Hence, we expected the VNNs to be transferable between the cortical thickness datasets curated according to different scales of Schaefer's atlas. However, VNN transferability was not guaranteed between datasets organized according to Schaefer's atlas and those according to DKT atlas.


\section{Summary of the performance of VNNs on the test sets}\label{vnn_test_perf}
Here, we provide the performance of nominal VNNs on the test set for different sets of VNNs that were trained on FTDC100, FTDC500, and OASIS-3 datasets. The test set was completely unseen during the training procedure.

\noindent
{\bf OASIS-3}: The MAE across $100$ nominal models on the test set was $5.87\pm 0.177$ years and Pearson's correlation was $0.425\pm 0.006$.

\noindent
{\bf FTDC100}: The MAE across $100$ nominal models on the test set was $3.72\pm 0.22$ years and Pearson's correlation was $0.78\pm 0.01$. 

\noindent
{\bf FTDC300}: The MAE across $100$ nominal models on the test set was $3.93\pm 0.34$ years and Pearson's correlation was $0.75\pm 0.008$.

\noindent
{\bf FTDC500}: The MAE across $100$ nominal models on the test set was $4.07\pm 0.36$ years and Pearson's correlation was $0.72\pm 0.007$. 

\clearpage

\section{Illustration of regional residual analysis from VNN model outputs}\label{regional_illust}
In this section, we demonstrate the regional analysis described in Section~\ref{regbage} for a VNN model that was trained to predict chronological age for HC group in OASIS-3 dataset. All mathematical notations referred to in this section are borrowed from Section~\ref{regbage}. Note that no further training was performed for this VNN model to evaluate brain age or regional residuals. 

The covariance matrix in this VNN model was replaced with ${\bC}^{\sf AD+}_{148}$ and a combined group of HC and AD+ individuals were processed. Thus, for each individual $i$ in this combined dataset, we obtained a age prediction $\hat y_i$ and a vector of residuals $\br_i$. The size of residual vector $\br_i$ was $148\times 1$ and hence, each element of $\br_i$ corresponded to a distinct brain region as defined by the DKT brain atlas with 148 parcellations. By evaluating the vector of residuals $\br_i$ for every individual in the combined dataset, we were able to form a population of residual vectors from HC group (referred to as $\br_{\sf HC}$) and AD+ group (referred to as $\br_{\sf AD+}$). The elements of these residual vectors are referred to as regional residuals throughout the paper.

Each dimension of these residual vectors was investigated for group differences between HC and AD+ groups via ANOVA as described in Section~\ref{regbage}. Thus, for every VNN model, we eventually performed $m = 148$ number of ANOVA tests and evaluated the brain regions for significance in group differences in their respective residuals. The significance of group differences between the distributions of regional residuals for HC and AD+ groups corresponding to a brain region was determined after correcting the $p$-values of ANOVA test for multiple comparisons via Bonferroni correction (Bonferroni corrected $p$-value $< 0.05$). The group differences were additionally investigated for significance at an uncorrected level using ANCOVA with age and sex as covariates.

Figure~\ref{example_AD} illustrates the results obtained via ANOVA in this context. The brain regions are shaded according to the magnitude of the F-values from ANOVA and all shaded regions were significant according to the criteria provided in Section~\ref{regbage}. The box plots for various brain regions show that the residuals were significantly elevated in AD+ group as compared to HC. Although many brain regions were identified as significant, we note that there was a clear variability in significance, with some brain regions having more significant group difference than others. We had 100 trained VNN models for the OASIS-3 dataset and performed similar analyses for each of them. Further, we counted the number of models for which the above described analysis yielded a brain region to be significant. A brain region with robust group difference in its regional residual distribution in HC vs AD+ was expected to be more frequently labeled as significant by the VNN models. The results of this robustness analyses on the OASIS-3 dataset are shown in Fig.~\ref{roi_AD}b. 

Figure~\ref{example_AD_FTDC} illustrates the results obtained by analyzing the regional residuals that were derived by transferring a VNN model from FTDC100 dataset to OASIS-3 dataset. Here, we observed that  brain regions similar to those in Fig.~\ref{example_AD} exhibited significantly elevated regional residuals in the AD+ group (albeit with lower significance in terms of $F$-values of the ANOVA test). Thus, Fig.~\ref{example_AD_FTDC} provides the evidence that VNN trained on FTDC100 dataset can transfer the interpretability from FTDC100 dataset to OASIS-3 dataset despite the lack of transferability of quantitative performance. 


\begin{figure}[!htbp]
  \centering
  \includegraphics[scale=0.4]{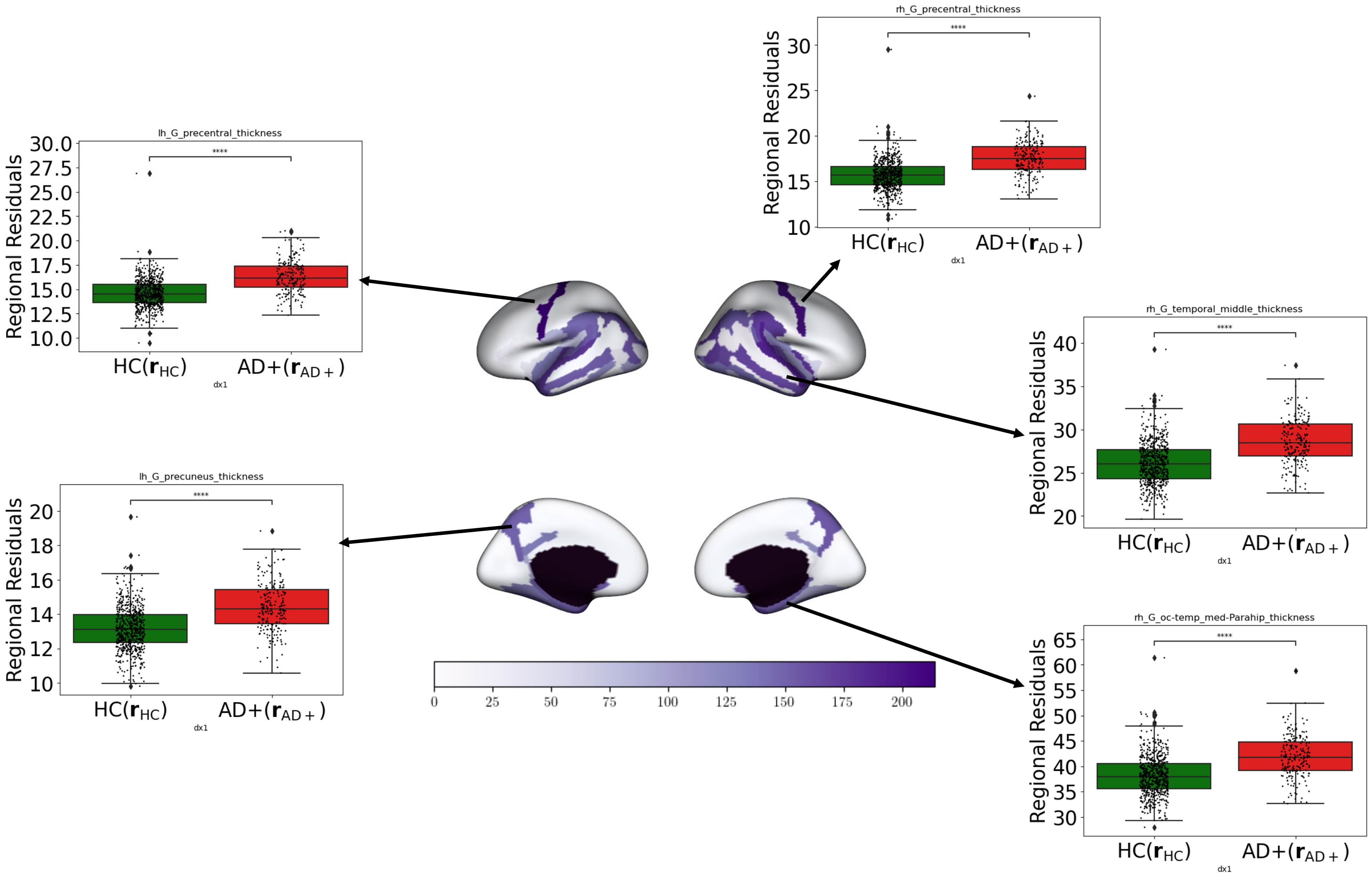}
   \caption{Results depicting the brain regions with significantly elevated regional residuals for AD+ group with respect to HC group in OASIS-3. The results here were derived by a VNN model that was trained as a regression model to predict chronological age from cortical thickness data for HC group in OASIS-3. The contrasts of the shaded regions are proportional to the F-values obtained from ANOVA and all shaded regions exhibited elevated regional residuals in AD+ with respect to HC.}
   \label{example_AD}
\end{figure}

\begin{figure}[!htbp]
  \centering
  \includegraphics[scale=0.4]{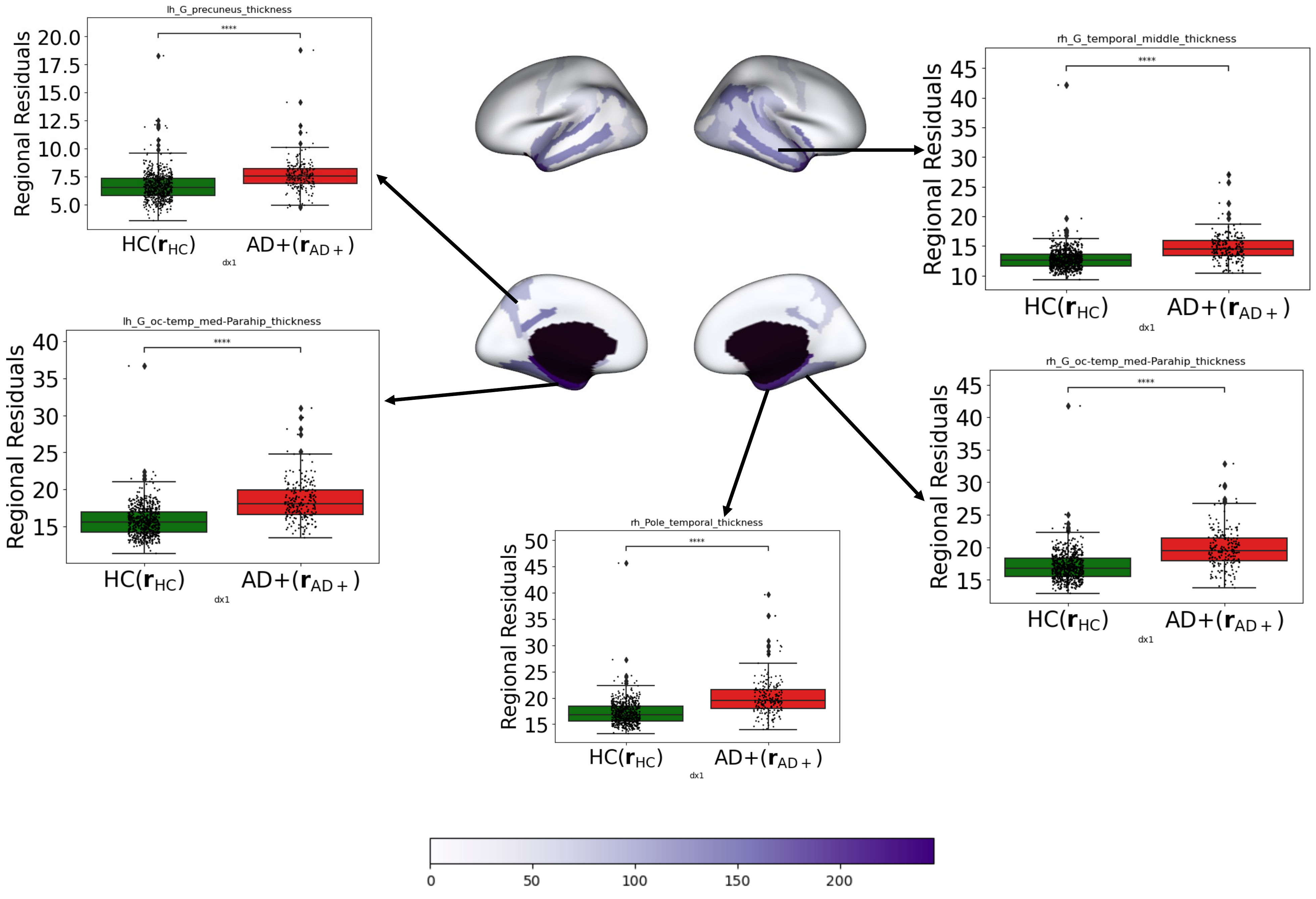}
   \caption{Results depicting the brain regions with significantly elevated regional residuals for AD+ group with respect to HC group in OASIS-3. The results here were derived by a VNN model that was transferred from FTDC100 to process OASIS-3 and had been trained as a regression model to predict chronological age from cortical thickness data for healthy controls in the FTDC100 dataset. The contrasts of the shaded regions are proportional to the F-values obtained from ANOVA and all shaded regions exhibited significant elevated regional residuals in AD+ with respect to HC (Bonferroni-corrected $p$-value smaller than $0.05$).}
   \label{example_AD_FTDC}
\end{figure}

\clearpage

\iftrue
\clearpage
\section{Additional details on brain age prediction in OASIS-3}\label{brain_age_supp}
In this section, we provide additional figures and discussions pertaining to the results for interpretable brain age prediction in Fig.~\ref{roi_AD}. Figure~\ref{age_dist} displays the distributions of chronological age for AD+ and HC groups in OASIS-3 dataset.
\begin{figure}[!htbp]
  \centering
  \includegraphics[scale=0.35]{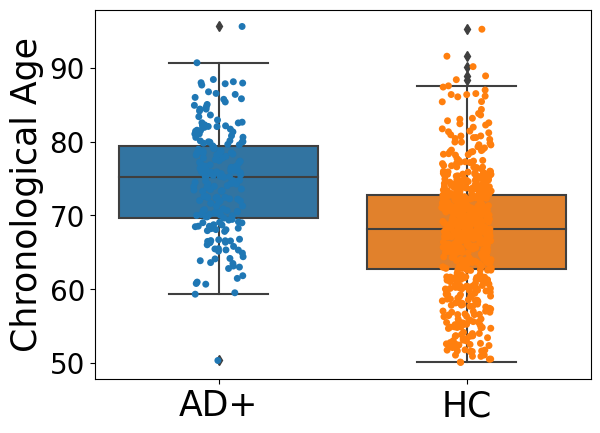}
   \caption{Distribution of chronological age in AD+ and HC groups.}
   \label{age_dist}
\end{figure}

Since VNNs were trained for the regression task, a VNN processed cortical thickness data and provided an estimate for the chronological age for each individual. Since we trained 100 VNN models on different permutations of the training set in OASIS-3, we use the mean of all VNN estimates as the VNN prediction for an individual. This VNN prediction is further leveraged to form brain age estimates and $\Delta$-Age metrics. Figure~\ref{brain_age_AD}a displays the plot for VNN predictions versus chronological age (ground truth) for the complete HC group. The Pearson's correlation between VNN prediction and chronological age (ground truth) for HC group was $0.43$ which was similar to that reported in Table~\ref{transfer_tbl1}. This observation implied that there was no degradation in the performance of VNNs on the task of predicting chronological age for HC group despite replacing the anatomical covariance matrix only from HC group with that from the combined HC and AD+ group ($\hat\bC_{148}^{\sf AD+}$). However, VNN outputs clearly under-estimated the chronological age for older individuals and over-estimated the chronological age for individuals on the younger end of the age distribution for HC group.

Figure~\ref{brain_age_AD}b displays the plot for VNN predictions versus chronological age (ground truth) for the complete AD+ group. The Pearson's correlation between VNN prediction and chronological age (ground truth) for AD+ group was $0.28$. We further note that the VNN architecture and our analysis of regional residuals helped quantify the contribution of each brain region to a data point in Fig.~\ref{brain_age_AD}a and Fig.~\ref{brain_age_AD}b. Hence, the scatter plot in Fig.~\ref{brain_age_AD}b could be affected by larger contributions of certain brain regions for AD+ group relative to the HC group. 

Figure~\ref{brain_age_AD}c illustrates the box plots of residuals evaluated by the difference between VNN predictions and chronological age for HC and AD+ groups. Figure~\ref{brain_age_AD}c suggests that the chronological age for AD+ group was underestimated as compared to that for HC group. This observation was also expected since AD+ group is significantly older than the HC group. However, we expect that the robust elevated regional residuals from brain regions in Fig.~\ref{roi_AD}b mitigated the under-estimation effect due to higher age of AD+ group to some extent. 

Figures~\ref{brain_age_AD}d-f display the results after age-bias correction is applied to the VNN outputs. As expected, the brain age for HC group in Fig.~\ref{brain_age_AD}d is largely concentrated around the line of equality ($x = y$ line). In contrast, the brain age for AD+ group in Fig.~\ref{brain_age_AD}e is concentrated above the line of equality. These effects manifest into the box plots for $\Delta$-Age in Fig.~\ref{brain_age_AD}f where we observe the AD+ group to have elevated $\Delta$-Age as compared to HC group.

VNN architecture facilitated isolation of the effects of accelerated aging before age-bias correction was applied. Hence, the transformation of VNN outputs to brain age from Fig.~\ref{brain_age_AD}a-c to Fig.~\ref{brain_age_AD}d-f was not surprising. However, such insights may be infeasible for machine learning approaches that lack transparency and hence, the impact of deviations due to neurodegeneration from the healthy control population cannot be interpreted or isolated. In this context, if the learning model was a black box, Fig.~\ref{brain_age_AD}a-c may appear to be counter-intuitive to the goal of detecting accelerated aging in the AD+ group and the effect of age-bias correction can be unclear, thus, leading to several criticisms~\cite{butler2021pitfalls}. 

\begin{figure}[!htbp]
  \centering
  \includegraphics[scale=0.35]{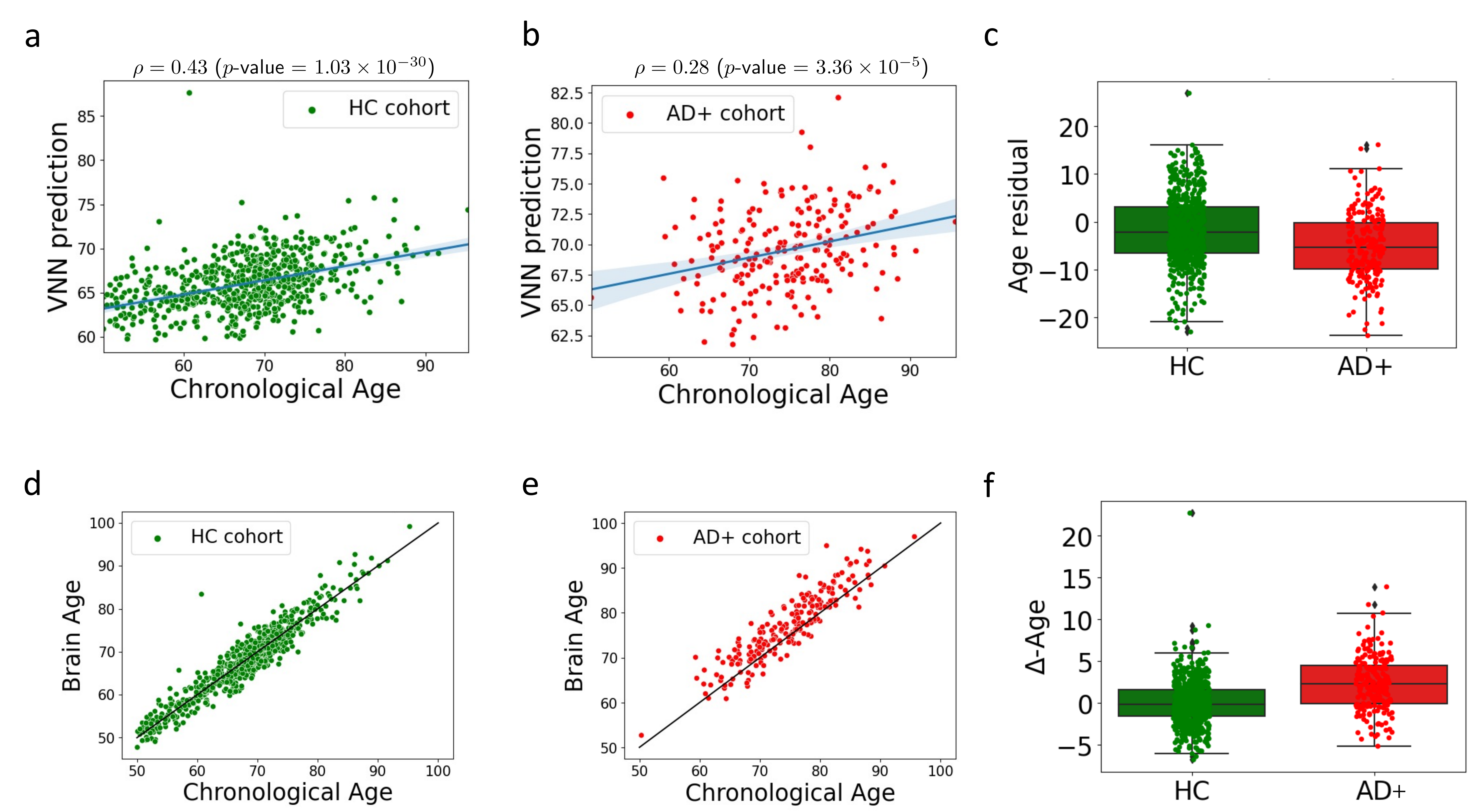}
   \caption{{\bf Supplementary figures to Fig.~\ref{roi_AD}.} Panel {\bf a} displays the plot of VNN prediction versus chronological age for HC group. VNN predictions were obtained as the average of the outputs of $100$ nominal VNNs that were trained on OASIS-3 and operated on the anatomical covariance matrix $\hat\bC_{148}^{\sf AD+}$. Panel {\bf b} displays the results similar to that in panel {\bf a} for the AD+ group. The solid line in panels {\bf a} and {\bf b} is the least squares line. Panel {\bf c} includes the boxplots for residuals derived from the difference between VNN predictions and chronological age for HC and AD+ groups. Panel {\bf d} and {\bf e} display the plots for brain age versus chronological age for HC and AD+ groups, respectively. The solid line in panels {\bf d} and {\bf e} is the identity line. Panel {\bf f} displays the box plots for $\Delta$-Age in HC and AD+ groups.  }
   \label{brain_age_AD}
\end{figure}


\fi 
\clearpage

\section{Additional results on associations between regional residuals derived from VNNs and eigenvectors of $\hat\bC_{148}^{\sf AD+}$}\label{inner_prod_all}

Figure~\ref{eig_plot_ftdc} displays the results for the associations between the regional contributions to the VNN outputs ($\bar\bp_{\sf HC}$) and the eigenvectors of the anatomical covariance matrix $\hat\bC_{148}$ for VNNs trained on FTDC300 or FTDC500 datasets and transferred to process the data from HC group in OASIS-3 dataset. The results in Fig.~\ref{eig_plot_ftdc} show that the first, third, and fourth eigenvectors of $\hat\bC_{148}$ had the three highest association with the regional contributions derived from VNNs trained on FTDC300 or FTDC500 datasets, which was consistent with the results derived for VNNs trained on OASIS-3 or FTDC100 datasets in Fig.~\ref{vnn_hc_eig_fig}. Besides these three eigenvectors, there were a smaller number of eigenvectors for which $|\!<\!\bar\bp_{\sf HC}, \bv_i\!>\!|$ had a coefficient of variation less than $30\%$ across the HC group for the results derived from VNNs trained on FTDC300 or FTDC500 datasets in Fig.~\ref{eig_plot_ftdc} as compared to those derived from VNNs trained on OASIS-3 dataset in Fig.~\ref{vnn_hc_eig_fig}a. 
\begin{figure}[!htbp]
  \centering
  \includegraphics[scale=0.5]{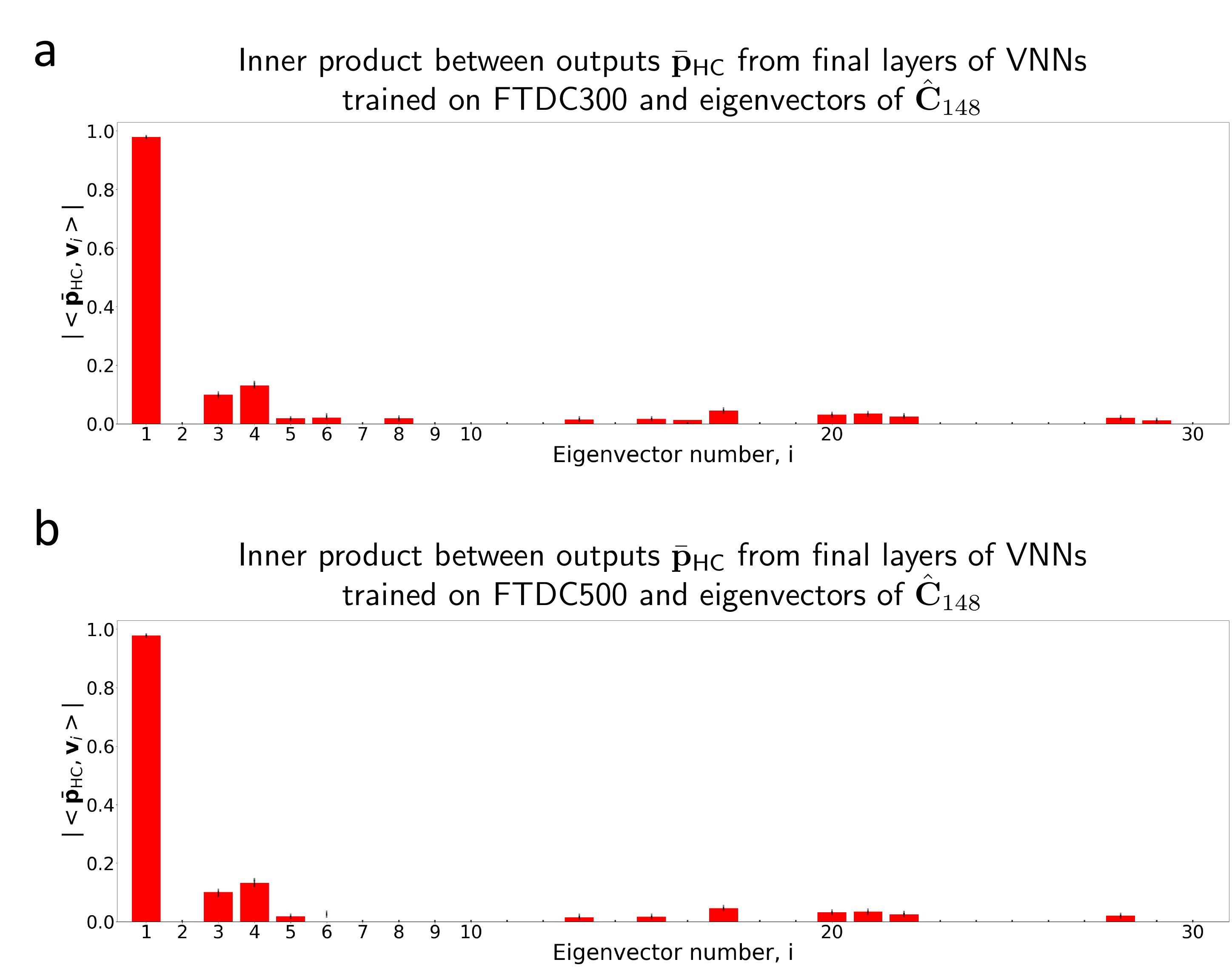}
   \caption{ {\bf Associations between eigenvectors of $\hat\bC_{148}$ and regional contributions to the VNN outputs ($\bar\bp_{\sf HC}$) for VNNs transferred from FTDC300 or FTDC500 datasets to OASIS-3 dataset.} Panel~{\bf a} illustrates a bar plot for $|\!<\!\bar\bp_{\sf HC}, \bv_i\!>\!|$ for $i\in \{1,\dots, 30\}$, where $\bv_i$ is the $i$-th eigenvector (principal component) of covariance matrix $\hat\bC_{148}$ and associated with $i$-largest eigenvalue in terms of magnitude and the vectors of regional contributions, $\bar\bp_{\sf HC}$ were obtained by VNNs that were trained on FTDC300 dataset and transferred to process the data from HC group in OASIS-3 dataset. The inner product results for eigenvectors with coefficient of variation greater than $30\%$ across the HC group of OASIS-3 were excluded (and hence, their contributions set as 0). For every individual in HC group, the associations between their corresponding vector of regional contributions, $\bar\bp_{\sf HC}$ and eigenvectors of $\hat\bC_{148}$ were evaluated over $100$ nominal VNN models. Panel~{\bf b} illustrates similar results as in Panel~{\bf a} for the VNNs that were trained on FTDC500 dataset and transferred to process the data from HC group in OASIS-3.}
   \label{eig_plot_ftdc}
\end{figure}

Figure~\ref{inner_prod_reg_profile}a displays the associations between the eigenvectors of $\hat\bC_{148}^{\sf AD+}$ and regional residuals for AD+ group that were derived from VNNs trained on different datasets. For VNNs trained on OASIS-3, the three largest associations between regional residuals and eigenvectors were observed for the first, third, and fourth eigenvectors of $\hat\bC_{148}^{\sf AD+}$. Interestingly, this phenomenon was retained when regional residuals were evaluated using the final layer outputs of the VNNs that had been transferred from FTDC datasets to OASIS-3 dataset. However, we also observed that the variance in the inner products was larger for VNNs trained on FTDC datasets as compared to those trained on OASIS-3 dataset. As a consequence, in Fig.~\ref{inner_prod_reg_profile}a, the number of eigenvectors whose associations with regional residuals had a coefficient of variation smaller than $30\%$ diminished considerably for VNNs trained on FTDC300 or FTDC500 datasets. 

Figure~\ref{inner_prod_reg_profile}b displays the regional profiles corresponding to the robustness of the elevated regional residuals in AD+ group with respect to HC group in OASIS-3 derived from different sets of 100 VNNs. Considering the regional profile derived from VNNs trained on OASIS-3 dataset as the baseline, the regional profile derived from the analyses of regional residuals obtained from VNNs trained on FTDC100 was the most consistent. The regional profiles derived from the VNNs trained on FTDC300 or FTDC500 datasets retained the most robustness in parahippocampal and subcallosal regions. These observations were expected from the results in Fig.~\ref{inner_prod_reg_profile}a as the third and fourth eigenvectors of $\hat\bC_{148}^{\sf AD+}$ spanned the parahippocampal and subcallosal regions and were highly correlated with the regional residuals derived from VNNs trained on FTDC datasets. The results in Fig.~\ref{inner_prod_reg_profile} further corroborate our conclusion that the regional profiles characteristic of the contributors to elevated $\Delta$-Age were dependent on the ability of VNNs to exploit specific eigenvectors of the anatomical covariance matrix.

\begin{figure}[!htbp]
  \centering
  \includegraphics[scale=0.33]{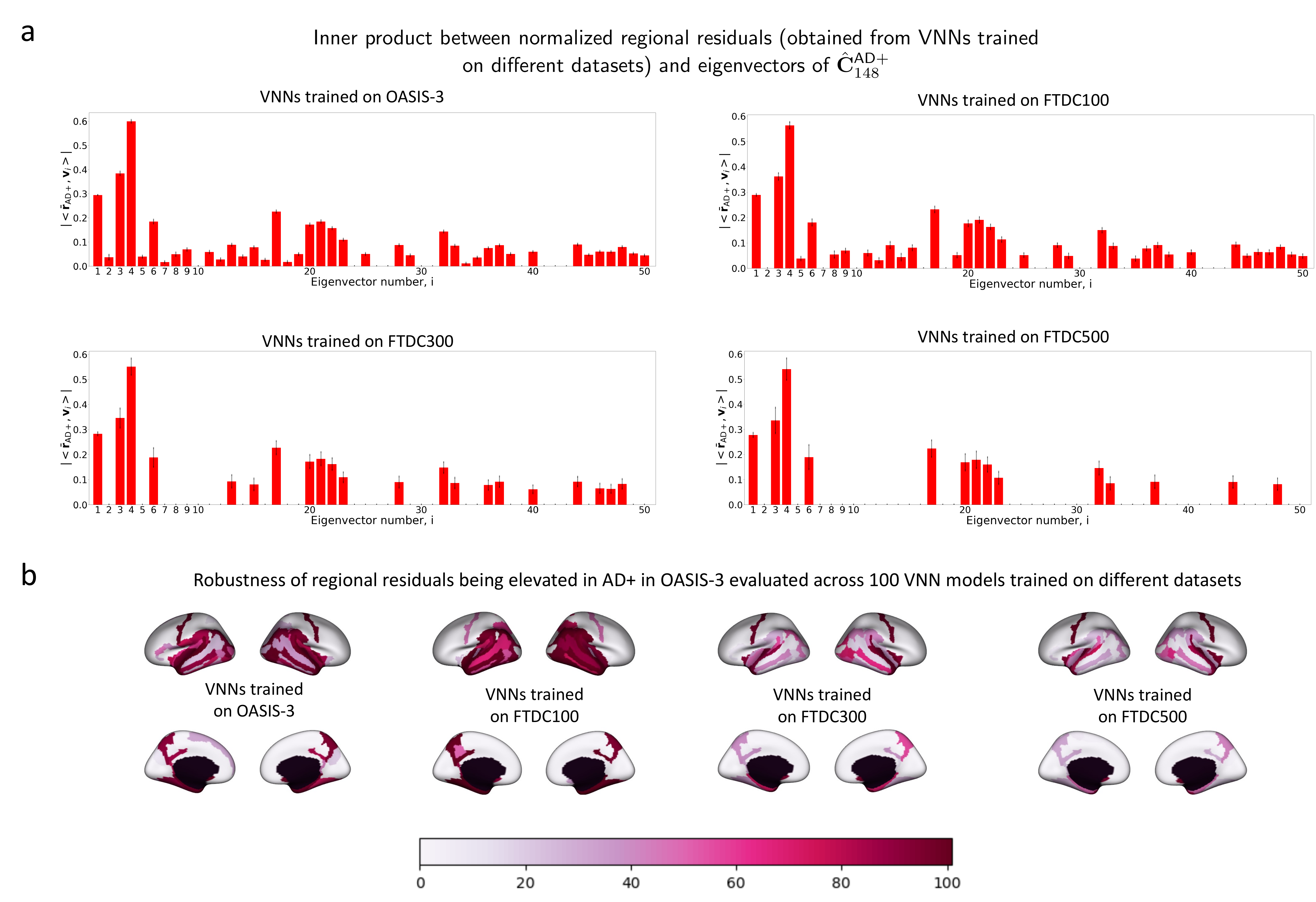}
   \caption{ {\bf Associations between eigenvectors of $\hat\bC_{148}^{\sf AD+}$ and regional residuals for AD+ group; and regional profiles corresponding to elevated regional residuals in AD+ group.} Panel~{\bf a} illustrates the bar plots for the inner products  $|\!<\!\bar\br_{\sf AD+},\bv_i\!>\!|$ between the first $50$ eigenvectors of the anatomical covariance matrix $\hat\bC_{148}^{\sf AD+}$ and the normalized regional residuals evaluated across the AD+ group, where the regional residuals were derived from the final layer outputs of different sets of VNNs. The results for eigenvectors for which  $|\!<\!\bar\br_{\sf AD+},\bv_i\!>\!|$ had a coefficient of variation across the AD+ group greater than $30\%$ were excluded (and hence, have entry $0$ on the bar plot). Panel~{\bf b} displays the regional profiles for the brain regions associated with robust and significant elevated regional residuals in AD+ group as compared to HC group in OASIS-3. The regional profiles obtained from the analyses of regional residuals that were derived from VNNs trained on OASIS-3 and FTDC datasets are included. }
   \label{inner_prod_reg_profile}
\end{figure}

\clearpage

\section{Transferability of VNNs for evaluating regional profiles and $\Delta$-Age in multi-scale FTDC datasets}\label{explor}

In this section, we provide a proof of concept that transferability between different scales of Schaefer's atlas for chronological age prediction lead to consistent regional profiles for $\Delta$-Age determined by regional residuals when VNNs were transferred across cortical thickness datasets corresponding to different resolutions of Schaefer's atlas. The dataset used for this experiment is described below.

\noindent
{\bf Dataset:} This dataset consisted of $105$ healthy controls from the FTDC dataset and $67$ individuals with mild cognitive impairment or Alzheimer's disease diagnosis (AD+; age = $68.52 \pm 9.29$ years, $28$ females). 
The cortical thickness data were available at $100$, $300$, and $500$ resolutions of the Schaefer's atlas. 

The regional residuals were derived for the dataset above using a VNN model that had been trained to predict chronological age in FTDC100 dataset. The group differences in the regional residuals for HC and the combined AD+ group were evaluated using ANOVA for cortical thickness at $100$, $300$, and $500$ resolutions of Schaefer's atlas. The brain regions with robust, significantly elevated regional residuals in AD+ with respect to HC are projected on the brain template for $100$, $300$, and $500$ resolutions in Fig.~\ref{bvftd_vnn}. In Fig.~\ref{bvftd_vnn}, the isolated brain regions were concentrated in similar regions across all resolutions of Schaefer's atlas and were consistent with the brain regions in Fig.~\ref{roi_AD}b (with the exception of the precuneus and superior parietal regions in both hemispheres).

Subsequently, $\Delta$-Age was evaluated for all individuals. The $\Delta$-Age evaluated from the cortical thickness dataset with $100$ features for AD group was $3.67 \pm 3.73$ years and for HC was $0\pm 2.06$ years. Figure~\ref{bvftd_vnn}b displays the box plots for $\Delta$-Age in HC and AD+ groups for FTDC100 dataset as well as those obtained after transferring the VNNs from FTDC100 to FTDC300 and FTDC500 datasets. Due to the transferability of VNNs across different scales of Schaefer's atlas, $\Delta$-Age profiles and brain age versus chronological age plots are consistent across $100$, $300$, and $500$ resolution data even when the VNNs were trained only on the FTDC100 dataset. 

\begin{figure}[!htbp]
  \centering
  \includegraphics[scale=0.65]{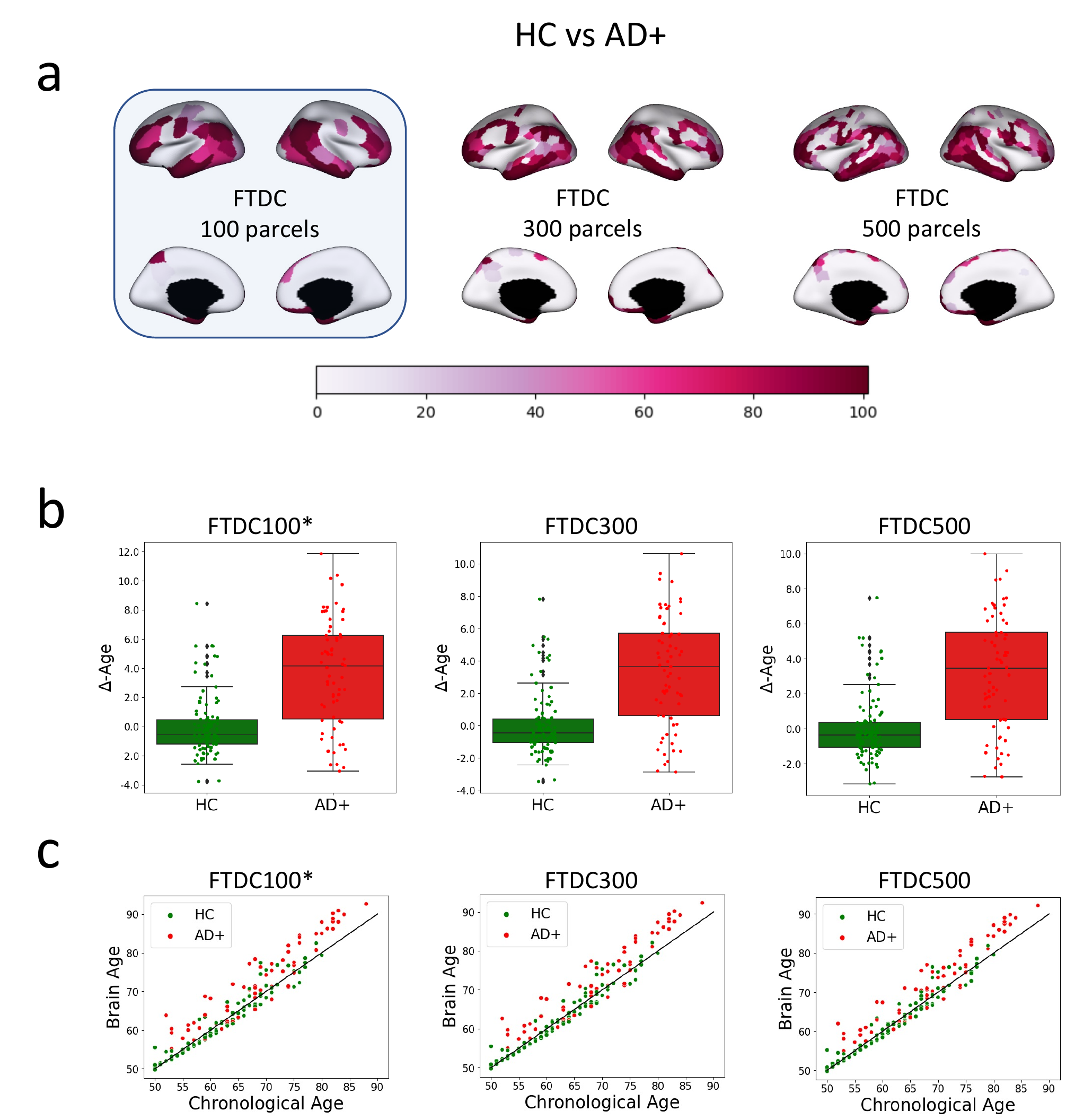}
   \caption{ {\bf Brain age prediction across datasets curated according to multiple scales of Schaefer's atlas.} Panel~{\bf a} illustrates the regional profiles consisting of brain regions with robust, elevated regional residuals in the combined AD+ group with respect to the HC group. The VNNs were trained to predict chronological age on FTDC100 and the robustness was evaluated over $100$ nominal VNN models. The regional profiles were obtained for the datasets with $300$ features and $500$ features after transferring the VNNs from FTDC100 to FTDC300 and FTDC500. Panel~{\bf b} displays the box plots for $\Delta$-Age corresponding to the regional profiles in Panel~{\bf a}. Panel~{\bf c} plots brain age versus chronological age for datasets with $100$, $300$, and $500$ cortical thickness features according to different scales of Schaefer's atlas.}
   \label{bvftd_vnn}
\end{figure}

\clearpage

\section{Regional profiles corresponding to elevated regional residuals in AD+ group are stable to the composition of data used to estimate anatomical covariance matrix $\hat\bC_{148}^{\sf AD+}$}\label{el_stab}
Recall that $\Delta$-Age and associated regional profiles were evaluated using VNNs that operated upon a composite anatomical covariance matrix $\hat\bC_{148}^{\sf AD+}$. We next checked whether the results derived from VNNs were stable to the changes in composition of the combined HC and AD+ group used to estimate the anatomical covariance matrix. For this purpose, we performed two sets of experiments. In the first set, we included the whole HC group and gradually varied the number of individuals from the AD+ group to be included to form the estimate $\hat\bC_{148}^{\sf AD+}$. Figure~\ref{oasis_stability}a includes the results obtained from a randomly selected VNN model corresponding to anatomical covariance matrix formed by different combinations of the individuals from HC and AD+ groups. The results in Fig.~\ref{oasis_stability}a display the F-values of ANOVA when the regional residuals from AD+ group are higher than HC group (Bonferroni corrected $p$-value $<0.05$). The results obtained by the VNN when it used $\hat\bC_{148}^{\sf AD+}$ that was estimated from all $652$ HC individuals and $209$ AD+ individuals displays the most robust group differences in regions that constitute the regional profile for group differences in $\Delta$-Age between AD+ and HC in Fig.~\ref{roi_AD}b. When the covariance matrix $\hat\bC_{148}^{\sf AD+}$ was perturbed by using smaller number of individuals from the AD+ group to estimate it, we observed that the obtained results were preserved till exclusion of about $134$ AD+ individuals. The results obtained via VNN when operating on $\hat\bC_{148}^{\sf AD+}$  estimated from $50$ or smaller number of AD+ individuals and all $652$ HC individuals became noticeably less significant in terms of $F$-values in general, with the $F$-values in subcallosal, parahippocampal, temporal pole, and medial temporal lobe regions in both hemispheres among those affected noticeably. Hence, the removal of AD related information from the anatomical covariance matrix led to the results in brain regions characteristic of AD becoming less significant.

Figure~\ref{oasis_stability}b illustrates the results obtained for a similar experiment as above, with the difference that the regional residuals were evaluated for the VNN when the anatomical covariance matrix $\hat\bC_{148}^{\sf AD+}$  was perturbed by reducing the number of individuals from the HC group used to estimate it. Using the result obtained for $\hat\bC_{148}^{\sf AD+}$ estimated from $652$ HC individuals and $209$ AD+ individuals in Fig.~\ref{oasis_stability}a as the baseline, the results pertaining to ANOVA between regional residuals for AD+ and HC groups (with AD+ elevated as compared to HC) remained consistent as long as $100$ or more individuals from HC group were included in forming $\hat\bC_{148}^{\sf AD+}$. With less than 100 number of HC individuals included in $\hat\bC_{148}^{\sf AD+}$, the results became noticeably less significant in precuneus and supramarginal regions in the left hemisphere.

In summary, the group differences observed between the regional residuals for AD+ and HC groups in OASIS-3 dataset were robust to perturbations in the covariance matrix $\hat\bC_{148}^{\sf AD+}$ when it was perturbed from the baseline by using a different combination of HC and AD+ individuals to estimate it. However, (nearly) complete exclusion of HC or AD+ groups from $\hat\bC_{148}^{\sf AD+}$ resulted in loss of significance of the elevation in regional residuals in AD+ for various regions, including bilateral parahippocampal and temporal pole regions, and precuneus and supramarginal regions in the left hemisphere. Thus, both HC and AD+ groups were relevant to the anatomical covariance matrix $\hat\bC_{148}^{\sf AD+}$ that resulted in regional profiles in Fig.~\ref{roi_AD} but the regional profiles in Fig.~\ref{roi_AD} were not overfit on the combination of individuals from HC and AD+ used to estimate $\hat\bC_{148}^{\sf AD+}$.

\section{Adaptive readouts may penalize the interpretability of regional residuals and $\Delta$-Age}\label{vnn_adaptive}
Thus far, we have focused on VNNs that operate with a non-adaptive readout (unweighted average) function. However, it is expected that the performance of the VNNs on their original task of chronological age prediction could be improved significantly with the help of an adaptive readout function. Our experiments showed that this was indeed the case. If a single-layer fully connected perceptron consisting of $10$ neurons was added to the VNNs with the same architecture as the ones that were trained on OASIS-3 dataset, we could improve the performance on the chronological prediction task. For 100 VNNs with adaptive readout that were trained on random permutations of the training data, the median MAE for the HC group was $4.64$ years, which was significantly smaller than the MAE achieved by VNNs with non-adaptive readouts (Table~\ref{transfer_tbl1}). Among the 100 VNN models with adaptive readouts, we analyzed the regional residuals for one VNN model with adaptive readout that had the best performance on chronological age prediction in HC group (test set: MAE = $4.17$ years, Pearson's correlation between prediction and ground truth = $0.73$; complete HC group: MAE = $4.26$ years, Pearson's correlation between prediction and ground truth = $0.725$). Our regional residuals revealed no significant difference between the regional residuals for AD+ group and HC group. This observation suggested that VNN lost its interpretability due to the addition of adaptive readout function. Moreover, we also observed a diminished gap between $\Delta$-Age for AD+ and HC groups determined using this VNN model ($\Delta$-Age for AD group: $1.58\pm 4.67$ years, $\Delta$-Age for HC group: $0\pm 3.45$ years, Cohen's $d$ = 0.384). The findings discussed here suggest that boosting the performance on chronological age prediction task by using an adaptive readout function may penalize the interpretability offered by VNNs with non-adaptive readouts and also diminish the $\Delta$-Age gap between AD+ and HC groups. 

\iftrue

\begin{FPfigure}
  \centering 
  \includegraphics[scale=0.6]{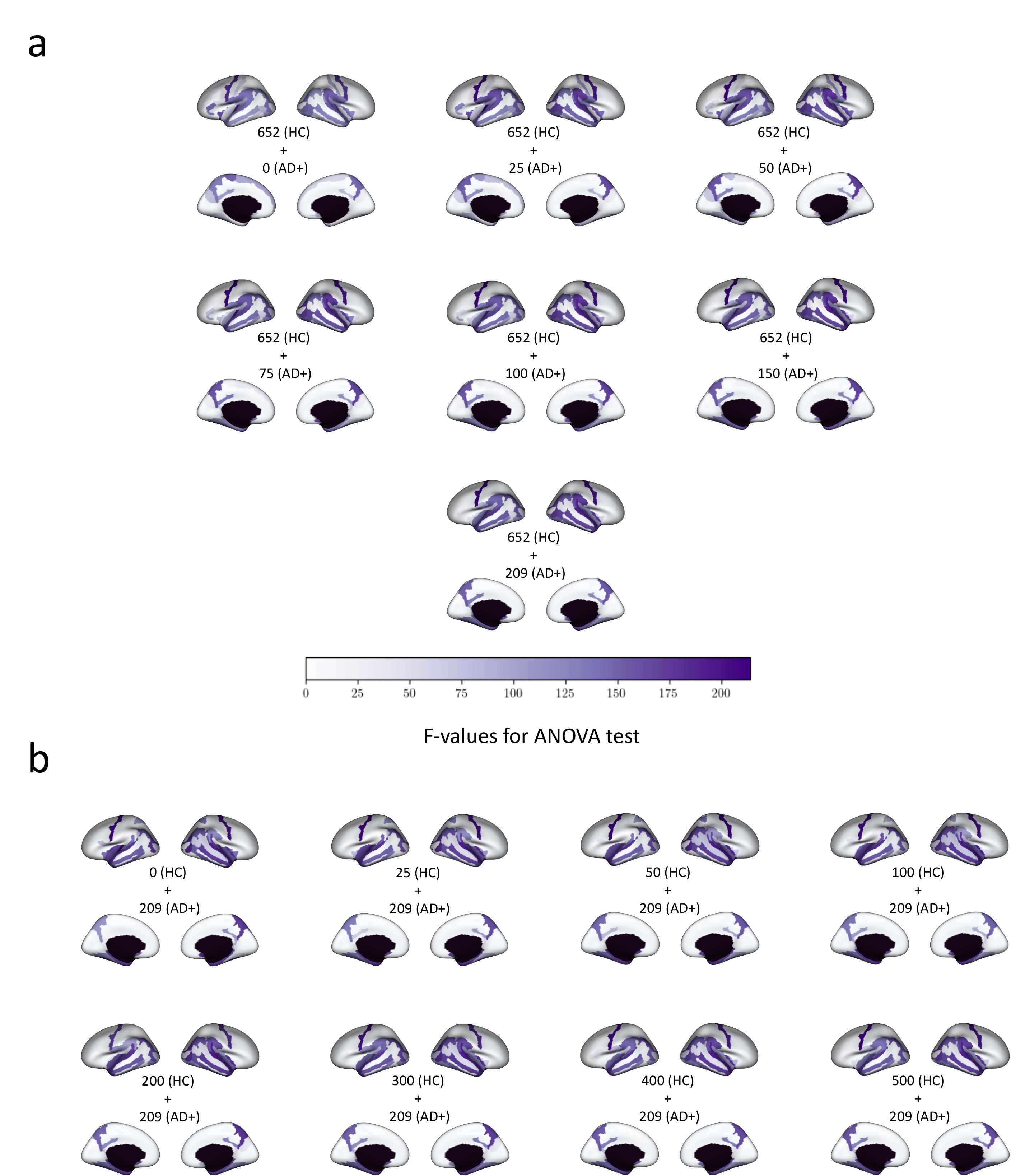}
   \caption{{\bf Stability to perturbations in the anatomical covariance matrix for group differences between AD+ and HC groups observed in regional residuals.} For a VNN model that was trained to predict chronological age for HC group in OASIS-3 dataset, the regional residuals were first determined using the anatomical covariance matrix $\hat\bC_{148}^{\sf AD+}$ formed by the cortical thickness data of complete OASIS-3 dataset (i.e., 652 HC individuals and 209 individuals in the AD+ group. The group differences in regional residuals between AD+ and HC group were investigated according to the procedure in subsection~\ref{regbage} in the Methods section. In the procedure described therein, we evaluated the F-values for the ANOVA test between regional residuals for AD+ group and HC group. The F-values for the regional residuals that were elevated in AD+ group with respect to HC group are projected on the brain template in panel {\sf a}. The stability of the group differences to perturbations in $\hat\bC_{148}^{\sf AD+}$ was further investigated by varying the composition of cortical thickness data from AD+ and HC groups used to estimate $\hat\bC_{148}^{\sf AD+}$. Panel~{\bf a} displays the results obtained via analysis of regional residuals by VNNs that processes the cortical thickness data from complete OASIS-3 dataset over the anatomical covariance matrix $\hat\bC_{148}^{\sf AD+}$ estimated from $652$ HC individuals and varying number of individuals from the AD+group. The scenarios in which $\hat\bC_{148}^{\sf AD+}$ was estimated using cortical thickness data from $652$ HC individuals and no AD+ individuals formed one extreme of these experiments.  Panel~{\bf b} illustrates the results of similar experiments as in Panel~{\bf a} with the difference that the anatomical covariance matrix $\hat\bC_{148}^{\sf AD+}$ was estimated using all $209$ individuals in the AD+ group and varying number of individuals from the HC group. The results corresponding to  $\hat\bC_{148}^{\sf AD+}$ that was estimated using $209$ AD+ individuals and $652$ HC individuals formed the baseline for comparison for all scenarios in Panels {\bf a} and {\bf b}.}
   \label{oasis_stability}
\end{FPfigure}

\fi

\end{appendices}

\clearpage
\bibliographystyle{IEEEtran}
{\footnotesize \bibliography{vnn_brain_age}}
\end{document}